\pdfoutput=1
\documentclass[10pt,journal,compsoc]{IEEEtran}

\usepackage{cite}
\usepackage{preamble}

\newif\iflongversion
\longversiontrue

\newif\ifprop
\newif\ifproof
\proptrue
\prooffalse

\newif\ifintable
\intablefalse

\newcommand{\LossPrefix}{\ifintable\else\small\fi}

\newcommand{\Sm}{\Delta}

\newcommand{\Ska}{\Delta^{\alpha}_{k}}
\newcommand{\Skb}{\Delta^{\beta}_{k}}
\newcommand{\kerr}[1]{\err_k\!\left(#1\right)}
\newcommand{\nerr}[2]{\err_{#1}\!\left(#2\right)}

\DeclareRobustCommand{\PrecAtK}[1]{\ifmmode\mathrm{P}@{#1}\else%
{\LossPrefix$\mathrm{P}@{#1}$}\fi}
\DeclareRobustCommand{\RecAtK}[1]{\ifmmode\mathrm{R}@{#1}\else%
{\LossPrefix$\mathrm{R}@{#1}$}\fi}

\DeclareRobustCommand{\RLoss}{\ifmmode\mathrm{RLoss}\else%
{\LossPrefix$\mathrm{RLoss}$}\fi}
\DeclareRobustCommand{\HLoss}{\ifmmode\mathrm{HLoss}\else%
{\LossPrefix$\mathrm{HLoss}$}\fi}
\DeclareRobustCommand{\Acc}{\ifmmode\mathrm{Acc}\else%
{\LossPrefix$\mathrm{Acc}$}\fi}
\DeclareRobustCommand{\SAcc}{\ifmmode\mathrm{SAcc}\else%
{\LossPrefix$\mathrm{SAcc}$}\fi}

\DeclareRobustCommand{\Fmicro}{\ifmmode\mathrm{F_1^{micro}}\else%
{\LossPrefix$\mathrm{F_1^{mic}}$}\fi}
\DeclareRobustCommand{\Fmacro}{\ifmmode\mathrm{F_1^{macro}}\else%
{\LossPrefix$\mathrm{F_1^{mac}}$}\fi}
\DeclareRobustCommand{\Finst}{\ifmmode\mathrm{F_1^{inst}}\else%
{\LossPrefix$\mathrm{F_1^{inst}}$}\fi}

\newcommand{\LrOva}{{\LossPrefix$\mathrm{LR^{OVA}}$}}
\newcommand{\LrMulti}{{\LossPrefix$\mathrm{LR^{Multi}}$}}
\newcommand{\LrTopK}[1]{{\LossPrefix$\mathrm{top\mhyphen{}{#1}~Ent}$}}
\newcommand{\LrTopKn}[1]{{\LossPrefix$\mathrm{top\mhyphen{}{#1}~Ent_{tr}}$}}
\newcommand{\LrML}{{\LossPrefix$\mathrm{LR^{ML}}$}}
\newcommand{\SvmOva}{{\LossPrefix$\mathrm{SVM^{OVA}}$}}
\newcommand{\SvmOvag}[1]{{\LossPrefix$\mathrm{SVM^{OVA}_{#1}}$}}

\newcommand{\SvmMulti}{{\LossPrefix$\mathrm{SVM^{Multi}}$}}
\newcommand{\SvmMultig}[1]{{\LossPrefix$\mathrm{SVM^{Multi}_{#1}}$}}
\newcommand{\SvmML}{{\LossPrefix$\mathrm{SVM^{ML}}$}}
\newcommand{\SvmMLg}[1]{{\LossPrefix$\mathrm{SVM^{ML}_{#1}}$}}
\newcommand{\SvmTopK}[1]{{\LossPrefix$\mathrm{top\mhyphen{}{#1}~SVM}$}}
\newcommand{\SvmTopKg}[2]{{\LossPrefix$\mathrm{top\mhyphen{}{#1}~SVM_{#2}}$}}
\newcommand{\SvmTopKa}[1]{{\LossPrefix$\mathrm{top\mhyphen{}{#1}~SVM^{\alpha}}$}}
\newcommand{\SvmTopKb}[1]{{\LossPrefix$\mathrm{top\mhyphen{}{#1}~SVM^{\beta}}$}}
\newcommand{\SvmTopKag}[2]{{\LossPrefix$\mathrm{top\mhyphen{}{#1}~SVM^{\alpha}_{#2}}$}}
\newcommand{\SvmTopKbg}[2]{{\LossPrefix$\mathrm{top\mhyphen{}{#1}~SVM^{\beta}_{#2}}$}}
\newcommand{\SvmTopKab}[1]{{\LossPrefix$\mathrm{top\mhyphen{}{#1}~SVM^{\alpha/\beta}}$}}
\newcommand{\SvmTopKabg}[2]{{\LossPrefix$\mathrm{top\mhyphen{}{#1}~SVM^{\alpha/\beta}_{#2}}$}}

\newcommand{\RFPCT}{{\LossPrefix$\mathrm{RF\mhyphen{}PCT}$}}
\newcommand{\HOMER}{{\LossPrefix$\mathrm{HOMER}$}}
\newcommand{\BR}{{\LossPrefix$\mathrm{BR}$}}

\def\R{\Rb}

\allowdisplaybreaks

\usepackage[pagebackref=false,breaklinks=true,colorlinks,bookmarks=true]{hyperref}
\usepackage[numbered,open]{bookmark}

\newcommand{\eqreftext}[1]{%
\begingroup%
\hypersetup{hidelinks}\eqref{#1}%
\endgroup%
}
\newcommand{\reftext}[1]{%
\begingroup%
\hypersetup{hidelinks}\ref{#1}%
\endgroup%
}

\begin{document}
\title{Analysis and Optimization of Loss Functions
for Multiclass, Top-k, and Multilabel Classification}

\author{Maksim Lapin, Matthias Hein, and Bernt Schiele%
\IEEEcompsocitemizethanks{
\IEEEcompsocthanksitem
M.\ Lapin and B.\ Schiele are with the
Computer Vision and Multimodal Computing group,
Max Planck Institute for Informatics,
Saarbr\"ucken, Saarland, Germany. %
E-mail: \{mlapin, schiele\}@mpi-inf.mpg.de.
\IEEEcompsocthanksitem
M.\ Hein is with the
Department of Mathematics and Computer Science,
Saarland University, Saarbr\"ucken, Saarland, Germany.\protect\\
E-mail: hein@cs.uni-saarland.de.
}}

\IEEEtitleabstractindextext{
\begin{abstract}
Top-k error is currently a popular performance measure on large scale
image classification benchmarks such as ImageNet and Places.
Despite its wide acceptance, our understanding of this metric is limited
as most of the previous research is focused on its special case, the top-1 error.
In this work, we explore two directions that shed more light on the top-k error.
First, we provide an in-depth analysis of established
and recently proposed single-label multiclass methods
along with a detailed account of efficient optimization algorithms for them.
Our results indicate that the softmax loss and the smooth multiclass SVM
are surprisingly competitive in top-k error uniformly across all k,
which can be explained by our analysis of multiclass top-k calibration.
Further improvements for a specific k are possible with
a number of proposed top-k loss functions.
Second, we use the top-k methods to explore the transition
from multiclass to multilabel learning.
In particular, we find that it is possible to obtain effective
multilabel classifiers on Pascal VOC using a single label per image
for training, while the gap between multiclass and multilabel methods
on MS COCO is more significant.
Finally, our contribution of efficient algorithms for training with the
considered top-k and multilabel loss functions is of independent interest.
\end{abstract}

\begin{IEEEkeywords}
Multiclass classification, multilabel classification,
top-k error, top-k calibration, SDCA optimization
\end{IEEEkeywords}}

\maketitle

\IEEEdisplaynontitleabstractindextext

\IEEEpeerreviewmaketitle

\section{Introduction}%
\label{sec:introduction}
\IEEEPARstart{M}{odern} computer vision benchmarks are large scale
\cite{ILSVRCarxiv14,zhou2014learning,lin2014microsoft},
and are only likely to grow further both in terms of the sample size
as well as the number of classes.
While simply collecting more data may be a relatively straightforward exercise,
obtaining high quality ground truth annotation is hard.
Even when the annotation is just a list of image level tags,
collecting a \emph{consistent} and \emph{exhaustive} list
of labels for every image requires significant effort.
Instead, existing benchmarks often offer only a single label
per image, albeit the images may be inherently multilabel.
The increased number of classes then leads to ambiguity
in the labels as classes start to overlap
or exhibit a hierarchical structure.
The issue is illustrated in Figure~\ref{fig:teaser},
where it is difficult even for humans to guess the ground truth label
correctly on the first attempt \cite{zhou2014learning,xiao2010sun}.

Allowing $k$ guesses instead of one leads to what we call
the \emph{top-$k$ error}, which is one of the main subjects of this work.
While previous research is focused on minimizing the top-$1$ error,
we consider $k \geq 1$.
We are mainly interested in two cases:
(i) achieving small top-$k$ error for \emph{all} $k$ simultaneously;
and
(ii) minimization of a specific top-$k$ error.
These goals are pursued in the first part of the paper which is concerned
with single label multiclass classification.
We propose extensions of the established multiclass loss functions
to address top-$k$ error minimization
and derive appropriate optimization schemes based on
stochastic dual coordinate ascent (SDCA) \cite{shalev2013stochastic}.
We analyze which of the multiclass methods are calibrated
for the top-$k$ error
and perform an extensive empirical evaluation
to better understand their benefits and limitations.
An earlier version of this work appeared in \cite{lapin2016loss}.

\begin{figure}[t]\small\centering%
\setlength\fboxsep{0pt}%
\setlength\fboxrule{0.25pt}
\fbox{\includegraphics[height=60pt]%
{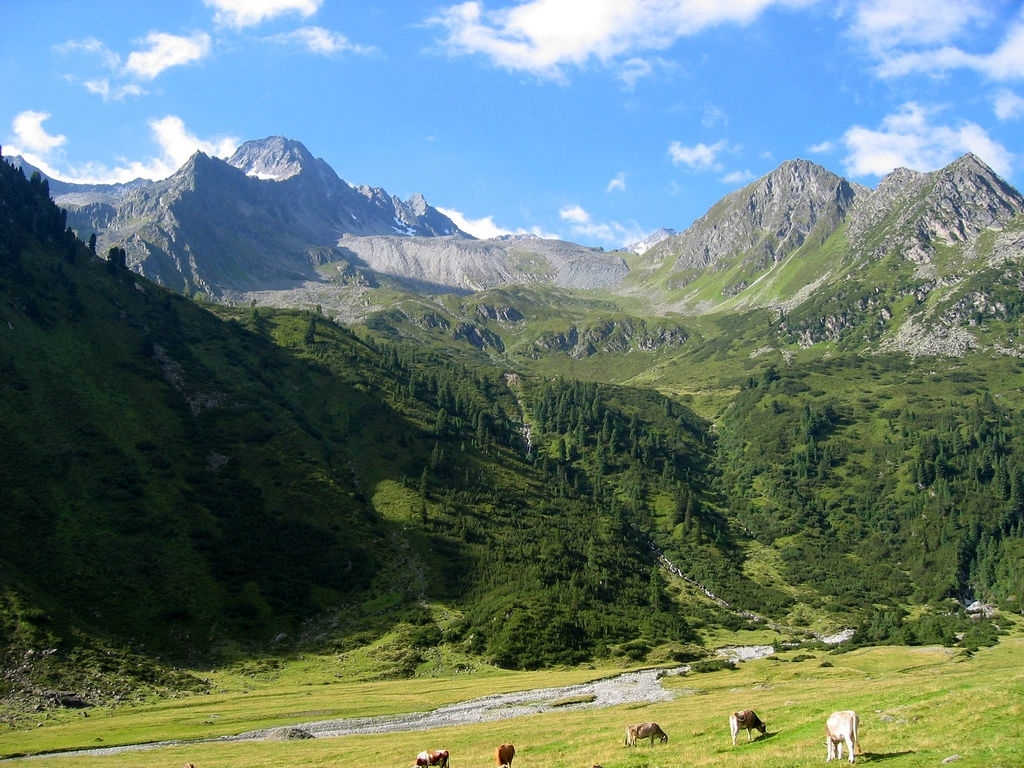}}%
\fbox{\includegraphics[height=60pt]%
{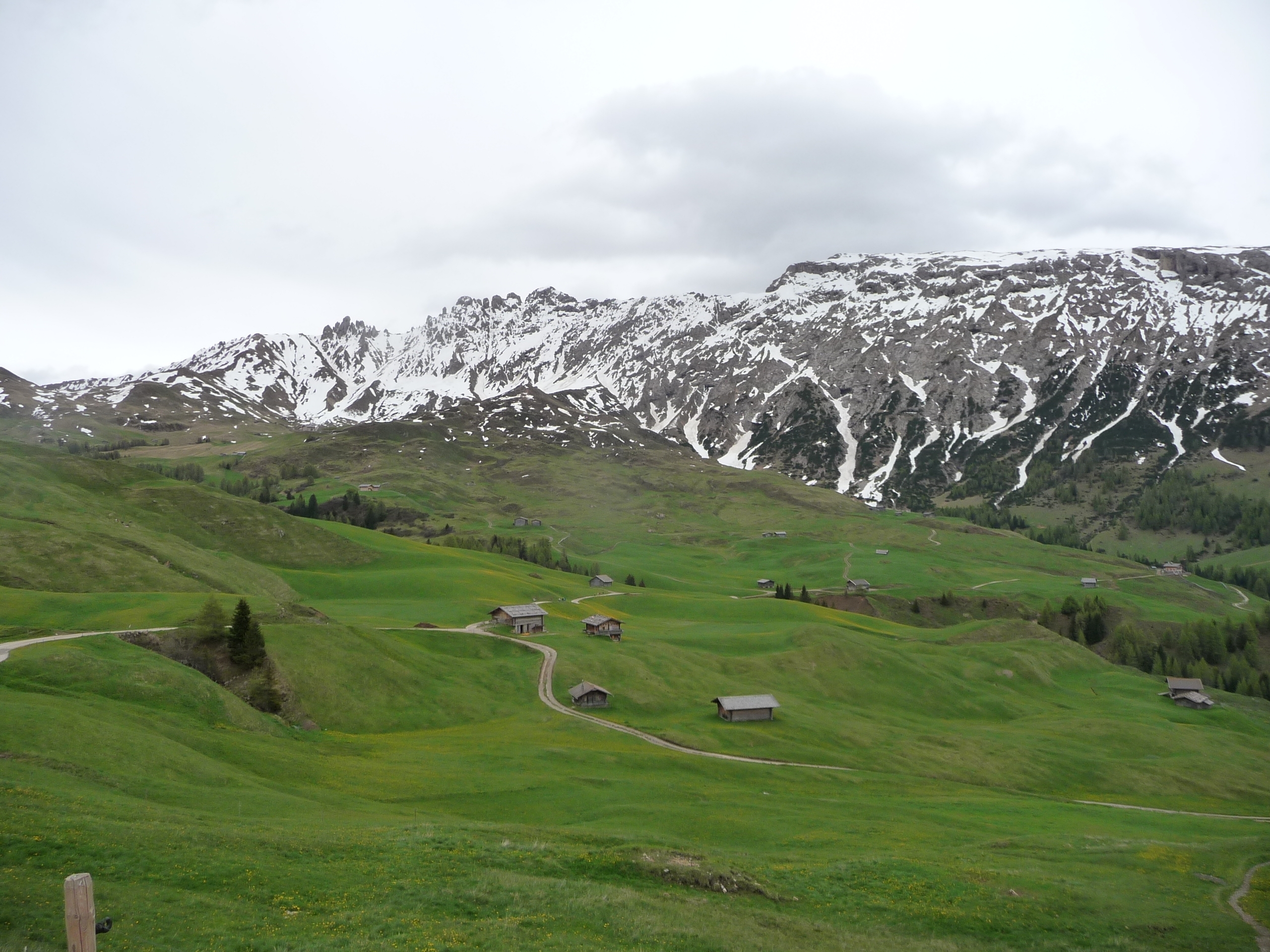}}%
\fbox{\includegraphics[height=60pt]%
{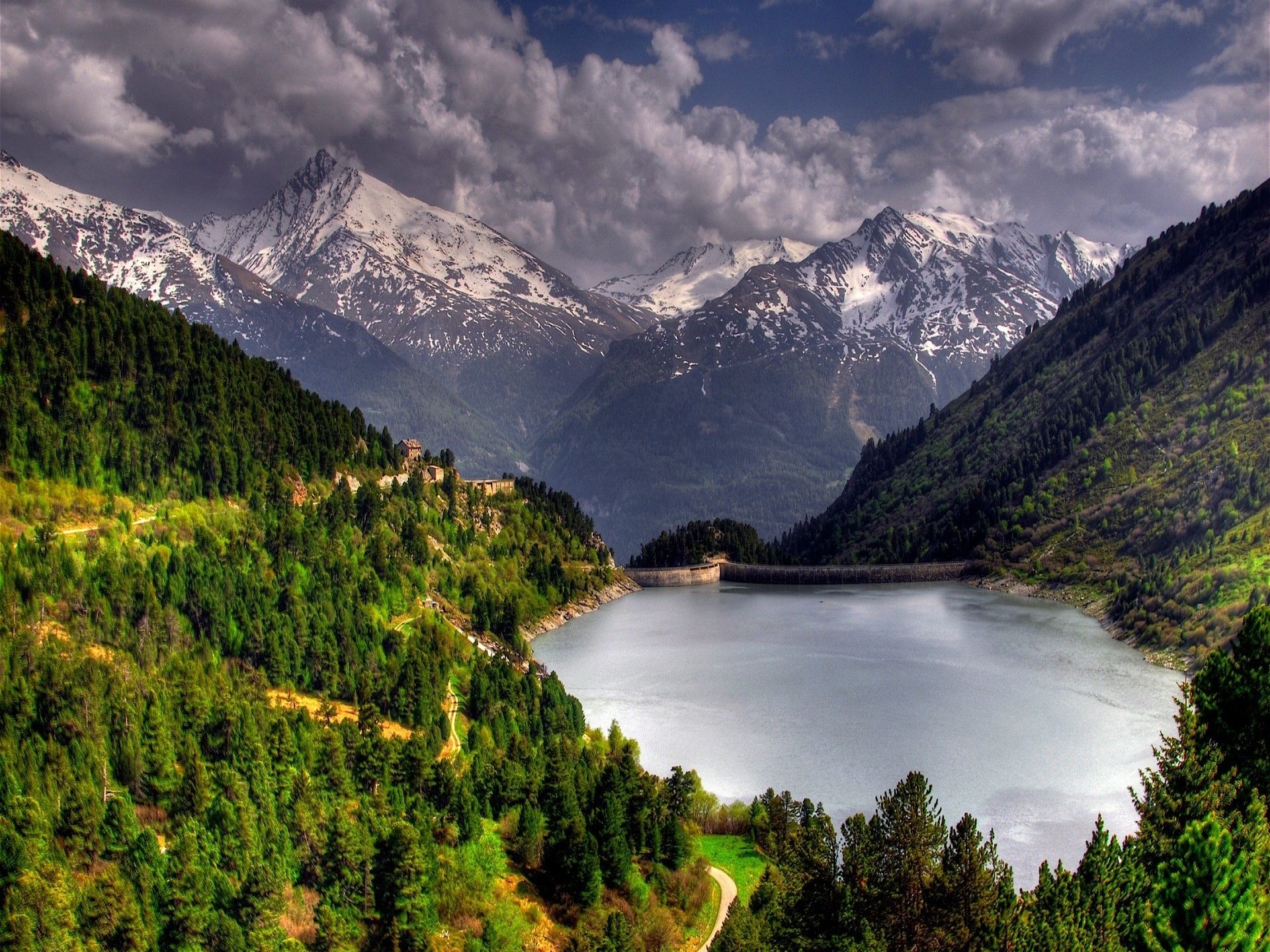}}%
\\
\fbox{\includegraphics[height=57.4pt]%
{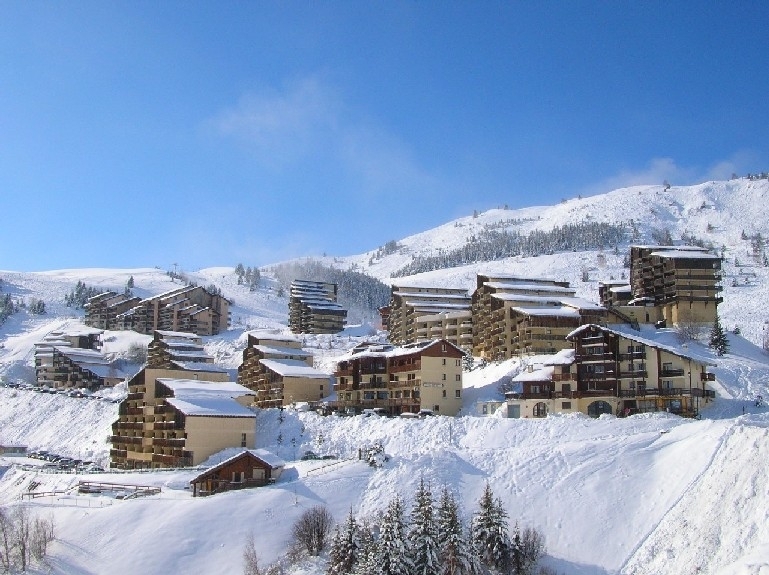}}%
\fbox{\includegraphics[height=57.4pt]%
{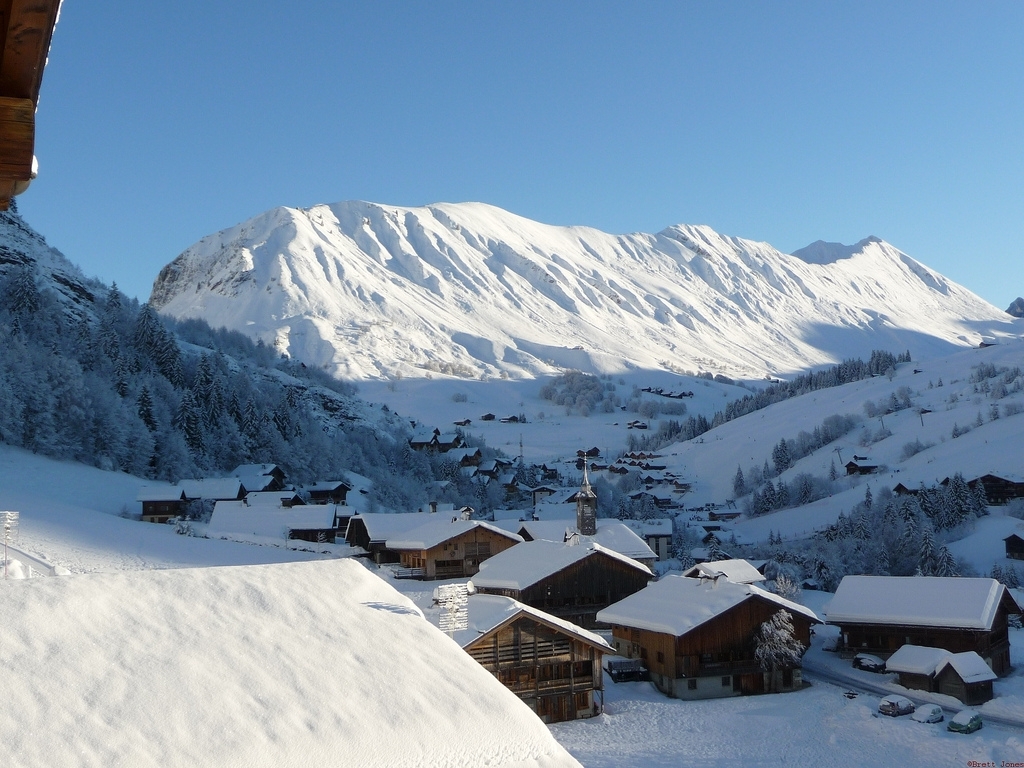}}%
\fbox{\includegraphics[height=57.4pt]%
{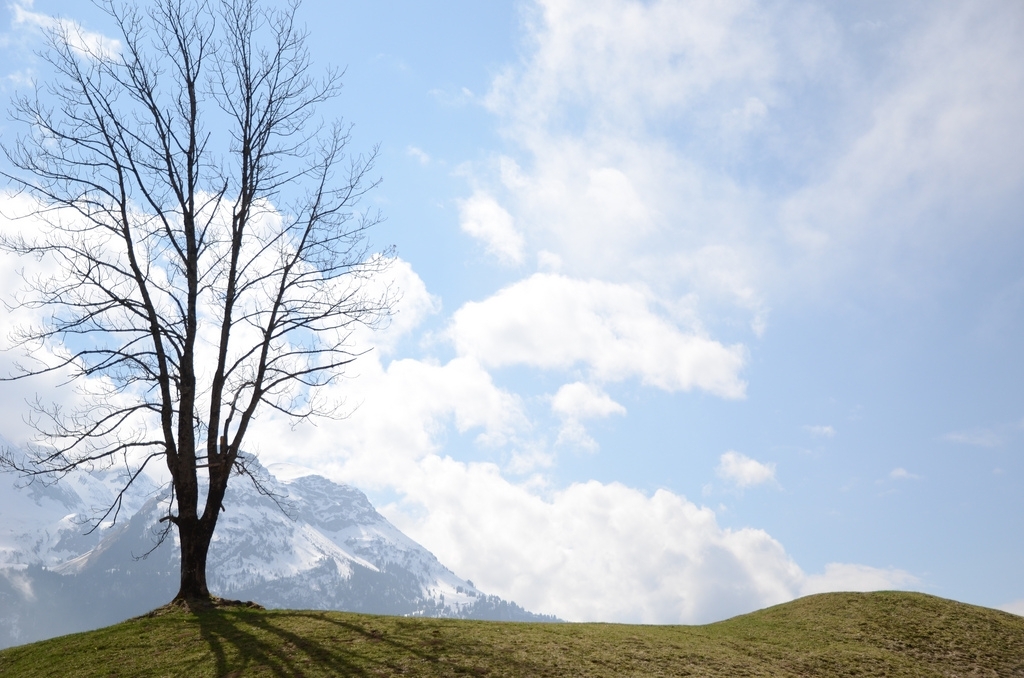}}%
\caption{%
Class ambiguity with a single label on Places 205 \cite{zhou2014learning}.
{\bfseries Labels:} Valley, Pasture, Mountain;
Ski resort, Chalet, Sky.
Note that multiple labels apply to each image
and $k$ guesses may be required to guess
the ground truth label correctly.
}\label{fig:teaser}%
\end{figure}

Moving forward,
we see top-$k$ classification as a natural transition step
between multiclass learning with a single label per training example
and multilabel learning with a complete set of relevant labels.
Multilabel learning forms the second part of this work,
where we introduce a smoothed version of
the multilabel SVM loss \cite{crammer2003family},
and contribute two novel projection algorithms for efficient optimization
of multilabel losses in the SDCA framework.
Furthermore, we compare all multiclass, top-$k$, and multilabel
methods in a novel experimental setting, where we want to quantify
the utility of multilabel annotation.
Specifically, we want to understand if it is possible to obtain
effective multilabel classifiers from single label annotation.

\bgroup
\def\arraystretch{1.4}
\begin{table*}[ht]\small\centering\setlength{\tabcolsep}{.45em}
\caption{Overview of the methods considered in this work and our contributions.}
\label{tbl:summary}
\begin{tabular}{l|l|l|c|c|c}
\toprule
Method & Name & Loss function & Conjugate & SDCA update & Top-$k$ calibrated
\\
\midrule
\midrule
\reftext{eq:ova-hinge} & One-vs-all (OVA) SVM &
$\max\{0, \, 1 - y f(x) \}$ &
\multirow{2}{*}{\cite{shalev2014accelerated}} &
\multirow{2}{*}{\cite{shalev2014accelerated}} &
no$^\dagger$ (Prop.~\ref{prop:calibrated-hinge})
\\
\reftext{eq:ova-lr} & OVA logistic regression &
$\log\big(1 + \exp(- y f(x)) \big)$ &
&
&
yes (Prop.~\ref{prop:calibrated-lr})
\\
\cline{1-6}
\reftext{eq:multi-hinge} & Multiclass SVM &
$\max \big\{0, (a + c)_{\pi_1} \big\}$ &
\cite{lapin2015topk,shalev2014accelerated} &
\cite{lapin2015topk,shalev2014accelerated} &
no (Prop.~\ref{prop:multi-hinge-topk-calibrated})
\\
\reftext{eq:softmax} & Softmax (cross entropy) &
$\log\big( \sum_{j \in \Yc} \exp(a_j) \big)$ &
Prop.~\ref{prop:softmax-conjugate} &
Prop.~\ref{prop:topk-entropy-update} &
yes (Prop.~\ref{prop:softmax-topk-calibrated})
\\
\cline{1-6}
\reftext{eq:topk-hinge-alpha} & Top-$k$ SVM ($\alpha$) &
$\max \big\{0, \frac{1}{k} \sum_{j=1}^k (a + c)_{\pi_j} \big\}$ &
\multirow{2}{*}{\cite{lapin2015topk}} &
\multirow{2}{*}{\cite{lapin2015topk}} &
\multirow{5}{*}{\multirowcell{open\\question}}
\\
\reftext{eq:topk-hinge-beta} & Top-$k$ SVM ($\beta$) &
$\frac{1}{k} \sum_{j=1}^k \max \big\{0, (a + c)_{\pi_j} \big\}$ &
&
\\
\cline{1-5}
\reftext{eq:smooth-topk-alpha} & Smooth top-$k$ SVM ($\alpha$) &
$L_{\gamma}$ in Prop.~\ref{prop:topk-smooth} w/ $\Ska$ &
\multirow{2}{*}{Prop.~\ref{prop:topk-smooth}} &
\multirow{2}{*}{Prop.~\ref{prop:smooth-topk-hinge-update}} &
\\
\SvmTopKbg{k}{\gamma} & Smooth top-$k$ SVM ($\beta$) &
$L_{\gamma}$ in Prop.~\ref{prop:topk-smooth} w/ $\Skb$ &
&
&
\\
\cline{1-5}
\reftext{eq:topk-entropy} & Top-$k$ entropy &
$L$ in Prop.~\ref{prop:topk-entropy-primal} &
Eq.~\ref{eq:softmax-conjugate} w/$\,\Ska$ &
Prop.~\ref{prop:topk-entropy-update} &
\\
\cline{6-6}
\reftext{eq:truncated-topk-entropy} & Truncated top-$k$ entropy &
$\log\big( 1 + \tsum_{j \in \Jc_y^k} \exp(a_j) \big)$ &
- &
- &
yes (Prop.~\ref{prop:truncated-topk-entropy-topk-calibrated})
\\
\midrule
\midrule
\reftext{eq:svm-ml} & Multilabel SVM &
$\max_{y \in Y, \, \bar{y} \in \bar{Y}}
\max \{0, 1 + u_{\bar{y}} - u_y \}$ &
Prop.~\ref{prop:svm-ml-conj} &
Prop.~\ref{prop:smooth-svm-ml-update} &
\multirow{3}{*}{\multirowcell{
see \eg \cite{gao2011consistency}\\
for multilabel \\
consistency}}
\\
\reftext{eq:smooth-svm-ml} & Smooth multilabel SVM &
$L_{\gamma}$ in Prop.~\ref{prop:smooth-svm-ml} &
Prop.~\ref{prop:smooth-svm-ml} &
Prop.~\ref{prop:smooth-svm-ml-update} &
\\
\reftext{eq:lr-ml} & Multilabel Softmax &
$\tfrac{1}{\abs{Y}} \tsum_{y \in Y}
\log\big( \tsum_{\bar{y} \in \Yc} \exp(u_{\bar{y}} - u_y) \big)$ &
Prop.~\ref{prop:lr-ml-conj} &
Prop.~\ref{prop:lr-ml-update} &
\\
\midrule
\midrule
\multicolumn{6}{c}{\multirowcell{
Let
$a \bydef (f_j(x) - f_y(x))_{j \in \Yc}$,
$c \bydef \ones - e_y$ (multiclass);
$u \bydef (f_y(x))_{y \in \Yc}$ (multilabel);
$\pi:$ $a_{\pi_1} \geq \ldots \geq a_{\pi_m}$;
$\Jc_y^k$ is defined in \S~\ref{sec:multiclass-methods}.
\\
\SvmMulti\ $\equiv$ \SvmTopKa{1} $\equiv$ \SvmTopKb{1};
\LrMulti\ $\equiv$ \LrTopK{1} $\equiv$ \LrTopKn{1}.
$^\dagger$Smooth \SvmOvag{\gamma} \emph{is} top-$k$ calibrated
(Prop.~\ref{prop:smooth-ova-hinge-calibrated}).
}}
\\\bottomrule
\end{tabular}
\end{table*}
\egroup

The contributions of this work are as follows.

\begin{itemize}
\item In \S~\ref{sec:related-work},
we provide an overview of the related work
and establish connections to a number of related research directions.
In particular, we point to an intimate link that exists between
top-$k$ classification, label ranking, and learning to rank
in information retrieval.

\item In \S~\ref{sec:loss-functions},
we introduce the learning problem for multiclass and multilabel classification,
and discuss the respective performance metrics.
We also propose $4$ novel loss functions for minimizing the top-$k$ error
and a novel smooth multilabel SVM loss.
A brief summary of the methods that we consider
is given in Table~\ref{tbl:summary}.

\item In \S~\ref{sec:topk-calibration},
we introduce the notion of top-$k$ calibration
and analyze which of the multiclass methods are calibrated
for the top-$k$ error.
In particular, we highlight
that the softmax loss is uniformly top-$k$ calibrated for all $k \geq 1$.

\item In \S~\ref{sec:optimization},
we develop efficient optimization schemes based on the SDCA framework.
Specifically, we contribute a set of algorithms for computing
the proximal maps that can be used to train classifiers
with the specified multiclass, top-$k$, and multilabel loss functions.

\item In \S~\ref{sec:experiments},
the methods are evaluated empirically in three different settings:
on synthetic data (\S~\ref{sec:synthetic}),
on multiclass datasets (\S~\ref{sec:multiclass-experiments}), and
on multilabel datasets (\S~\ref{sec:multilabel-experiments}).

\item In \S~\ref{sec:multiclass-experiments},
we perform a set of experiments on $11$ multiclass benchmarks including
the ImageNet 2012 \cite{ILSVRCarxiv14}
and the Places 205 \cite{zhou2014learning} datasets.
Our evaluation reveals, in particular, that the softmax loss
and the proposed smooth \SvmMultig{\gamma} loss
are competitive uniformly in all top-$k$ errors,
while improvements for a specific $k$
can be obtained with the new top-$k$ losses.

\item In \S~\ref{sec:multilabel-experiments},
we evaluate the multilabel methods on $10$ datasets
following \cite{madjarov2012extensive},
where our smooth multilabel \SvmMLg{\gamma}
shows particularly encouraging results.
Next, we perform experiments on
Pascal VOC 2007 \cite{everingham2010pascal}
and Microsoft COCO \cite{lin2014microsoft},
where we train multiclass and top-$k$ methods using only a single
label of the most prominent object per image,
and then compare their multilabel performance on test data
to that of multilabel methods trained with full annotation.
Surprisingly, we observe a gap of just above $2\%$ mAP on Pascal VOC
between the best multiclass and multilabel methods.
\end{itemize}

We release our implementation of SDCA-based solvers for training models
with the loss functions considered in this work\footnote{
\url{https://github.com/mlapin/libsdca}}.
We also publish code for the corresponding proximal maps,
which may be of independent interest.

\section{Related Work}%
\label{sec:related-work}
In this section, we place our work in a broad context
of related research directions.
First, we draw connections to the general problem of \emph{learning to rank}.
While it is mainly studied in the context of information search
and retrieval, there are clear ties to multiclass and multilabel classification.
Second, we briefly review related results on \emph{consistency}
and classification calibration.
These form the basis for our theoretical analysis of top-$k$ calibration.
Next, we focus on the technical side including the \emph{optimization method}
and the algorithms for efficient computation of proximal operators.
Finally, we consider multiclass and multilabel \emph{image classification},
which are the main running examples in this paper.

\textbf{Learning to rank.}
Learning to rank is a supervised learning problem that arises
whenever the structure in the output space admits a partial order
\cite{tsochantaridis2005large}.
The classic example is ranking in information retrieval (IR),
see \eg \cite{liu2009learning} for a recent review.
There, a feature vector $\Phi(q, d)$
is computed for every query $q$ and every document $d$,
and the task is to learn a model that ranks the relevant documents
for the given query before the irrelevant ones.
Three main approaches are recognized within that framework:
the pointwise, the pairwise, and the listwise approach.
Pointwise methods cast the problem of predicting document relevance
as a regression \cite{cossock2006subset}
or a classification \cite{li2007mcrank} problem.
Instead, the pairwise approach is focused on predicting the relative
order between documents
\cite{burges2005learning,freund2003efficient,joachims2002optimizing}.
Finally, the listwise methods attempt to optimize a given
performance measure directly on the full list of documents
\cite{taylor2008softrank,yue2007support,xu2007adarank},
or propose a loss function on the predicted and the ground truth
lists \cite{cao2007learning,xia2008listwise}.

Different from ranking in IR, our main interest in this work
is \emph{label ranking} which generalizes the basic binary classification
problem to multiclass, multilabel, and even hierarchical classification,
see \cite{vembu2010label} for a survey.
A link between the two settings is established if we consider
queries to be examples (\eg images) and documents to be class labels.
The main contrast, however, is in
the employed loss functions and performance evaluation at test time
(\S~\ref{sec:loss-functions}).

Most related to our work is a general family
of convex loss functions for ranking and classification
introduced by Usunier \etal \cite{usunier2009ranking}.
One of the loss functions that we consider
(\SvmTopKb{k} \cite{lapin2015topk}) is a member of that family.
Another example is \textsc{Wsabie} \cite{Weston2011,Gupta2014},
which learns a joint embedding model
optimizing an approximation of a loss from \cite{usunier2009ranking}.

Top-$k$ classification in our setting is directly related to label ranking
as the task is to place the ground truth label in the set of top $k$
labels as measured by their prediction scores.
An alternative approach is suggested
by \cite{mcauley2013optimization}
who use structured learning
to aggregate the outputs of pre-trained one-vs-all binary classifiers
and directly predict a set of $k$ labels,
where the labels missing from the annotation are modelled
with latent variables.
That line of work is pursued further in \cite{xu2016image}.
The task of predicting a set of items
is also considered in \cite{ross2013learning},
who frame it as a problem of maximizing a submodular reward function.
A probabilistic model for ranking and top-$k$ classification
is proposed by \cite{swersky2012probabilistic},
while \cite{guillaumin2009tagprop,mensink2013distance}
use metric learning to train a nearest neighbor model.
An interesting setting related to top-$k$ classification
is learning with positive and unlabeled data
\cite{kanehira2016multilabel,du2014analysis},
where the absence of a label does not imply it is a negative label,
and also learning with label noise
\cite{frenay2014classification,liu2016classification}.

Label ranking is closely related to multilabel classification
\cite{madjarov2012extensive,zhang2014review},
which we consider later in this paper,
and to tag ranking \cite{wang2012assistive}.
Ranking objectives have been also considered for training convolutional
architectures \cite{gong2013deep}, most notably with a loss
on triplets \cite{wang2014learning,zhao2015deep},
that consideres both positive and negative examples.
Many recent works focus on the top of the ranked list
\cite{agarwal2011infinite,boyd2012accuracy,
ICML2012Rakotomamonjy_664,rudin2009p,li2014top}.
However, they are mainly interested in search and retrieval,
where the number of relevant documents by far exceeds
what users are willing to consider.
That setting suggests a different trade-off
for recall and precision compared to our setting
with only a few relevant labels.
This is correspondingly reflected in performance evaluation,
as mentioned above.

\textbf{Consistency and calibration.}
Classification is a discrete prediction problem where minimizing
the expected (0-1) error is known to be computationally hard.
Instead, it is common to minimize a \emph{surrogate} loss
that leads to efficient learning algorithms.
An important question, however, is whether the minimizers of
the expected surrogate loss also minimize the expected error.
Loss functions which have that property are called \emph{calibrated}
or \emph{consistent} with respect to the given discrete loss.
Consistency in binary classification is well understood
\cite{BarJorAuc2006,zhang2004statistical,reid2010composite},
and significant progress has been made in the analysis of
multiclass \cite{tewari2007consistency,pires2016multiclass,vernet2011composite},
multilabel \cite{gao2011consistency,koyejo2015consistent},
and ranking \cite{cossock2008statistical,duchi2010consistency,calauzenes2012non}
methods.
In this work, we investigate calibration of a number of
surrogate losses with respect to the top-$k$ error,
which generalizes previously established results for multiclass methods.

\textbf{Optimization.}
To facilitate experimental evaluation of the proposed loss functions,
we also implement the corresponding optimization routines.
We choose the stochastic dual coordinate ascent (SDCA)
framework of \cite{shalev2013stochastic} for its ease of implementation,
strong convergence guarantees, and the possibility to compute
certificates of optimality via the duality gap.
While \cite{shalev2013stochastic} describe the general SDCA algorithm
that we implement, their analysis is limited to
\emph{scalar} loss functions (both Lipschitz and smooth)
with $\ell_2$ regularization,
which is only suitable for binary problems.
A more recent work \cite{shalev2014accelerated}
extends the analysis to \emph{vector valued}
smooth (or Lipschitz) functions and general strongly convex regularizers,
which is better suited to our multiclass and multilabel loss functions.
A detailed comparison of recent coordinate descent algorithms
is given in \cite{shalev2014accelerated,fercoq2015accelerated}.

Following \cite{shalev2006efficient} and \cite{shalev2014accelerated},
the main step in the optimization algorithm updates the dual variables
by computing a projection or, more generally, the proximal operator
\cite{parikh2014proximal}.
The proximal operators that we consider here
can be equivalently expressed as instances of a
continuous nonlinear resource allocation problem,
which has a long research history,
see \cite{patriksson2015algorithms} for a recent survey.
Most related to our setting is the Euclidean projection onto
the unit simplex or the $\ell_1$-ball in $\Rb^n$,
which can be computed approximately via bisection
in $O(n)$ time \cite{liu2009efficient},
or exactly via breakpoint searching
\cite{kiwiel2008breakpoint}
and variable fixing \cite{kiwiel2008variable}.
The former can be done in $O(n \log n)$ time
with a simple implementation based on sorting,
or in $O(n)$ time with an efficient median finding algorithm.
In this work, we choose the variable fixing scheme
which does not require sorting and is easy to implement.
Although its complexity is $O(n^2)$ on pathological inputs
with elements growing exponentially \cite{condat2014fast},
the observed complexity in practice is linear
and is competitive with breakpoint searching algorithms
\cite{condat2014fast,kiwiel2008variable}.

While there exist efficient projection algorithms for optimizing
the SVM hinge loss and its descendants, the situation is a bit
more complicated for logistic regression, both binary and multiclass.
There exists no analytical solution for an update with the logistic loss,
and \cite{shalev2014accelerated} suggest a formula in the binary case
which computes an approximate update in closed form.
Multiclass logistic (softmax) loss is optimized
in the SPAMS toolbox \cite{mairal2010network},
which implements FISTA \cite{beck2009fast}.
Alternative optimization methods are considered
in \cite{yu2011dual} who also propose a two-level coordinate
descent method in the multiclass case.
Different from these works, we propose to follow closely
the same variable fixing scheme that is used for SVM training
and use the Lambert $W$ function \cite{corless1996lambertw}
in the resulting entropic proximal map.
Our runtime compares favourably with SPAMS,
as we show in \S~\ref{sec:multiclass-experiments}.

\textbf{Image classification.}
Multiclass and multilabel image classification are the main
applications that we consider in this work to evaluate
the proposed loss functions.
We employ a relatively simple image recognition pipeline
following \cite{simonyan2014very},
where feature vectors are extracted from
a convolutional neural network (ConvNet),
such as the VGGNet \cite{simonyan2014very}
or the ResNet \cite{he2016deep},
and are then used to train a linear classifier
with the different loss functions.
The ConvNets that we use are pre-trained on the large scale
ImageNet \cite{ILSVRCarxiv14} dataset,
where there is a large number of object categories ($1000$),
but relatively little variation in scale and location
of the central object.
For scene recognition,
we also use a VGGNet-like architecture \cite{wang2015places}
that was trained on the Places 205 \cite{zhou2014learning} dataset.

Despite the differences between the benchmarks \cite{torralba2011unbiased},
image representations learned by ConvNets on large datasets
have been observed to transfer well \cite{oquab2014learning,razavian2014cnn}.
We follow that scheme in single-label experiments,
\eg when recognizing birds \cite{wah2011caltech}
and flowers \cite{nilsback2008automated}
using a network trained on ImageNet,
or when transferring knowledge in scene recognition
\cite{quattoni2009recognizing,xiao2010sun}.
However, moving on to multi-label classification on
Pascal VOC \cite{everingham2010pascal}
and Microsoft COCO \cite{lin2014microsoft},
we need to account for increased variation in scale and object placement.

While the earlier works ignore explicit search for object location
\cite{chatfield2014return,zeiler2014visualizing},
or require bounding box annotation
\cite{ren2015faster,sermanet2013overfeat,zhao2016regional},
recent results indicate that effective classifiers
for images with multiple objects in cluttered scenes
can be trained from weak image-level annotation
by explicitly searching over multiple scales and locations
\cite{oquab2015object,wei2016hcp,wei2016region,wang2016cnn,wang2016beyond}.
Our multilabel setup follows closely the pipeline of \cite{wei2016region}
with a few exceptions detailed in \S~\ref{sec:multilabel-experiments}.

\section{Loss Functions for Classification}%
\label{sec:loss-functions}
When choosing a loss function, one may want to consider several aspects.
First, at the basic level, the loss function depends on %
the available annotation and the performance metric one is interested in,
\eg we distinguish between (single label) multiclass
and multilabel losses in this work.
Next, there are two fundamental factors that control
the \emph{statistical} and the \emph{computational} behavior of learning.
For computational reasons, we work with convex surrogate losses
rather than with the performance metric directly.
In that context, a relevant distinction is between
the nonsmooth Lipschitz functions (\SvmMulti, \SvmTopK{k})
and the smooth functions (\LrMulti, \SvmMultig{\gamma}, \SvmTopKg{k}{\gamma})
with strongly convex conjugates that lead to faster convergence rates.
From the statistical perspective, it is important to understand
if the surrogate loss is classification calibrated as it is
an attractive asymptotic property that leads to Bayes consistent classifiers.
Finally, one may exploit duality and introduce modifications to
the conjugates of existing functions that have desirable effects
on the primal loss (\LrTopK{k}).

The rest of this section covers the technical background
that is used later in the paper.
We discuss our notation, introduce multiclass and multilabel
classification, recall the standard approaches to classification,
and introduce our recently proposed methods for top-$k$ error minimization.

In \S~\ref{sec:perf-metrics}, we discuss multiclass and multilabel
performance evaluation measures that are used later in our experiments.
In \S~\ref{sec:multiclass-methods}, we review established
multiclass approaches and introduce our novel top-$k$ loss functions;
we also recall Moreau-Yosida regularization as a smoothing technique
and compute convex conjugates for SDCA optimization.
In \S~\ref{sec:multilabel-methods}, we discuss multilabel classification
methods, introduce the smooth multilabel SVM,
and compute the corresponding convex conjugates.
To enhance readability, we defer all the proofs to the appendix.

\textbf{Notation.}
We consider classification problems with a predefined set of $m$ classes.
We begin with \emph{multiclass} classification,
where every example $x_i \in \Xc$
has exactly \emph{one} label $y_i \in \Yc \bydef \{ 1, \ldots, m \}$,
and later generalize to the \emph{multilabel} setting,
where each example is associated with a \emph{set} of labels
$Y_i \subset \Yc$.
In this work, a classifier is a function $f: \Xc \rightarrow \R^m$
that induces a ranking of class labels via the prediction scores
$f(x) = \big( f_y(x) \big)_{y \in \Yc}$.
In the linear case, each predictor $f_y$ has
the form $f_y(x) = \inner{w_y, x}$,
where $w_y \in \Rb^d$ is the parameter to be learned.
We stack the individual parameters into a weight matrix
$W \in \R^{d \times m}$, so that $f(x) = \tra{W} x$.
While we focus on linear classifiers with $\Xc \equiv \Rb^d$
in the exposition below and in most of our experiments,
all loss functions are formulated in the general setting
where the kernel trick \cite{scholkopf2002learning}
can be employed to construct nonlinear decision surfaces.
In fact, we have a number of experiments with the RBF kernel as well.

At test time, prediction depends on the evaluation metric
and generally involves sorting / producing the top-$k$
highest scoring class labels in the multiclass setting,
and predicting the labels that score above a certain threshold $\delta$
in multilabel classification.
We come back to performance metrics shortly.

We use $\pi$ and $\tau$ to denote permutations of (indexes) $\Yc$.
Unless stated otherwise, $a_\pi$ reorders components of a vector $a$
in descending order,
$ a_{\pi_1} \geq a_{\pi_2} \geq \ldots \geq a_{\pi_m} $.
Therefore, for example, $a_{\pi_1} = \max_j a_j$.
If necessary, we make it clear which vector is being sorted
by writing $\pi(a)$ to mean $\pi(a) \in \argsort a$
and let $\pi_{1:k}(a) \bydef \{ \pi_1(a), \ldots, \pi_k(a) \}$.
We also use the Iverson bracket defined as $\iv{P} = 1$ if $P$ is true
and $0$ otherwise;
and introduce a shorthand for the conditional probability
$p_y(x) \bydef \Pr(Y=y \given X=x)$.
Finally, we let $\wo{a}{y}$ be obtained
by removing the $y$-th coordinate from $a$.

We consider $\ell_2$-regularized objectives in this work,
so that if $L : \Yc \times \R^m \rightarrow \R_+$ is a multiclass loss
and $\lambda > 0$ is a regularization parameter,
classifier training amounts to solving
$
\min_W \frac{1}{n}\sum_{i=1}^n L(y_i, \tra{W} x_i) + \lambda \norm{W}_F^2 .
$
Binary and multilabel classification problems only differ in the loss $L$.

\subsection{Performance Metrics}\label{sec:perf-metrics}

Here, we briefly review performance evaluation metrics employed
in multiclass and multilabel classification.

\textbf{Multiclass.}
A standard performance measure for classification problems is the zero-one loss,
which simply counts the number of classification mistakes
\cite{friedman2001elements, duda2012pattern}.
While that metric is well understood and
inspired such popular surrogate losses as the SVM hinge loss,
it naturally becomes more stringent as the number of classes increases.
An alternative to the standard zero-one error
is to allow $k$ guesses instead of one.
Formally, the \textbf{top-$k$ zero-one loss} (\textbf{top-$k$ error}) is
\begin{align}\label{eq:topk-error} 
\kerr{y, f(x)} \bydef \iv{f_{\pi_k}(x) > f_y(x)} .
\end{align}
That is, we count a mistake if the ground truth label $y$ scores below
$k$ other class labels.
Note that for $k=1$ we recover the standard zero-one error.
\textbf{Top-$k$ accuracy} is defined as $1$ minus the top-$k$ error,
and performance on the full test sample is computed as the mean across
all test examples.

\textbf{Multilabel.}
Several groups of multilabel evaluation metrics are established
in the literature and it is generally suggested that
multiple contrasting measures should be reported to avoid skewed results.
Here, we give a brief overview of the metrics that we report
and refer the interested reader to
\cite{madjarov2012extensive,zhang2014review,koyejo2015consistent},
where multilabel metrics are discussed in more detail.

\textsl{Ranking based.}
This group of performance measures compares the ranking of the labels
induced by $f_y(x)$ to the ground truth ranking.
We report the \textbf{rank loss} defined as
$$
\RLoss(f) = \tfrac{1}{n} \tsum_{i=1}^n
\abs{D_i} / (\abs{Y_i} \abs{\bar{Y}_i}) ,
$$
where
$D_i = \{ (y, \bar{y}) \given f_y(x_i) \leq f_{\bar{y}}(x_i), \;
(y, \bar{y}) \in Y_i \times \bar{Y}_i \}$
is the set of reversely ordered pairs,
and $\bar{Y}_i \bydef \Yc \setminus Y_i$ is the complement of $Y_i$.
This is the loss that is implicitly optimized by all
multiclass / multilabel loss functions that we consider
since they induce a penalty when $f_{\bar{y}}(x_i) - f_y(x_i) > 0$.

Ranking class labels for a given image is similar to
ranking documents for a user query in information retrieval
\cite{liu2009learning}.
While there are many established metrics \cite{manning2008introduction},
a popular measure that is relevant to our discussion is
\textbf{precision-at-$k$} (\PrecAtK{k}),
which is the fraction of relevant items within the top $k$ retrieved
\cite{joachims2005support,mcfee2010metric}.
Although this measure makes perfect sense when $k \ll \abs{Y_i}$,
\ie there are many more relevant documents than we possibly want to examine,
it is not very useful when there are only a few correct
labels per image -- once all the relevant labels are in the top $k$ list,
\PrecAtK{k} starts to decrease as $k$ increases.
A better alternative in our multilabel setting is
a complementary measure, \textbf{recall-at-$k$},
defined as
$$
\RecAtK{k}(f) = \tfrac{1}{n} \tsum_{i=1}^n
\big( \pi_{1:k}(f(x_i)) \cap \abs{Y_i} \big) / \abs{Y_i} ,
$$
which measures the fraction of relevant labels in the top $k$ list.
Note that \RecAtK{k} is a natural generalization of the top-$k$ error
to the multilabel setting and coincides with that multiclass metric
whenever $Y_i$ is singleton.

Finally, we report the standard Pascal VOC \cite{everingham2010pascal}
performance measure, mean average precision (\textbf{mAP}),
which is computed as the one-vs-all AP averaged over all classes.

\textsl{Partition based.}
In contrast to ranking evaluation,
partition based measures assess the quality of the actual
multilabel prediction which requires
a cut-off \textbf{threshold} $\delta \in \Rb$.
Several threshold selection strategies have been proposed in the literature:
(i) setting a constant threshold prior to experiments \cite{dembczynski2010bayes};
(ii) selecting a threshold \emph{a posteriori}
by matching label cardinality \cite{read2009classifier};
(iii) tuning the threshold on a validation set
\cite{koyejo2015consistent,yang1999evaluation};
(iv) learning a regression function \cite{elisseeff2001kernel};
(v) bypassing threshold selection altogether by introducing
a (dummy) calibration label \cite{furnkranz2008multilabel}.
We have experimented with options (ii) and (iii),
as discussed in \S~\ref{sec:multilabel-experiments}.

Let $h(x) \bydef \{ y \in \Yc \given f_y(x) \geq \delta \}$
be the set of predicted labels for a given threshold $\delta$,
and let
\begin{align*}
\tp_{i,j} &= \iv{j \in h(x_i), j \in Y_i}, &
\tn_{i,j} &= \iv{j \notin h(x_i), j \notin Y_i}, \\
\fp_{i,j} &= \iv{j \in h(x_i), j \notin Y_i}, &
\fn_{i,j} &= \iv{j \notin h(x_i), j \in Y_i},
\end{align*}
be a set of $m \cdot n$ primitives defined as in \cite{koyejo2015consistent}.
Now, one can use any performance measure $\Psi$
that is based on the binary confusion matrix,
but, depending on where the averaging occurs,
the following three groups of metrics are recognized.

\noindent
\textbf{Instance-averaging.}
The binary metrics are computed on the averages over labels
and then averaged across examples:
$$
\Psi^{\rm inst}(h) = \tfrac{1}{n} \tsum_{i=1}^n \Psi\big(
\tfrac{1}{m} \tsum_{j=1}^m \tp_{i,j}, \ldots,
\tfrac{1}{m} \tsum_{j=1}^m \fn_{i,j} \big) .
$$

\noindent
\textbf{Macro-averaging.}
The metrics are averaged across labels:
$$
\Psi^{\rm mac}(h) = \tfrac{1}{m} \tsum_{j=1}^m \Psi\big(
\tfrac{1}{n} \tsum_{i=1}^n \tp_{i,j}, \ldots,
\tfrac{1}{n} \tsum_{i=1}^n \fn_{i,j} \big) .
$$

\noindent
\textbf{Micro-averaging.}
The metric is applied on the averages over both labels and examples:
$$
\Psi^{\rm mic}(h) = \Psi\big(
\tfrac{1}{mn} \tsum_{i,j} \tp_{i,j}, \ldots,
\tfrac{1}{mn} \tsum_{i,j} \fn_{i,j} \big) .
$$
Following \cite{madjarov2012extensive}, we consider
the \textbf{$\mathbf{F_1}$ score} as the binary metric $\Psi$
with all three types of averaging.
We also report multilabel \textbf{accuracy}, \textbf{subset accuracy},
and the \textbf{hamming loss} defined respectively as
\begin{align*}
\Acc(h) &= \tfrac{1}{n} \tsum_{i=1}^n
(\abs{h(x_i) \cap Y_i}) / (\abs{h(x_i) \cup Y_i}) , \\
\SAcc(h) &= \tfrac{1}{n} \tsum_{i=1}^n
\iv{h(x_i) = Y_i} , \\
\HLoss(h) &= \tfrac{1}{mn} \tsum_{i=1}^n \abs{h(x_i) \triangle Y_i} ,
\end{align*}
where $\triangle$ is the symmetric set difference.

\subsection{Multiclass Methods}\label{sec:multiclass-methods}

In this section, we switch from performance evaluation at test time
to how the quality of a classifier is measured during training.
In particular, we introduce the loss functions
used in established multiclass methods
as well as our novel loss functions for optimizing the top-$k$ error
\eqref{eq:topk-error}.

\textbf{OVA.}
A multiclass problem is often solved using the one-vs-all (OVA)
reduction to $m$ independent binary classification problems.
Every class is trained versus the rest
which yields $m$ classifiers $\{f_y\}_{y \in \Yc}$.
Typically, each classifier $f_y$ is trained
with a convex margin-based loss function $L(\tilde{y} f_y(x))$,
where $L: \Rb \rightarrow \Rb_+$, $\tilde{y} = \pm 1$.
Simplifying the notation, we consider
\begin{align}
L(yf(x)) &= \max\{0, \, 1 - y f(x)\} ,
\tag{\SvmOva}\label{eq:ova-hinge} \\
L(yf(x)) &= \log(1 + e^{- y f(x)} ) .
\tag{\LrOva}\label{eq:ova-lr}
\end{align}
The hinge \eqreftext{eq:ova-hinge} and
logistic \eqreftext{eq:ova-lr} losses correspond to
the SVM and logistic regression methods respectively.

\textbf{Multiclass.}
An alternative to the OVA scheme above
is to use a \emph{multiclass} loss
$L: \Yc \times \Rb^m \rightarrow \Rb_+$ directly.
All multiclass losses that we consider only depend on pairwise
differences between the ground truth score $f_y(x)$
and all the other scores $f_j(x)$.
Loss functions from the SVM family additionally require a \emph{margin}
$\Delta(y,j)$, which can be interpreted as a distance
in the label space \cite{tsochantaridis2005large}
between $y$ and $j$.
To simplify the notation, we use vectors $a$ (for the differences)
and $c$ (for the margin) defined for a given $(x,y)$ pair as
\begin{align*}
a_j \bydef f_j(x) - f_y(x), \;
c_j \bydef 1 - \iv{y = j} , \;
j = 1, \ldots, m.
\end{align*}
We also write $L(a)$ instead of the full $L(y,f(x))$.

We consider two generalizations of \reftext{eq:ova-hinge} and
\reftext{eq:ova-lr}:
\begin{align}
&L(a) = \max_{j \in \Yc} \{ a_j + c_j \} ,
\tag{\SvmMulti}\label{eq:multi-hinge} \\
&L(a) = \log\big( \tsum_{j \in \Yc} \exp(a_j) \big) .
\tag{\LrMulti}\label{eq:softmax}
\end{align}
Both the multiclass SVM loss \eqreftext{eq:multi-hinge}
of \cite{crammer2001algorithmic}
and the softmax loss \eqreftext{eq:softmax}
are common in multiclass problems.
The latter is particularly popular in deep architectures
\cite{bengio2009learning, krizhevsky2012imagenet, simonyan2014very},
while \reftext{eq:multi-hinge} is also competitive
in large-scale image classification \cite{akata2014good}.

The OVA and multiclass methods were designed
with the goal of minimizing the standard zero-one loss.
Now, if we consider the top-$k$ error \eqref{eq:topk-error}
which does not penalize $(k-1)$ mistakes,
we discover that convexity of the above losses
leads to phenomena where $\kerr{y, f(x)} = 0$,
but $L(y,f(x)) \gg 0$.
That happens, for example, when
$f_{\pi_1}(x) \gg f_y(x) \geq f_{\pi_k}(x)$,
and creates a bias if we are working with
rigid function classes such as linear classifiers.
Next, we introduce loss functions
that are modifications of the above losses
with the goal of alleviating that phenomenon.

\textbf{Top-$k$ SVM.}
Recently, we introduced Top-$k$ Multiclass SVM \cite{lapin2015topk},
where two modifications of the multiclass hinge loss \eqreftext{eq:multi-hinge}
were proposed.
The first version ($\alpha$) is motivated directly by the top-$k$ error
while the second version ($\beta$) falls into a general family of ranking losses
introduced earlier by Usunier \etal~\cite{usunier2009ranking}.
The two top-$k$ SVM losses are
\begin{align}
L(a) &= \max \big\{0, \tfrac{1}{k} \tsum_{j=1}^k
(a + c)_{\pi_j} \big\} ,
\tag{\SvmTopKa{k}}\label{eq:topk-hinge-alpha} \\
L(a) &= \tfrac{1}{k} \tsum_{j=1}^k \max \big\{0,
(a + c)_{\pi_j} \big\} ,
\tag{\SvmTopKb{k}}\label{eq:topk-hinge-beta}
\end{align}
where $\pi$ reorders the components of $(a + c)$ in descending order.
We show in \cite{lapin2015topk} that \reftext{eq:topk-hinge-alpha}
offers a tighter upper bound on the top-$k$ error than
\reftext{eq:topk-hinge-beta}.
However, both losses perform similarly in our experiments
with only a small advantage of \reftext{eq:topk-hinge-beta}
in some settings.
Therefore, when the distinction is not important,
we simply refer to them as the top-$k$ hinge
or the top-$k$ SVM loss.
Note that they both reduce to \reftext{eq:multi-hinge} for $k=1$.

Top-$k$ SVM losses are not smooth which has implications
for their optimization (\S~\ref{sec:optimization})
and top-$k$ calibration (\S~\ref{sec:topk-calibration-multiclass}).
Following \cite{shalev2014accelerated},
who employed Moreau-Yosida regularization
\cite{BecTeb2012, nesterov2005smooth}
to obtain a smoothed version of
the binary hinge loss \eqreftext{eq:ova-hinge},
we applied the same technique in \cite{lapin2016loss}
and introduced smooth top-$k$ SVM.

\textbf{Moreau-Yosida regularization.}
We follow \cite{parikh2014proximal}
and give the main points here for completeness.
The \emph{Moreau envelope} or \emph{Moreau-Yosida regularization}
$M_f$ of the function $f$ is
$$
M_f(v) \bydef \inf_x \big( f(x) + (1/2) \norms{x-v}_2 \big).
$$
It is a smoothed or regularized form of $f$ with the following nice properties:
it is continuously differentiable on $\Rb^d$, even if $f$ is not,
and the sets of minimizers of $f$ and $M_f$ are the same\footnote{
That does not imply that we get the same classifiers
since we are minimizing a regularized sum
of individually smoothed loss terms.}.
To compute a smoothed top-$k$ hinge loss, we use
$$
M_f = \big( f^* + (1/2)\norms{\,\cdot\,}_2 \big)^*,
$$
where $f^*$ is the convex conjugate\footnote{
The \textbf{convex conjugate} of $f$ is
$f^*(x^*) = \sup_x \{ \inner{x^*, x} - f(x) \}$.
}
of $f$.
A classical result in convex analysis \cite{HirLem2001}
states that a conjugate of a strongly convex function
has Lipschitz smooth gradient, therefore,
$M_f$ is indeed a smooth function.

\textbf{Top-$k$ hinge conjugate.}
Here, we compute the conjugates of the top-$k$ hinge losses
$\alpha$ and $\beta$.
As we show in \cite{lapin2015topk},
their effective domains\footnote{
The \textbf{effective domain} of $f$ is
$\dom f = \{x \in X \given f(x) < +\infty \}$.
}
are given by the \textbf{top-$k$ simplex}
($\alpha$ and $\beta$ respectively)
of radius $r$ defined as
\begin{align}
\Ska(r) &\bydef
\big\{ x \given \inner{\ones, x} \leq r, \;
0 \leq x_i \leq \tfrac{1}{k} \inner{\ones, x}, \; \forall i \big\} ,
\label{eq:topk-simplex-alpha} \\
\Skb(r) &\bydef
\big\{ x \given \inner{\ones, x} \leq r, \;
0 \leq x_i \leq \tfrac{1}{k} r , \; \forall i \big\} .
\label{eq:topk-simplex-beta}
\end{align}
We let $\Ska = \Ska(1)$, $\Skb = \Skb(1)$,
and note the relation
$
\Ska \subset \Skb \subset \Sm,
$
where
$\Sm = \big\{ x \given \inner{\ones, x} \leq 1, \; x_i \geq 0 \big\}$
is the unit simplex and the inclusions are
proper for $k > 1$, while for $k=1$ all three sets coincide.

\begin{proposition}[\!\!\cite{lapin2015topk}]\label{prop:topk-conj}
The convex conjugate of \reftext{eq:topk-hinge-alpha} is
\begin{align*}
L^*(v) &= \begin{cases}
- \tsum_{j \neq y} v_j &
\text{if } \inner{\ones, v} = 0 \text{ and } \wo{v}{y} \in \Ska, \\
+ \infty & \text{otherwise}.
\end{cases}
\end{align*}
The conjugate of \reftext{eq:topk-hinge-beta} is defined
in the same way, but with the set $\Skb$ instead of $\Ska$.
\end{proposition}

Note that the conjugates of both top-$k$ SVM losses
coincide and are equal to the conjugate of
the \reftext{eq:multi-hinge} loss
with the exception of their effective domains,
which are $\Ska$, $\Skb$, and $\Sm$ respectively.
As becomes evident in \S~\ref{sec:optimization},
the effective domain of the conjugate is
the feasible set for the dual variables.
Therefore, as we move from \reftext{eq:multi-hinge}
to \reftext{eq:topk-hinge-beta},
to \reftext{eq:topk-hinge-alpha},
we introduce more and more constraints on the dual variables
thus limiting the extent to which a single training example
can influence the classifier.

\textbf{Smooth top-$k$ SVM.}
We apply the smoothing technique introduced above
to \reftext{eq:topk-hinge-alpha}.
Smoothing of \reftext{eq:topk-hinge-beta} is done similarly,
but the set $\Ska(r)$ is replaced with $\Skb(r)$.

\ifprop
\begin{proposition}\label{prop:topk-smooth}
Let $\gamma > 0$ be the smoothing parameter.
The smooth top-$k$ hinge loss ($\alpha$) and its conjugate are
\begin{align}
L_{\gamma}(a) &= \tfrac{1}{\gamma} \big(
\innern{\wo{(a + c)}{y}, p} - \tfrac{1}{2} \normn{p}^2 \big) ,
\tag{\SvmTopKag{k}{\gamma}}\label{eq:smooth-topk-alpha} \\
L_{\gamma}^*(v) &= \begin{cases}
\tfrac{\gamma}{2} \normn{\wo{v}{y}}^2 - \innern{\wo{v}{y}, \wo{c}{y}} &
\text{if } \inner{\ones, v} = 0, \; \wo{v}{y} \in \Ska, \\
+\infty & \text{otherwise},
\end{cases}\notag%
\end{align}
where
$p = \proj_{\Ska(\gamma)}\wo{(a + c)}{y}$
is the Euclidean projection of $\wo{(a + c)}{y}$ onto $\Ska(\gamma)$.
Moreover,
$L_{\gamma}(a)$ is $1/\gamma$-smooth.
\end{proposition}
\fi
\ifproof
\begin{proof}
We take the convex conjugate of the top-$k$ hinge loss,
which was derived in \cite[Proposition~2]{lapin2015topk},
and add a regularizer $\frac{\gamma}{2} \inner{v,v}$
to obtain the $\gamma$-strongly convex conjugate loss $L_{\gamma}^*(v)$.
Note that since $v_y=-\sum_{j \neq y} v_j$ and $a_y = f_y(x) - f_y(x) = 0$,
we only need to work with $(m-1)$-dimensional vectors where the $y$-th
coordinate is removed.
The primal loss $L_{\gamma}(a)$,
obtained as the convex conjugate of
$L_{\gamma}^*(v)$, is $1/\gamma$-smooth
due to a known result in convex analysis \cite{HirLem2001}
(see also \cite[Lemma~2]{shalev2014accelerated}).
We now derive a formula to compute it based on the Euclidean
projection onto the top-$k$ simplex.
By definition,
\begin{align*}
L_{\gamma}(a) &=
\sup_{v' \in \Rb^m} \{ \inner{a,v'} - L_{\gamma}^*(v') \} \\
&=
\max_{v \in \Ska(1)}
\big\{ \innern{\wo{a}{y},v} - \tfrac{\gamma}{2} \inner{v,v}
+ \innern{v, \wo{c}{y}} \big\} \\
&=
- \tfrac{1}{\gamma} \min_{\frac{v}{\gamma} \in \Ska(1)}
\big\{ \tfrac{1}{2} \inner{v, v} - \innern{\wo{(a + c)}{y}, v} \big\} .
\end{align*}
For the constraint $\frac{v}{\gamma} \in \Ska(1)$, we have
\begin{align*}
\inner{\ones, v / \gamma} &\leq 1, &
0 &\leq v_i / \gamma \leq
\tfrac{1}{k} \inner{\ones, v / \gamma} \;
\Longleftrightarrow \\
\inner{\ones, v } &\leq \gamma, &
0 &\leq v_i \leq
\tfrac{1}{k} \inner{\ones, v} \;
\Longleftrightarrow 
\; v \in \Ska(\gamma) .
\end{align*}
The final expression follows from the fact that
\begin{align*}
&\argmin_{v \in \Ska(\gamma)}
\big\{ \tfrac{1}{2} \inner{v, v} - \innern{\wo{(a + c)}{y}, v} \big\} \\
&\equiv \argmin_{v \in \Ska(\gamma)} \normn{\wo{(a + c)}{y} - v}^2 
\equiv \proj_{\Ska(\gamma)}\wo{(a + c)}{y} .
\end{align*}
\end{proof}
\fi

While there is no analytic formula for the
\reftext{eq:smooth-topk-alpha} loss,
it can be computed efficiently via the projection
onto the top-$k$ simplex \cite{lapin2015topk}.
We can also compute its gradient as
\begin{align*}
\nabla L_{\gamma}(a) &= (1/\gamma)
\big( \Id_y - e_y \ones_y^\top \big)
\proj_{\Ska(\gamma)}\wo{(a + c)}{y} ,
\end{align*}
where $\Id_y$ is the identity matrix w/o the $y$-th column,
$e_y$ is the $y$-th standard basis vector,
and $\ones_y$ is the $(m-1)$-dimensional vector of all ones.
This follows from the definition of $a$, the fact that $L_{\gamma}(a)$
can be written as $\tfrac{1}{2 \gamma}(\norms{x} - \norms{x - p})$
for $x = \wo{(a + c)}{y}$ and $p = \proj_{\Ska(\gamma)}(x)$,
and a known result
\cite{borwein2000convex}
which says that
$\nabla_x \tfrac{1}{2}\norms{x - \proj_C(x)} = x - \proj_C(x)$
for any closed convex set $C$.

\textbf{Smooth multiclass SVM (\SvmMultig{\gamma}).}
We also highlight an important special case of \reftext{eq:smooth-topk-alpha}
that performed remarkably well in our experiments.
It is a smoothed version of \reftext{eq:multi-hinge}
and is obtained with $k=1$ and $\gamma > 0$.

\textbf{Softmax conjugate.}
Before we introduce a top-$k$ version of the softmax loss \eqreftext{eq:softmax},
we need to recall its conjugate.

\ifprop
\begin{proposition}\label{prop:softmax-conjugate}
The convex conjugate of the \reftext{eq:softmax} loss is
\begin{align}\label{eq:softmax-conjugate}
L^*(v) &= \begin{cases}
\sum_{j \neq y} v_j \log v_j + (1 + v_y) \log(1 + v_y) , \\
\quad\quad\quad\!\text{if }
\inner{\ones, v} = 0 \text{ and } \wo{v}{y} \in \Sm , \\
+ \infty \quad \text{otherwise},
\end{cases}
\end{align}
where $\Sm = \big\{ x \given \inner{\ones, x} \leq 1, \; x_j \geq 0 \big\}$
is the unit simplex.
\end{proposition}
\fi
\ifproof
\begin{proof}
Here, we use the notation $u \bydef f(x)$
as we need to take special care of the differences $f_j(x) - f_y(x)$
when computing the conjugate.
Therefore, the softmax loss is
\begin{align*}
L(u) = \log\big( \tsum_{j \in \Yc} \exp(u_j - u_y) \big)
= \log\big( \tsum_{j \in \Yc} \exp(a_j) \big) ,
\end{align*}
where $a = H_y u$ as before and $H_y \bydef \Id - \ones \tra{e_y}$.
Define
$$
\phi(u) \bydef \log\big( \tsum_{j \in \Yc} \exp(u_j) \big) ,
$$
then $L(u) = \phi(H_y u)$ and
the convex conjugate is computed similar to
\cite[Lemma~2]{lapin2015topk} as follows.
\begin{align*}
L^*(v)
= \sup\{ &\inner{u,v} - L(u) \given u \in \Rb^m \} \\
= \sup\{ &\inner{u,v} - \phi(H_y u) \given u \in \Rb^m \} \\
= \sup\{ &\innern{\para{u},v} + \innern{\ort{u},v} - \phi(H_y \ort{u})
\given \\
& \para{u} \in \Ker H_y, \ort{u} \in \ort{\Ker} H_y \} ,
\end{align*}
where
$\Ker H_y = \{ u \given H_y u = 0 \} = \{ t\ones \given t \in \Rb \}$
and
$\ort{\Ker} H_y = \{ u \given \inner{\ones, u} = 0 \}$.
It follows that $L^*(v)$ can only be finite if
$\innern{\para{u},v} = 0$, which implies
$v \in \ort{\Ker} H_y \Longleftrightarrow \inner{\ones, v} = 0$.
Let $\pinv{H}_y$ be the Moore-Penrose pseudoinverse of $H_y$.
For a $v \in \ort{\Ker} H_y$, we write
\begin{align*}
L^*(v) &=
\sup\{ \innern{\pinv{H}_y H_y \ort{u},v} - \phi(H_y \ort{u}) \given \ort{u} \} \\
&= \sup\{ \innern{z, \tra{(\pinv{H}_y)} v} - \phi(z) \given z \in \Img H_y \} ,
\end{align*}
where
$\Img H_y = \{ H_y u \given u \in \Rb^m \} = \{ u \given u_y = 0 \}$.
Using rank-$1$ update of the pseudoinverse
\cite[\S~3.2.7]{petersen2008matrix}, we have
$$
\tra{(\pinv{H}_y)} = \Id - e_y \tra{e}_y
- \frac{1}{m}(\ones - e_y) \tra{\ones} ,
$$
which together with $\inner{\ones, v} = 0$ implies
$
\tra{(\pinv{H}_y)} v = v - v_y e_y .
$
\begin{align*}
L^*(v)
&= \sup\{ \innern{u, v - v_y e_y} - \phi(u) \given u_y = 0 \} \\
&= \sup\big\{ \innern{\wo{u}{y}, \wo{v}{y}} -
\log\big( 1 + \tsum_{j \neq y} \exp(u_j) \big) \big\} .
\end{align*}
The function inside $\sup$ is concave and differentiable,
hence the global optimum is at the critical point \cite{boyd2004convex}.
Setting the partial derivatives to zero yields
\begin{align*}
v_j = \exp(u_j) / \big(1 + \tsum_{j \neq y} \exp(u_j) \big)
\end{align*}
for $j \neq y$, from which we conclude,
similar to \cite[\S~5.1]{shalev2014accelerated},
that
$\inner{\ones, v} \leq 1$ and
$0 \leq v_j \leq 1$ for all $j \neq y$, \ie
$\wo{v}{y} \in \Sm$.
Let $Z \bydef \sum_{j \neq y} \exp(u_j)$, we have at the optimum
\begin{align*}
u_j &= \log(v_j) + \log(1 + Z), \quad
\forall j \neq y .
\end{align*}
Since $\inner{\ones, v} = 0$, we also have that
$v_y = - \sum_{j \neq y} v_j$, hence
\begin{align*}
L^*(v)
&= \tsum_{j \neq y} u_j v_j - \log(1 + Z) \\
&= \tsum_{j \neq y} v_j \log(v_j) + \log(1 + Z)
\big( \tsum_{j \neq y} v_j - 1 \big) \\
&= \tsum_{j \neq y} v_j \log(v_j) - \log(1 + Z)(1 + v_y) .
\end{align*}
Summing $v_j$ and using the definition of $Z$,
\begin{align*}
\tsum_{j \neq y} v_j
= \tsum_{j \neq y} e^{u_j} / \big(1 + \tsum_{j \neq y} e^{u_j} \big)
= Z / (1 + Z) .
\end{align*}
Therefore,
\begin{align*}
1 + Z = 1 / \big(1 - \tsum_{j \neq y} v_j \big) = 1 / (1 + v_y) ,
\end{align*}
which finally yields
\begin{align*}
L^*(v)
&= \tsum_{j \neq y} v_j \log(v_j) + \log(1 + v_y)(1 + v_y) ,
\end{align*}
if $\inner{\ones, v} = 0$ and $\wo{v}{y} \in \Sm$
as stated in the proposition.
\end{proof}
\fi

Note that the conjugates of both
the \reftext{eq:multi-hinge} and the \reftext{eq:softmax} losses
share the same effective domain, the unit simplex $\Sm$,
and differ only in their functional form:
a linear function for \reftext{eq:multi-hinge}
and a negative entropy for \reftext{eq:softmax}.
While we motivated top-$k$ SVM directly from the top-$k$ error,
we see that the only change compared to \reftext{eq:multi-hinge}
was in the effective domain of the conjugate loss.
This suggests a general way to \emph{construct novel losses}
with specific properties by taking the conjugate of an existing loss function,
and modifying its effective domain in a way
that enforces the desired properties.
The motivation for doing so comes from the interpretation of the dual
variables as forces with which every training example pushes
the decision surface in the direction given by the ground truth label.
Therefore, by reducing the feasible set we can limit the maximal contribution
of any given training example.

\textbf{Top-$k$ entropy.}
As hinted above,
we first construct the conjugate of the top-$k$ entropy loss ($\alpha$)
by taking the conjugate of \reftext{eq:softmax}
and replacing $\Sm$ in \eqref{eq:softmax-conjugate} with $\Ska$,
and then take the conjugate again to obtain the primal loss
\reftext{eq:topk-entropy}.
A $\beta$ version can be constructed using the set $\Skb$ instead.

\ifprop
\begin{proposition}\label{prop:topk-entropy-primal}
The top-$k$ entropy loss is defined as
\begin{equation}\tag{\LrTopK{k}}\label{eq:topk-entropy}
\mkern-16mu
\begin{aligned}
L(a) = \max&\big\{
\innern{\wo{a}{y},x} - (1-s) \log(1-s) \\
&- \inner{x, \log x}
\given x \in \Ska, \; \inner{\ones, x} = s \big\} .
\end{aligned}
\mkern-16mu
\end{equation}
Moreover, we recover the \reftext{eq:softmax} loss when $k=1$.
\end{proposition}
\fi
\ifproof
\begin{proof}
The convex conjugate of the top-$k$ entropy loss is
\begin{align*}
L^*(v) &\bydef \begin{cases}
\sum_{j \neq y} v_j \log v_j + (1 + v_y) \log(1 + v_y) , \\
\quad\quad\quad\!\text{if }
\inner{\ones, v} = 0 \text{ and } \wo{v}{y} \in \Ska , \\
+ \infty \quad \text{otherwise}.
\end{cases}
\end{align*}
The (primal) top-$k$ entropy loss is defined as the convex conjugate
of the $L^*(v)$ above.
We have
\begin{align*}
L(a)
= \sup\{ &\inner{a,v} - L^*(v) \given v \in \Rb^m \} \\
= \sup\{ &\inner{a,v} - \tsum_{j \neq y} v_j \log v_j 
- (1 + v_y) \log(1 + v_y) \\
&\given \inner{\ones, v} = 0, \; \wo{v}{y} \in \Ska \} \\
= \sup\{ &\innern{\wo{a}{y},\wo{v}{y}} - a_y \tsum_{j \neq y} v_j
- \tsum_{j \neq y} v_j \log v_j \\
& \!\! - (1 - \tsum_{j \neq y} v_j) \log(1 - \tsum_{j \neq y} v_j)
\given \wo{v}{y} \in \Ska \} .
\end{align*}
Note that $a_y = 0$, and hence the corresponding term vanishes.
Finally, we let $x \bydef \wo{v}{y}$
and $s \bydef \sum_{j \neq y} v_j = \inner{\ones, x}$.

Next, we discuss how this problem can be solved and show that it
reduces to the softmax loss for $k=1$.
Let $a \bydef \wo{a}{y}$ and consider an equivalent problem below.
\begin{equation}\label{eq:topk-entropy-primal-min}
\begin{aligned}
L(a) = -\min\big\{&
\inner{x, \log x} + (1-s) \log(1-s) \\
&- \innern{a,x} \given x \in \Ska, \; \inner{\ones, x} = s \big\} .
\end{aligned}
\end{equation}
The Lagrangian for (\ref{eq:topk-entropy-primal-min}) is
\begin{gather*}
\Lc(x, s, t, \lambda, \mu, \nu)
= \inner{x, \log x} + (1-s) \log(1-s) - \innern{a,x} \\
+ t(\inner{\ones, x} - s) + \lambda (s - 1)
- \inner{\mu, x} + \inner{\nu, x - \tfrac{s}{k} \ones} ,
\end{gather*}
where $t \in \Rb$ and $\lambda, \mu, \nu \geq 0$ are the dual variables.
Computing the partial derivatives of $\Lc$ \wrt $x_j$ and $s$,
and setting them to zero, we obtain
\begin{align*}
\log x_j &= a_j - 1 - t + \mu_j - \nu_j , \quad \forall j \\
\log(1-s) &= -1 - t - \tfrac{1}{k} \inner{\ones, \nu} + \lambda .
\end{align*}
Note that $x_j = 0$ and $s = 1$ cannot satisfy the above conditions
for any choice of the dual variables in $\Rb$.
Therefore, $x_j > 0$ and $s < 1$, which implies
$\mu_j = 0$ and $\lambda = 0$.
The only constraint that might be active is $x_j \leq \frac{s}{k}$.
Note, however, that in view of $x_j > 0$
it can only be active if either $k > 1$ or
we have a one dimensional problem.
We consider the case when this constraint is active below.

Consider $x_j$'s for which
$0 < x_j < \frac{s}{k}$ holds at the optimum.
The complementary slackness conditions imply that
the corresponding $\mu_j = \nu_j = 0$.
Let $p \bydef \inner{\ones, \nu}$ and re-define $t$
as $t \leftarrow 1 + t$.
We obtain the simplified equations
\begin{align*}
\log x_j &= a_j - t , \\
\log(1-s) &= - t - \tfrac{p}{k} .
\end{align*}
If $k=1$, then $0 < x_j < s$ for all $j$
in a multiclass problem as discussed above,
hence also $p=0$.
We have
\begin{align*}
x_j &= e^{a_j - t} , &
1-s &= e^{-t} ,
\end{align*}
where $t \in \Rb$ is to be found.
Plugging that into the objective,
\begin{align*}
-L(a)
&= \tsum_j (a_j - t) e^{a_j - t} - t e^{-t} - \tsum_j a_j e^{a_j - t} \\
&= e^{-t} \Big[
\tsum_j (a_j - t) e^{a_j} - t - \tsum_j a_j e^{a_j} \Big] \\
&= -t e^{-t} \big[ 1 + \tsum_j e^{a_j} \big] 
= -t \big[ e^{-t} + \tsum_j e^{a_j - t} \big] \\
&= -t \big[ 1 - s + s \big] = -t .
\end{align*}
To compute $t$, we note that
$$
\tsum_j e^{a_j - t} = \inner{\ones, x} = s = 1 - e^{-t} ,
$$
from which we conclude
\begin{align*}
1 &= \big( 1 + \tsum_j e^{a_j} \big) e^{-t} \; \Longrightarrow \;
-t = - \log(1 + \tsum_j e^{a_j}) .
\end{align*}
Taking into account the minus in front of the $\min$ in
(\ref{eq:topk-entropy-primal-min})
and the definition of $a$, we finally recover the softmax loss
\begin{align*}
L(y,f(x)) = \log\big(1 + \tsum_{j \neq y} \exp(f_j(x) - f_y(x)) \big) .
\end{align*}
\end{proof}
\fi

While there is no closed-form solution for the \reftext{eq:topk-entropy}
loss when $k > 1$, we can compute and optimize it efficiently
as we discuss later in \S~\ref{sec:optimization}.

\textbf{Truncated top-$k$ entropy.}
A major limitation of the softmax loss for top-$k$ error optimization
is that it cannot ignore the $(k-1)$ highest scoring predictions.
This can lead to a situation where the loss is high
even though the top-$k$ error is zero.
To see that, let us rewrite the \reftext{eq:softmax} loss as
\begin{equation}\label{eq:softmax-2}
L(y,f(x)) = \log\big( 1 + \tsum_{j\neq y} \exp(f_{j}(x) - f_y(x)) \big).
\end{equation}
If there is only a \emph{single} $j$ such that $f_j(x) - f_y(x) \gg 0$,
then $L(y,f(x)) \gg 0$ even though $\nerr{2}{y, f(x)}$ is zero.

This problem is also present in all top-$k$ hinge losses
considered above and is an inherent limitation due to their convexity.
The origin of the problem is the fact that ranking
based losses \cite{usunier2009ranking} are based on functions such as
$$
\phi(f(x)) = (1/m) \tsum_{j \in \Yc} \alpha_j f_{\pi_j}(x) - f_y(x) .
$$
The function $\phi$ is convex if the sequence $(\alpha_j)$
is monotonically non-increasing \cite{boyd2004convex}.
This implies that convex ranking based losses have to put \emph{more} weight
on the highest scoring classifiers,
while we would like to put \emph{less} weight on them.
To that end, we drop the first $(k-1)$ highest scoring
predictions from the sum in \eqref{eq:softmax-2},
sacrificing convexity of the loss,
and define the truncated top-$k$ entropy loss as follows
\begin{align}\tag{\LrTopKn{k}}\label{eq:truncated-topk-entropy}
L(a) = \log\big( 1 + \tsum_{j \in \Jc_y^k} \exp(a_j) \big),
\end{align}
where $\Jc_y^k$
are the indexes corresponding to the $(m - k)$ \emph{smallest}
components of $(f_j(x))_{j \neq y}$.
This loss can be seen as a smooth version
of the top-$k$ error \eqref{eq:topk-error},
as it is small whenever the top-$k$ error is zero.
We show a synthetic experiment in \S~\ref{sec:synthetic},
where the advantage of discarding the highest scoring classifier in
\reftext{eq:truncated-topk-entropy} becomes apparent.

\subsection{Multilabel Methods}\label{sec:multilabel-methods}

In this section, we introduce natural extensions of the classic
multiclass methods discussed above
to the setting where there is a \emph{set} of ground truth labels
$Y \subset \Yc$ for each example $x$.
We focus on the loss functions that produce a ranking
of labels and optimize a multilabel loss
$L: 2^\Yc \times \R^m \rightarrow \R_+$.
We let $u \bydef f(x)$ and use a simplified notation
$L(u) = L(Y,f(x))$.
A more complete overview of multilabel classification methods is given in
\cite{tsoumakas2007multi,madjarov2012extensive,zhang2014review}.

\textbf{Binary relevance (BR).}
Binary relevance is the standard one-vs-all scheme
applied to multilabel classification.
It is the default baseline for direct multilabel methods
as it does not consider possible correlations between the labels.

\textbf{Multilabel SVM.}
We follow the line of work by \cite{crammer2003family}
and consider the Multilabel SVM loss below:
\begin{equation}\tag{\SvmML}\label{eq:svm-ml}
\begin{aligned}
L(u) &= \max_{y \in Y} \max_{j \in \bar{Y}}
\max \{0, 1 + u_j - u_y \} \\
&= \max\{ 0, 1
+ \max_{j \in \bar{Y}} u_j - \min_{y \in Y} u_y \} .
\end{aligned}
\end{equation}
This method is also known as
the \emph{multiclass multilabel perceptron} (MMP) \cite{furnkranz2008multilabel}
and the \emph{separation ranking loss} \cite{guo2011adaptive}.
It can be contrasted with another \reftext{eq:multi-hinge} extension,
the RankSVM of Elisseeff and Weston \cite{elisseeff2001kernel},
which optimizes the \emph{pairwise ranking loss}:
$$
\tfrac{1}{\abs{Y_i} \abs{\bar{Y}_i}}
\tsum_{(y, j) \in Y \times \bar{Y}}
\max \{0, 1 + u_j - u_y \} .
$$
Note that both the \reftext{eq:svm-ml} that we consider
and RankSVM avoid expensive enumeration of all the $2^\Yc$
possible labellings by considering only pairwise label ranking.
A principled large margin approach
that accounts for all possible label interactions
is structured output prediction \cite{tsochantaridis2005large}.

\textbf{Multilabel SVM conjugate.}
Here, we compute the convex conjugate of the \reftext{eq:svm-ml} loss
which is used later to define a Smooth Multilabel SVM.
Note that the \reftext{eq:svm-ml} loss depends on the partitioning
of $\Yc$ into $Y$ and $\bar{Y}$ for every given $(x,Y)$ pair.
This is reflected in the definition of a set $S_Y$ below,
which is the effective domain of the conjugate:
\begin{align*}
S_Y \bydef \big\{ x \given
- \tsum_{y \in Y} x_y = \tsum_{j \in \bar{Y}} x_j \leq 1, \;
x_y \leq 0, \;  x_j \geq 0 \big\} .
\end{align*}
In the multiclass setting, the set $Y$ is singleton,
therefore $x_y = - \sum_{j \in \bar{Y}} x_j$ has no degrees of freedom
and we recover the unit simplex $\Sm$ over $(x_j)$,
as in \eqref{eq:softmax-conjugate}.
In the true multilabel setting, on the other hand,
there is freedom to distribute the weight across all the classes in $Y$.

\ifprop
\begin{proposition}\label{prop:svm-ml-conj}
The convex conjugate of the \reftext{eq:svm-ml} loss is
\begin{align}\label{eq:svm-ml-conj}
L^*(v) = - \tsum_{j \in \bar{Y}} v_j , \; \text{if} \; v \in S_Y , \;
+\infty, \; \text{otherwise} .
\end{align}
\end{proposition}
\fi
\ifproof
\begin{proof}
We compute the convex conjugate of \eqreftext{eq:svm-ml} as
\begin{align*}
L^*(v) &= - \inf_{u \in \Rb^m} \{
\max\{ 0, 1
+ \max_{j \in \bar{Y}} u_j - \min_{y \in Y} u_y \}
- \inner{u, v}
\} .
\end{align*}
When the infimum is attained, the conjugate can be computed by solving
the following optimization problem, otherwise the conjugate is $+\infty$.
The corresponding dual variables are given on the right.
\begin{align*}
\min_{u, \alpha, \beta, \xi} \;\;
& \xi - \inner{u, v} \\
& \xi \geq 1 + \beta - \alpha , && (\lambda \geq 0) \\
& \xi \geq 0 , && (\mu \geq 0) \\
& \alpha \leq u_y, \quad \forall \, y \in Y , && (\nu_y \geq 0) \\
& \beta \geq u_j, \quad \forall \, j \in \bar{Y} . && (\eta_j \geq 0)
\end{align*}
The Lagrangian is given as
\begin{gather*}
\Lc(u, \alpha, \beta, \xi, \lambda, \mu, \nu, \eta)
= \xi - \inner{u, v}
+ \lambda (1 + \beta - \alpha - \xi) \\
- \mu \xi + \tsum_{y \in Y} \nu_y (\alpha - u_y)
+ \tsum_{j \in \bar{Y}} \eta_j (u_j - \beta) .
\end{gather*}
Computing the partial derivatives and setting them to zero,
\begin{align*}
\partial_{u_y} \Lc &= - v_y - \nu_y , &&&
\nu_y &= - v_y , && \forall \, y \in Y , \\
\partial_{u_j} \Lc &= - v_j + \eta_j , &&&
\eta_j &= v_j , && \forall \, j \in \bar{Y} , \\
\partial_{\alpha} \Lc &= - \lambda + \inner{\ones, \nu} , &&&
\lambda &= \inner{\ones, \nu} , \\
\partial_{\beta} \Lc &= \lambda - \inner{\ones, \eta} , &&&
\lambda &= \inner{\ones, \eta} , \\
\partial_{\xi} \Lc &= 1 - \lambda - \mu , &&&
\lambda &= 1 - \mu .
\end{align*}
After a basic derivation, we arrive at the solution of the dual problem
given by
\begin{align*}
\lambda = - \tsum_{y \in Y} v_y = \tsum_{j \in \bar{Y}} v_j ,
\end{align*}
where $v$ must be in the following feasible set $S_Y$:
\begin{align*}
S_Y \bydef \big\{ &v \in \Rb^m \given
- \tsum_{y \in Y} v_y = \tsum_{j \in \bar{Y}} v_j \leq 1, \\
&v_y \leq 0, \;  v_j \geq 0, \; 
\forall \, y \in Y, \; \forall \, j \in \bar{Y} \big\} .
\end{align*}
To complete the proof, note that $L^*(v) = - \lambda$
if $v \in S_Y$.
\end{proof}
\fi

Note that when $\abs{Y}=1$, \eqref{eq:svm-ml-conj} naturally reduces
to the conjugate of \reftext{eq:multi-hinge}
given in Proposition~\ref{prop:topk-conj} with $k=1$.

\textbf{Smooth multilabel SVM.}
Here, we apply the smoothing technique,
which worked very well for multiclass problems
\cite{shalev2014accelerated,lapin2016loss},
to the multilabel \reftext{eq:svm-ml} loss.

As with the smooth top-$k$ SVM,
there is no analytic formula for the smoothed loss.
However, we can both compute and optimize it within our framework
by solving the Euclidean projection problem
onto what we call a \textbf{bipartite simplex}.
It is a convenient modification of the set $S_Y$ above:
\begin{align}\label{eq:bipartite-simplex}
B(r) \bydef \{ (x,y) \given \inner{\ones, x} = \inner{\ones, y} \leq r,
x \in \Rb^m_+, y \in \Rb^n_+ \}.
\end{align}

\ifprop
\begin{proposition}\label{prop:smooth-svm-ml}
Let $\gamma > 0$ be the smoothing parameter.
The smooth multilabel SVM loss and its conjugate are
\begin{align}
L_{\gamma}(u) &= \tfrac{1}{\gamma}\big(
\inner{b,p} - \tfrac{1}{2}\norms{p} +
\inner{\bar{b},\bar{p}} - \tfrac{1}{2}\norms{\bar{p}} \big),
\tag{\small\SvmMLg{\gamma}}\label{eq:smooth-svm-ml}\\
L_{\gamma}^*(v) &= \begin{cases}
\frac{1}{2} \big( \sum_{y \in Y} v_y - \sum_{j \in \bar{Y}} v_j \big)
+ \frac{\gamma}{2} \norms{v}, & v \in S_Y , \\
+\infty, & \text{o/w} ,
\end{cases}\notag
\end{align}
where
$(p,\bar{p}) = \proj_{B(\gamma)}(b,\bar{b})$
is the projection onto $B(\gamma)$ of
$b = \big( \tfrac{1}{2} - u_y \big)_{y \in Y}$,
$\bar{b} = \big( \tfrac{1}{2} + u_j \big)_{j \in \bar{Y}}$.
$L_{\gamma}(u)$ is $1/\gamma$-smooth.
\end{proposition}
\fi
\ifproof
\begin{proof}
The convex conjugate of the \reftext{eq:svm-ml} loss is
\begin{align*}
L^*(v) = \begin{cases}
\sum_{y \in Y} v_y , & \text{if } v \in S_Y , \\
+\infty, & \text{otherwise} .
\end{cases}
\end{align*}
Before we add $\frac{\gamma}{2} \norms{v}$,
recall that $\sum_{y \in Y} v_y = - \sum_{j \in \bar{Y}} v_j$,
and so
$\sum_{y \in Y} v_y =
\frac{1}{2} \big( \tsum_{y \in Y} v_y - \tsum_{j \in \bar{Y}} v_j \big)$.
We use the average instead of an individual sum
for symmetry and improved numerical stability.
The smoothed conjugate loss is then
\begin{align*}
L_{\gamma}^*(v) = \begin{cases}
\frac{1}{2} \big( \tsum_{y \in Y} v_y - \tsum_{j \in \bar{Y}} v_j \big)
+ \frac{\gamma}{2} \norms{v},
& \text{if } v \in S_Y , \\
+\infty, & \text{otherwise} .
\end{cases}
\end{align*}
To derive the primal loss, we take the conjugate again:
\begin{align*}
L_{\gamma}(u)
&= \sup_v \{ \inner{u, v} - L_{\gamma}^*(v) \} \\
&= \max_{v \in S_Y} \big\{ \inner{u, v}
- \tfrac{1}{2} \big( \textstyle\sum\limits_{y \in Y} v_y
- \textstyle\sum\limits_{j \in \bar{Y}} v_j \big)
- \tfrac{\gamma}{2} \norms{v} \big\} \\
&= - \tfrac{1}{\gamma} \min_{\tfrac{v}{\gamma} \in S_Y} \big\{
\tfrac{1}{2} \norms{v} + \tfrac{1}{2} \big(
\textstyle\sum\limits_{y \in Y} v_y
- \textstyle\sum\limits_{j \in \bar{Y}} v_j \big)
- \inner{u, v} \big\} \\
&= - \tfrac{1}{\gamma} \min_{\tfrac{v}{\gamma} \in S_Y} \big\{
\tfrac{1}{2} \norms{v} - \tsum_{y \in Y} (\tfrac{1}{2} - u_y) (- v_y) \\[-1em]
&\mkern159mu
- \tsum_{j \in \bar{Y}} (\tfrac{1}{2} + u_j) v_j \big\} .
\end{align*}
Next, we define the following auxiliary variables:
\begin{align*}
x_j &= -v_j, & b_j &= \tfrac{1}{2} - u_j, & \forall\, j &\in Y, \\
y_j &= v_j, & \bar{b}_j &= \tfrac{1}{2} + u_j, & \forall\, j &\in \bar{Y},
\end{align*}
and rewrite the smooth loss $L_{\gamma}(u)$ equivalently as
\begin{align*}
L_{\gamma}(u) = - \tfrac{1}{\gamma} \min_{x, y} \;
& \tfrac{1}{2} \norms{x} - \inner{x, b}
+ \tfrac{1}{2} \norms{y} - \inner{y, \bar{b}} \\
& \inner{\ones, x} = \inner{\ones, y} \leq \gamma, \\
& x \geq 0, \; y \geq 0 ,
\end{align*}
which is the Euclidean projection onto the set $B(\gamma)$.
\end{proof}
\fi

Note that the smooth \reftext{eq:smooth-svm-ml} loss
is a nice generalization of the smooth multiclass loss \SvmMultig{\gamma}
and we naturally recover the latter when $Y$ is singleton.
In \S~\ref{sec:optimization}, we extend
the variable fixing algorithm of \cite{kiwiel2008variable}
and obtain an efficient method
to compute Euclidean projections onto $B(r)$.

\textbf{Multilabel cross-entropy.}
Here, we discuss an extension of the \LrMulti\ loss to multilabel learning.
We use the softmax function to
model the distribution over the class labels $p_y(x)$,
which recovers the well-known multinomial logistic regression
\cite{krishnapuram2005sparse}
and the maximum entropy \cite{yu2011dual} models.

Assume that all the classes given in the ground truth set $Y$ are equally likely.
We define an empirical distribution for a given $(x,Y)$ pair as
$\hat{p}_y = (1/\abs{Y}) \iv{y \in Y}$,
and model the conditional probability $p_y(x)$ via the softmax:
\begin{align*}
p_y(x) = (\exp u_y ) / \big(\tsum_{j \in \Yc} \exp u_j \big),
\quad \forall\, y \in \Yc.
\end{align*}
The cross-entropy of the distributions $\hat{p}$ and $p(x)$ is given by
\begin{align*}
H(\hat{p}, p(x))
&= - \tfrac{1}{\abs{Y}} \sum_{y \in Y}
\log\Big( \frac{\exp u_y}{\sum_j \exp u_j} \Big) ,
\end{align*}
and the corresponding multilabel cross entropy loss is:
\begin{align}\tag{\small\LrML}\label{eq:lr-ml}
L(u) &= \tfrac{1}{\abs{Y}} \tsum_{y \in Y}
\log\big( \tsum_{j \in \Yc} \exp(u_j - u_y) \big) .
\end{align}

\textbf{Multilabel cross-entropy conjugate.}
Next, we compute the convex conjugate of the \reftext{eq:lr-ml} loss,
which is used later in our optimization framework.

\ifprop
\begin{proposition}\label{prop:lr-ml-conj}
The convex conjugate of the \reftext{eq:lr-ml} loss is
\begin{align}\label{eq:lr-ml-conj}
L^*(v) &= \begin{cases}
\tsum_{y \in Y} (v_y + \tfrac{1}{k}) \log (v_y + \tfrac{1}{k})
+ \tsum_{j \in \bar{Y}} v_j \log v_j, \\
\quad\quad\quad\!\text{if } v \in D_Y, \\
+\infty \quad \text{otherwise} .
\end{cases}\raisetag{1.75\baselineskip}
\end{align}
where $k=\abs{Y}$ and $D_Y$ is the effective domain defined as:
\begin{align*}
D_Y \bydef \big\{ v \given
&\tsum_{y \in Y} \big( v_y + \tfrac{1}{k} \big)
+ \tsum_{j \in \bar{Y}} v_j = 1, \\
&\, v_y + \tfrac{1}{k} \geq 0, \;
v_j \geq 0, \; y \in Y, \; j \in \bar{Y}
\big\} .
\end{align*}
\end{proposition}
\fi
\ifproof
\begin{proof}
The conjugate loss is given by
$
L^*(v) = \sup \{ \inner{u, v} - L(u) \given u \in \Rb^m \} .
$
Since $L(u)$ is smooth and convex in $u$,
we compute the optimal $u^*$ by setting the partial derivatives to zero,
which leads to $v_j = \tfrac{\partial}{\partial u_j} L(u)$.
We have
\begin{align*}
\tfrac{\partial}{\partial u_l} L(u) &=
\tfrac{1}{\abs{Y}} \sum_{y \in Y}
\frac{\partial_{u_l} \big( \tsum_j \exp(u_j - u_y) \big)
}{\sum_j \exp(u_j - u_y)} ,
\end{align*}
\begin{align*}
\partial_{u_l} \big( \tsum_j \exp(u_j - u_y) \big)
&= \begin{cases}
\exp(u_l - u_y), & l \neq y, \\
- \sum_{j \neq y} \exp(u_j - u_y), & l = y.
\end{cases}
\end{align*}
Therefore,
\begin{align*}
\tfrac{\partial}{\partial u_l} L(u) &=
\tfrac{1}{\abs{Y}} \sum_{y \in Y}
\frac{1}{\sum_j \exp u_j}
\begin{cases}
\exp u_l , & \text{if } l \neq y, \\
- \sum_{j \neq y} \exp u_j , & \text{if } l = y.
\end{cases}
\end{align*}
Let $Z \bydef \sum_{j \in \Yc} \exp u_j$, then
\begin{align*}
\tfrac{\partial}{\partial u_l} L(u) &=
\tfrac{1}{\abs{Y}} \sum_{y \in Y}
\frac{1}{Z}
\begin{cases}
\exp u_l , & \text{if } l \neq y, \\
\exp u_l - Z , & \text{if } l = y.
\end{cases}
\end{align*}
Let $k \bydef \abs{Y}$, we have
\begin{align*}
&l  \notin Y \Longrightarrow &
\tfrac{\partial}{\partial u_l} L(u)
&= \tfrac{1}{k} \tsum_{y \in Y} \tfrac{1}{Z} \exp u_l
= \tfrac{1}{Z} \exp u_l , \\
&l  \in Y \Longrightarrow &
\tfrac{\partial}{\partial u_l} L(u)
&= \tfrac{1}{k Z} \big( \exp u_l - Z + (k - 1) \exp u_l \big) \\
&&&= \tfrac{1}{Z} \exp u_l - \tfrac{1}{k} .
\end{align*}
Thus, for the supremum to be attained, we must have
\begin{align}\label{eq:ce_v}
v_j &=
\begin{cases}
\tfrac{1}{Z} \exp u_j - \tfrac{1}{k} , & \text{if } j \in Y, \\
\tfrac{1}{Z} \exp u_j , & \text{if } j \in \bar{Y},
\end{cases}
\end{align}
which means $v_j \geq - \tfrac{1}{k}$ if $j \in Y$,
and $v_j \geq 0$ otherwise.
Moreover, we have
\begin{align*}
\inner{\ones, v} &=
\tsum_{j \in Y} \big( \tfrac{1}{Z} \exp u_j - \tfrac{1}{k} \big)
+ \tsum_{j \in \bar{Y}_i} \tfrac{1}{Z} \exp u_j \\
&= \tfrac{1}{Z} \tsum_{j \in \Yc} \exp u_j - 1 = 0
\end{align*}
and
\begin{align*}
\tsum_{j \in Y} v_j &=
\tsum_{j \in Y} \big( \tfrac{1}{Z} \exp u_j - \tfrac{1}{k} \big)
\leq \tfrac{1}{Z} \tsum_j \exp u_j - 1 = 0 , \\
\tsum_{j \in \bar{Y}_i} v_j &=
\tsum_{j \in \bar{Y}_i} \tfrac{1}{Z} \exp u_j
\leq \tfrac{1}{Z} \tsum_j \exp u_j = 1 .
\end{align*}
Solving (\ref{eq:ce_v}) for $u$, we get
\begin{align*}
u^*_j = \begin{cases}
\log(v_j + \tfrac{1}{k}) + \log Z , & \text{if } j \in Y , \\
\log v_j + \log Z, & \text{otherwise} .
\end{cases}
\end{align*}
Plugging the optimal $u^*$, we compute the conjugate as
\begin{align*}
L^*(Y, v) &= \inner{u^*, v}
- \tfrac{1}{\abs{Y}} \tsum_{y \in Y}
\log\Big( \tsum_j \exp(u^*_j - u^*_y) \Big) \\
&= \tsum_{y \in Y} v_y \log (v_y + \tfrac{1}{k})
+ \tsum_{j \in \bar{Y}} v_j \log v_j \\
&+ \tsum_j v_j \log Z
- \tfrac{1}{k} \tsum_{y \in Y} \big( \log Z - u^*_y \big) \\
&= \tsum_{y \in Y} v_y \log (v_y + \tfrac{1}{k})
+ \tsum_{j \in \bar{Y}} v_j \log v_j \\
&+ \tfrac{1}{k} \tsum_{y \in Y} \log(v_y + \tfrac{1}{k}) \\
&= \tsum_{y \in Y} (v_y + \tfrac{1}{k}) \log (v_y + \tfrac{1}{k})
+ \tsum_{j \in \bar{Y}} v_j \log v_j ,
\end{align*}
where $\inner{\ones, v} = 0$ and
\begin{align*}
\tsum_{y \in Y} v_y &\leq 0, &
v_y + \tfrac{1}{k} &\geq 0 , \; y \in Y, \\
\tsum_{j \in \bar{Y}} v_j &\leq 1, &
v_j &\geq 0 , \; j \in \bar{Y} .
\end{align*}
This leads to the definition of the effective domain $D_Y$, since
\begin{align*}
0 = \inner{\ones, v}
&= \tsum_{y \in Y} v_y + \tsum_{j \in \bar{Y}} v_j \\
&= \tsum_{y \in Y} (v_y + \tfrac{1}{k}) + \tsum_{j \in \bar{Y}} v_j - 1 .
\end{align*}
\end{proof}
\fi

The conjugates of the multilabel losses
\reftext{eq:svm-ml} and \reftext{eq:lr-ml}
no longer share the same effective domain,
which was the case for multiclass losses.
However, we still recover the conjugate of
the \reftext{eq:softmax} loss when $Y$ is singleton.

\section{Bayes Optimality and Top-k Calibration}%
\label{sec:topk-calibration}
This section is devoted to the theoretical analysis of multiclass losses
in terms of their top-$k$ performance.
We establish the best top-$k$ error in the Bayes sense,
determine when a classifier achieves it,
define the notion of top-$k$ calibration,
and investigate which loss functions possess this property.

\textbf{Bayes optimality.}
Recall that the Bayes optimal zero-one loss in binary classification
is simply the probability of the least likely class \cite{friedman2001elements}.
Here, we extend this notion to the top-$k$ error \eqref{eq:topk-error}
introduced in \S~\ref{sec:perf-metrics} for multiclass classification
and provide a description of top-$k$ Bayes optimal classifier.

\ifprop
\begin{lemma}\label{lem:bayes-topk-error}
The Bayes optimal top-$k$ error at $x$ is
\begin{align*}
\min_{g \in \Rb^m} \Exp_{Y \given X}[\kerr{Y,g} \given X=x]
= 1 - \tsum_{j=1}^k p_{\tau_j}(x),
\end{align*}
where
$ p_{\tau_1}(x) \geq p_{\tau_2}(x) \geq \ldots \geq p_{\tau_m}(x)$.
A classifier $f$ is \textbf{top-$k$ Bayes optimal at $x$} if and only if
\begin{align*}
\big\{ y \given f_y(x) \geq f_{\pi_k}(x) \big\} \subset
\big\{ y \given p_y(x) \geq p_{\tau_k}(x) \big\} ,
\end{align*}
where $ f_{\pi_1}(x) \geq f_{\pi_2}(x) \geq \ldots \geq f_{\pi_m}(x)$.
\end{lemma}
\fi
\ifproof
\begin{proof}
For any $g = f(x) \in \R^m$, let $\pi$ be a permutation such that
$g_{\pi_1} \geq g_{\pi_2} \geq \ldots \geq g_{\pi_m}$.
The expected top-$k$ error at $x$ is
\begin{align*}
&\Exp_{Y \given X}[\kerr{Y,g} \given X=x] 
= \tsum_{y \in \Yc} \iv{g_{\pi_k} > g_{y}} p_y(x) \\
&= \tsum_{y \in \Yc} \iv{g_{\pi_k} > g_{\pi_y}} p_{\pi_y}(x)
= \tsum_{j=k+1}^m p_{\pi_j}(x) \\
&= 1 - \tsum_{j=1}^k p_{\pi_j}(x) .
\end{align*}
The error is minimal when $\tsum_{j=1}^k p_{\pi_j}(x)$ is maximal,
which corresponds to taking the $k$ largest conditional probabilities
$\tsum_{j=1}^k p_{\tau_j}(x)$ and yields the Bayes optimal top-$k$ error at $x$.

Since the relative order within $\{ p_{\tau_j}(x) \}_{j=1}^k$ is irrelevant
for the top-$k$ error, any classifier $f(x)$, for which the sets
$\{\pi_1, \ldots, \pi_k\}$ and $\{\tau_1, \ldots, \tau_k\}$
coincide, is Bayes optimal.

Note that we assumed \Wlog that there is a clear cut
$p_{\tau_k}(x) > p_{\tau_{k+1}}(x)$
between the $k$ most likely classes and the rest.
In general, ties can be resolved arbitrarily as long as we can guarantee
that the $k$ largest components of $f(x)$ correspond to the classes (indexes)
that yield the maximal sum $\tsum_{j=1}^k p_{\pi_j}(x)$ and lead to
top-$k$ Bayes optimality.
\end{proof}
\fi

Another way to write the optimal top-$k$ error is
$\tsum_{j=k+1}^m p_{\pi_j}(x)$, which naturally leads to an optimal prediction
strategy according to the ranking of $p_y(x)$ in descending order.
However, the description of a top-$k$ Bayes optimal classifier reveals that
optimality for any given $k$ is better understood as a \emph{partitioning},
rather than ranking, where the labels are split into $\pi_{1:k}$ and the rest,
without any preference on the ranking in either subset.
If, on the other hand, we want a classifer that is top-$k$ Bayes optimal
\emph{for all} $k \geq 1$ \emph{simultaneously},
a proper ranking according to $p_y(x)$ is both necessary and sufficient.

\textbf{Top-$k$ calibration.}
Optimization of the zero-one loss and %
the top-$k$ error leads to hard combinatorial problems.
Instead of tackling a combinatorial problem directly,
an alternative is to use a convex surrogate loss
which upper bounds the discrete error.
Under mild conditions on the loss function
\cite{BarJorAuc2006,tewari2007consistency},
an optimal classifier for the surrogate yields
a Bayes optimal solution for the zero-one loss.
Such loss functions are called \emph{classification calibrated},
which is known in statistical learning theory as a necessary condition
for a classifier to be universally Bayes consistent \cite{BarJorAuc2006}.
We introduce now the notion of calibration for the top-$k$ error.
\begin{definition}\label{def:calibration}
A multiclass loss function $L:\Yc \times \Rb^m \rightarrow \Rb_+$
is called \textbf{top-$k$ calibrated} if 
for all possible data generating measures on $\Xc \times \Yc$
and all $x \in \Xc$
\begin{align*}
&\targmin_{g \in \Rb^m} \Exp_{Y \given X}[L(Y,g) \given X = x]  \\
\subseteq &\targmin_{g \in \Rb^m} \Exp_{Y \given X}[\kerr{Y,g} \given X = x] .
\end{align*}
\end{definition}
If a loss is \emph{not} top-$k$ calibrated,
it implies that even in the limit of infinite data,
one does not obtain a classifier with
the Bayes optimal top-$k$ error from Lemma~\ref{lem:bayes-topk-error}.
It is thus an important property, even though of an asymptotic nature.
Next, we analyse which of the multiclass classification methods
covered in \S~\ref{sec:multiclass-methods} are top-$k$ calibrated.

\subsection{Multiclass Top-k Calibration}%
\label{sec:topk-calibration-multiclass}

In this section, we consider top-$k$ calibration
of the standard OVA scheme, established multiclass classification methods,
and the proposed \reftext{eq:truncated-topk-entropy} loss.
First, we state a condition under which an OVA scheme is
uniformly top-$k$ calibrated, not only for $k=1$,
which corresponds to the standard zero-one loss,
but \emph{for all} $k \geq 1$ simultaneously.
The condition is given in terms of the Bayes optimal classifier
for each of the corresponding binary problems
and with respect to a given loss function $L$,
\eg the hinge or logistic losses.

\begin{lemma}\label{lem:ova-topk-calibrated}
The OVA reduction is top-$k$ calibrated for any $1 \leq k \leq m$
if the Bayes optimal function of a convex margin-based loss $L$ is a
strictly monotonically increasing function of $p_y(x) = \Pr(Y=y \given X=x)$
for every class $y \in \Yc$.
\end{lemma}
\begin{proof}
Let the Bayes optimal classifier for the binary problem corresponding
to a $y \in \Yc$ have the form
\[
f_y(x) = g\big(\Pr(Y=y \given X=x)\big) ,
\]
where $g$ is a strictly monotonically increasing function.
The ranking of $f_y$ corresponds to the 
ranking of $p_y(x)$ and hence the OVA reduction
is top-$k$ calibrated for any $k \geq 1$.
\end{proof}

Next, we use Lemma~\ref{lem:ova-topk-calibrated}
and the corresponding Bayes optimal classifiers
to check if the one-vs-all schemes employing hinge
and logistic regression losses are top-$k$ calibrated.

\ifprop
\begin{proposition}\label{prop:calibrated-hinge}
OVA SVM is not top-$k$ calibrated.
\end{proposition}
\fi
\ifproof
\begin{proof}
First, we show that the Bayes optimal function for the binary hinge loss is
\begin{align*}
f^*(x) &= 2 \iv{\Pr(Y=1 \given X=x) > \tfrac{1}{2}} -1 .
\end{align*}
We decompose the expected loss as
\[ \Exp_{X,Y}[L(Y,f(X))]=\Exp_{X}[\Exp_{Y|X}[L(Y,f(x)) \given X = x]].\]
Thus, one can compute the Bayes optimal classifier $f^*$
pointwise by solving
\[ \argmin_{\alpha \in \R} \Exp_{Y|X}[L(Y,\alpha) \given X=x],\]
for every $x \in \R^d$, which leads to the following problem
\begin{align*}
\argmin_{\alpha \in \R} \;
\max\{0, 1-\alpha\} p_1(x) + \max\{0, 1+\alpha\} p_{-1}(x) ,
\end{align*}
where $p_y(x) \bydef \Pr(Y=y \given X=x)$.
It is obvious that the optimal $\alpha^*$ is contained in $[-1,1]$.
We get
\[
\argmin_{-1\leq \alpha\leq 1} \;
(1 - \alpha) p_1(x) + (1 + \alpha)p_{-1}(x).
\]
The minimum is attained at the boundary and we get
\[
f^*(x) =
\begin{cases} +1 & \text{if } p_1(x) > \frac{1}{2} , \\
-1 & \text{if } p_1(x) \leq \frac{1}{2} .
\end{cases}
\]
Therefore, the Bayes optimal classifier for the hinge loss is not
a strictly monotonically increasing function of $p_1(x)$.

To show that OVA hinge is not top-$k$ calibrated, we construct an example problem
with $3$ classes and $p_1(x) = 0.4$, $p_2(x) = p_3(x) = 0.3$.
Note that for every class $y = 1, 2, 3$, the Bayes optimal binary classifier is $-1$,
hence the predicted ranking of labels is arbitrary and
may not produce the Bayes optimal top-$k$ error.
\end{proof}
\fi

\ifprop
\begin{proposition}\label{prop:calibrated-lr}
OVA logistic regression is top-$k$ calibrated.
\end{proposition}
\fi
\ifproof
\begin{proof}
First, we show that the Bayes optimal function for the binary logistic loss is
\begin{align*}
f^*(x) &= \log\Big(\frac{p_1(x)}{1-p_1(x)}\Big) .
\end{align*}
As above, the pointwise optimization problem is
\[ \argmin_{\alpha \in \R} \;
\log(1+\exp(-\alpha))p_1(x) + \log(1+\exp(\alpha))p_{-1}(x).\]
The logistic loss is known to be convex and differentiable and thus the optimum can be computed via
\[
\frac{-\exp(-\alpha)}{1+\exp(-\alpha)}p_1(x)
+ \frac{\exp(\alpha)}{1+\exp(\alpha)}p_{-1}(x)=0 .
\]
Re-writing the first fraction we get
\[
\frac{-1}{1+\exp(\alpha)}p_1(x)
+ \frac{\exp(\alpha)}{1+\exp(\alpha)}p_{-1}(x)=0 ,
\]
which can be solved as
$ \alpha^* = \log\Big(\frac{p_1(x)}{p_{-1}(x)}\Big)$
and leads to the formula for
the Bayes optimal classifier stated above.

We check now that the function $\phi: (0,1) \rightarrow \Rb$
defined as $\phi(x)=\log(\frac{x}{1-x})$ is strictly monotonically increasing.
\begin{align*}
\phi'(x) &= \frac{1-x}{x}\big( \frac{1}{1-x} + \frac{x}{(1-x)^2}\big) \\
&= \frac{1-x}{x}\frac{1}{(1-x)^2} = \frac{1}{x(1-x)} > 0, \quad \forall x \in (0,1).
\end{align*}
The derivative is strictly positive on $(0,1)$, which implies that
$\phi$ is strictly monotonically increasing.
The logistic loss, therefore, fulfills the conditions of
Lemma~\ref{lem:ova-topk-calibrated} and is top-$k$ calibrated for
any $1 \leq k \leq m$.
\end{proof}
\fi

The hinge loss is not calibrated since the corresponding
binary classifiers, being piecewise constant,
are subject to degenerate cases that result in arbitrary rankings of classes.
Surprisingly, the smoothing technique based on Moreau-Yosida regularization
(\S~\ref{sec:multiclass-methods})
makes a smoothed loss more attractive not only from the optimization
side, but also in terms of top-$k$ calibration.
Here, we show that a smooth binary hinge loss from \cite{shalev2014accelerated}
fulfills the conditions of Lemma~\ref{lem:ova-topk-calibrated}
and leads to a top-$k$ calibrated OVA scheme.

\ifprop
\begin{proposition}\label{prop:smooth-ova-hinge-calibrated}
OVA smooth SVM is top-$k$ calibrated.
\end{proposition}
\fi
\ifproof
\begin{proof}
In order to derive the smooth hinge loss,
we first compute the conjugate of the standard binary hinge loss,
\begin{align}
L(\alpha) &= \max\{0,1-\alpha\} , \nonumber\\
L^*(\beta) &=
\sup_{\alpha \in \R} \big\{ \alpha \beta - \max\{0,1-\alpha\} \big\} \nonumber\\
&= \begin{cases}
\beta & \textrm{if } -1 \leq \beta \leq 0, \\
\infty & \textrm{otherwise} .
\end{cases}
\label{eq:ova-hinge-conj}
\end{align}
The smoothed conjugate is
$$
L^*_\gamma(\beta) = L^*(\beta)+\frac{\gamma}{2}\beta^2 .
$$
The corresponding primal smooth hinge loss is given by
\begin{align}
L_\gamma(\alpha) &=
\sup_{-1 \leq \beta \leq 0}
\big\{ \alpha\beta - \beta-\tfrac{\gamma}{2}\beta^2 \big\} \nonumber\\
&= \begin{cases}
1 - \alpha - \frac{\gamma}{2} & \text{if } \alpha < 1 - \gamma, \\
\frac{(\alpha-1)^2}{2 \gamma} & \text{if } 1 - \gamma \leq \alpha \leq 1, \\
0, &  \text{if } \alpha > 1 .
\end{cases}
\label{eq:ova-hinge-smooth}
\end{align}
$L_\gamma(\alpha)$ is convex and differentiable with the derivative
\[
L'_\gamma(\alpha) =  \begin{cases}
-1 & \text{if } \alpha < 1-\gamma, \\
\frac{\alpha - 1}{\gamma} & \text{if } 1 - \gamma \leq \alpha \leq 1, \\
0, & \text{if } \alpha > 1.
\end{cases}
\]
We compute the Bayes optimal classifier pointwise.
\[
f^*(x) = \argmin_{\alpha \in \R} \; L(\alpha) p_1(x) + L(-\alpha) p_{-1}(x).
\]
Let $p \bydef p_1(x)$, the optimal $\alpha^*$ is found by solving
\[
L'(\alpha) p - L'(-\alpha) (1 - p) = 0 .
\]

\textbf{Case $0 < \gamma \leq 1$.}
Consider the case $1 - \gamma \leq \alpha \leq 1$,
\begin{align*}
\frac{\alpha-1}{\gamma} p + (1-p) = 0 \quad \Longrightarrow \quad
\alpha^* = 1 - \gamma \frac{1-p}{p}.
\end{align*}
This case corresponds to $p \geq \frac{1}{2}$,
which follows from the constraint $\alpha^* \geq 1 - \gamma$.
Next, consider $\gamma - 1 \leq \alpha \leq 1 - \gamma$,
\[ -p + (1-p) = 1-2p \neq 0 , \]
unless $p = \frac{1}{2}$, which is already captured by the first case.
Finally, consider $-1 \leq \alpha \leq \gamma - 1 \leq 1 - \gamma$.
Then
\[ -p - \frac{-\alpha-1}{\gamma}(1-p) = 0 \quad \Longrightarrow \quad
\alpha^* = -1 + \gamma\frac{p}{1-p},\]
where we have $-1 \leq \alpha^* \leq \gamma - 1$ if $p \leq \frac{1}{2}$.
We obtain the Bayes optimal classifier for $0 < \gamma \leq 1$ as follows:
\[
f^*(x) =
\begin{cases}
1 - \gamma \frac{1-p}{p} & \text{if } p \geq \frac{1}{2}, \\
-1 + \gamma \frac{p}{1-p} & \text{if } p < \frac{1}{2}.
\end{cases}
\]
Note that while $f^*(x)$ is not a continuous function of $p = p_1(x)$
for $\gamma < 1$,
it is still a strictly monotonically increasing function of $p$
for any $0 < \gamma \leq 1$.

\textbf{Case $\gamma>1$.}
First, consider $\gamma - 1 \leq \alpha \leq 1$,
\begin{align*}
\frac{\alpha - 1}{\gamma} p + (1-p) = 0 \quad \Longrightarrow \quad
\alpha^* = 1 - \gamma\frac{1-p}{p}.
\end{align*}
From $\alpha^* \geq \gamma - 1$, we get the condition $p \geq \frac{\gamma}{2}$.
Next, consider $1 - \gamma \leq \alpha \leq \gamma - 1$,
\[ \frac{\alpha-1}{\gamma}p - \frac{-\alpha-1}{\gamma} (1-p) =0
\quad \Longrightarrow \quad \alpha^* = 2p - 1, \]
which is in the range $[1 - \gamma, \gamma - 1]$
if $1 - \frac{\gamma}{2} \leq p \leq \frac{\gamma}{2}$.
Finally, consider $-1 \leq \alpha \leq 1 - \gamma$,
\[ -p - \frac{-\alpha-1}{\gamma} (1-p) = 0 \quad \Longrightarrow \quad
\alpha^* = -1 + \gamma\frac{p}{1-p}, \]
where we have $-1 \leq \alpha^* \leq 1 - \gamma$
if $p \leq 1 - \frac{\gamma}{2}$.
Overall, the Bayes optimal classifier for $\gamma>1$ is
\begin{align*}
f^*(x) &=
\begin{cases}
1 - \gamma \frac{1 - p}{p} &
\text{if } p \geq \frac{\gamma}{2}, \\
2 p - 1 &
\text{if } 1 - \frac{\gamma}{2} \leq p \leq \frac{\gamma}{2}, \\
-1 + \gamma \frac{p}{1 - p} &
\text{if } p <  1 - \frac{\gamma}{2}.
\end{cases}
\end{align*}
Note that $f^*$ is again a strictly monotonically increasing function
of $p = p_1(x)$.
Therefore, for any $\gamma > 0$,
the one-vs-all scheme with the smooth hinge loss (\ref{eq:ova-hinge-smooth})
is top-$k$ calibrated for all $1 \leq k \leq m$
by Lemma~\ref{lem:ova-topk-calibrated}.
\end{proof}
\fi

An alternative to the OVA scheme with binary losses
is to use a \emph{multiclass} loss
$L:\Yc \times \Rb^m \rightarrow \Rb_+$ directly. 
First, we consider the multiclass hinge loss \reftext{eq:multi-hinge},
which is known to be \emph{not} calibrated for the top-$1$ error
\cite{tewari2007consistency},
and show that it is not top-$k$ calibrated for any $k$.

\ifprop
\begin{proposition}\label{prop:multi-hinge-topk-calibrated}
Multiclass SVM is not top-$k$ calibrated.
\end{proposition}
\fi
\ifproof
\begin{proof}
Let $y \in \argmax_{j \in \Yc} p_j(x)$.
Given any $c \in \Rb$, we will show that a Bayes optimal classifier
$f^*:\Rb^d \rightarrow \Rb^m$ for the \reftext{eq:multi-hinge} loss is
\begin{align*}
f^*_{y}(x) &= \begin{cases}
c + 1 & \text{if } \max_{j \in \Yc} p_j(x) \geq \frac{1}{2} , \\
c & \text{otherwise} ,
\end{cases} \\
f^*_{j}(x) &= c, \; j \in \Yc \setminus \{y\} .
\end{align*}
Let $g = f(x) \in \Rb^m$, then
\begin{align*}
\Exp_{Y|X}[L(Y,g) \given X] =
\sum_{l \in \Yc} \max_{j \in \Yc} \big\{\iv{j \neq l} + g_j - g_l \big\} p_l(x) .
\end{align*}
Suppose that the maximum of $(g_j)_{j \in \Yc}$ is not unique.
In this case, we have
\[
\max_{j \in \Yc} \big\{\iv{j\neq l} + g_j - g_l \big\} \geq 1 , \; \forall \, l \in \Yc
\]
as the term $\iv{j\neq l}$ is always active.
The best possible loss is obtained by setting $g_j = c$ for all $j \in \Yc$,
which yields an expected loss of $1$.
On the other hand, if the maximum is unique and is achieved by $g_y$, then
\begin{multline*}
\max_{j \in \Yc} \big\{\iv{j\neq l} + g_j - g_l \big\} \\
= \begin{cases}
1 + g_{y} - g_l & \textrm{ if } l \neq y, \\
\max\big\{0,\, \max_{j \neq y} \{ 1 + g_j - g_y \} \big\} & \textrm{ if } l=y.
\end{cases}
\end{multline*}
As the loss only depends on the gap $g_{y} - g_l$,
we can optimize this with $\beta_l = g_{y} - g_l$.
\begin{align*}
\Exp&_{Y|X}[L(Y,g) \given X=x] \\
&= \sum_{l \neq y} (1 + g_y - g_l) p_l(x) \\
&+ \max\big\{0,\, \max_{l \neq y} \{ 1 + g_l - g_y \} \big\} p_y(x) \\
&= \sum_{l \neq y} (1 + \beta_l) p_l(x)
+ \max\big\{0,\, \max_{l \neq y} \{ 1 - \beta_l \} \big\} p_y(x) \\
&= \sum_{l \neq y} (1 + \beta_l) p_l(x)
+ \max\{0, 1 - \min_{l \neq y} \beta_l \} p_y(x) .
\end{align*}
As only the minimal $\beta_l$ enters the last term,
the optimum is achieved if all $\beta_l$ are equal for $l \neq y$
(otherwise it is possible to reduce the first term without affecting the last term).
Let $\alpha \bydef \beta_l$ for all $l\neq y$.
The problem becomes
\begin{align*}
&\min_{\alpha \geq 0} \sum_{l \neq y} (1+\alpha) p_l(x) + \max\{0, 1-\alpha \} p_y(x) \\
&\equiv \min_{0 \leq \alpha \leq 1} \alpha (1 - 2 p_y(x) )
\end{align*}
Let $p \bydef p_y(x) = \Pr(Y=y \given X=x)$.
The solution is
\begin{align*}
\alpha^* = \begin{cases}
0 & \textrm{if } p < \frac{1}{2} , \\
1 & \textrm{if } p \geq \frac{1}{2} ,
\end{cases}
\end{align*}
and the associated risk is
\begin{align*}
\Exp_{Y|X}[L(Y,g) \given X=x] =
\begin{cases}
1 & \text{if } p < \frac{1}{2}, \\
2 (1 - p) & \text{if } p \geq \frac{1}{2} .
\end{cases}
\end{align*}
If $p < \frac{1}{2}$, then the Bayes optimal classifier
$f^*_j(x) = c$ for all $j \in \Yc$ and any $c \in \Rb$.
Otherwise, $p \geq \frac{1}{2}$ and
$$
f^*_j(x) = \begin{cases}
c+1 & \text{if } j=y , \\
c & \text{if } j \in \Yc \setminus \{y\} .
\end{cases}
$$
Moreover, we have that the Bayes risk at $x$ is
$$
\Exp_{Y|X}[L(Y,f^*(x)) \given X=x] = \min\{1,2(1 - p)\} \leq 1 .
$$

It follows, that the multiclass hinge loss is not
(top-$1$) classification calibrated
at any $x$ where $\max_{y \in \Yc} p_y(x) < \frac{1}{2}$
as its Bayes optimal classifier reduces to a constant.
Moreover, even if $p_y(x) \geq \frac{1}{2}$ for some $y$,
the loss is not top-$k$ calibrated for $k \geq 2$
as the predicted order of the remaining classes need not be optimal.
\end{proof}
\fi

Tewari and Bartlett \cite{tewari2007consistency} provide a general framework
to study classification calibration that is applicable to a large
family of multiclass methods.
However, their characterization of calibration is derived in terms of
the properties of the convex hull of
$\{ ( L(1, f), \ldots, L(m, f) ) \given f \in \Fc \}$,
which might be difficult to verify in practice.
In contrast, our proofs of
Propositions~\ref{prop:multi-hinge-topk-calibrated}
and \ref{prop:softmax-topk-calibrated}
are straightforward and based on direct derivation
of the corresponding Bayes optimal classifiers
for the \reftext{eq:multi-hinge}
and the \reftext{eq:softmax} losses respectively.

\ifprop
\begin{proposition}\label{prop:softmax-topk-calibrated}
Multiclass softmax loss is top-$k$ calibrated.
\end{proposition}
\fi
\ifproof
\begin{proof}
The multiclass softmax loss is (top-$1$) calibrated for the zero-one error
in the following sense.
If
\[
f^*(x) \in \argmin_{ g \in \R^m} \Exp_{Y|X}[L(Y,g) \given X=x] ,
\]
then for some $\alpha>0$ and all $y \in \Yc$
\[
f^*_y(x) = \begin{cases}
\log(\alpha \, p_y(x) ) & \text{if } p_y(x) > 0, \\
-\infty & \text{otherwise},
\end{cases}
\]
which implies
\[
\argmax_{y \in \Yc} f^*_y(x) = \argmax_{y \in \Yc} \Pr(Y=y \given X=x) .
\]
We now prove this result and show that it also generalizes
to top-$k$ calibration for $k>1$.
Using the identity
$$
L(y,g) = \log\big(\tsum_{j \in \Yc} e^{g_j - g_y}\big)
= \log\big(\tsum_{j \in \Yc} e^{g_j}\big) - g_y
$$
and the fact that $\sum_{y \in \Yc} p_y(x) = 1$,
we write for a $g \in \R^m$
\begin{align*}
\Exp&_{Y|X}[L(Y,g) \given X=x] \\
&= \sum_{y \in \Yc} L(y,g) p_y(x)
= \log\big( \sum_{y \in \Yc} e^{g_y} \big) - \sum_{y \in \Yc} g_y p_x(y) .
\end{align*}
As the loss is convex and differentiable,
we get the global optimum by computing a critical point.
We have
\begin{align*}
\frac{\partial}{\partial g_j}\Exp_{Y|X}[L(Y,g) \given X=x] =
\frac{e^{g_j}}{\sum_{y \in \Yc} e^{g_y}} - p_j(x) = 0
\end{align*}
for $j \in \Yc$.
We note that the critical point is not unique as multiplication
$g\rightarrow \kappa g$ leaves the equation invariant for any $\kappa > 0$.
One can verify that $e^{g_j} = \alpha p_j(x)$
satisfies the equations for any $\alpha > 0$.
This yields a solution
\begin{align*}
f^*_y(x) &= \begin{cases}
\log(\alpha p_y(x) ) & \text{if } p_y(x) > 0 , \\
-\infty & \text{otherwise},
\end{cases}
\end{align*}
for any fixed $\alpha>0$.
We note that $f^*_y$ is a strictly monotonically increasing function
of the conditional class probabilities.
Therefore, it preserves the ranking of $p_y(x)$
and implies that $f^*$ is top-$k$ calibrated for any $1 \leq k \leq m$.
\end{proof}
\fi

The implicit reason for top-$k$ calibration of the OVA schemes
and the softmax loss is that one can estimate the probabilities $p_y(x)$
from the Bayes optimal classifier.
Loss functions which allow this are called \emph{proper}.
We refer to \cite{ReiWil2010} and references therein for a detailed discussion.

We have established that the OVA logistic regression and the softmax loss
are top-$k$ calibrated for any $k$, so why should we be interested
in defining new loss functions for the top-$k$ error?
The reason is that calibration is an asymptotic property
since the Bayes optimal functions are obtained by \emph{pointwise}
minimization of $\Exp_{Y|X}[L(Y,f(x)) \given X=x]$ at every $x \in \Xc$.
The picture changes if we use linear classifiers, 
since they obviously cannot be minimized independently at each point.
Indeed, the Bayes optimal classifiers, in general,
cannot be realized by linear functions. 

Furthermore, convexity of the softmax and multiclass hinge losses
leads to phenomena where $\kerr{y, f(x)} = 0$, but $L(y, f(x)) \gg 0$.
We discussed this issue \S~\ref{sec:multiclass-methods}
and motivated modifications of the above losses
for the top-$k$ error.
Next, we show that one of the proposed top-$k$ losses
is also top-$k$ calibrated.

\ifprop
\begin{proposition}\label{prop:truncated-topk-entropy-topk-calibrated}
The truncated top-$k$ entropy loss is top-$s$ calibrated
for any $k \leq s \leq m$.
\end{proposition}
\fi
\ifproof
\begin{proof}
Given any $g = f(x) \in \Rb^m$, let $\pi$ be a permutation such that
$g_{\pi_1} \geq g_{\pi_2} \geq \ldots \geq g_{\pi_m}$.
Then, we have
\begin{align*}
\Jc_y = \begin{cases}
\{ \pi_{k+1}, \ldots, \pi_m \} &\text{if } y \in \{ \pi_1, \ldots, \pi_{k-1} \}, \\
\{ \pi_{k}, \ldots, \pi_m \} \setminus \{y\} &\text{if } y \in \{ \pi_{k}, \ldots, \pi_m \}.
\end{cases}
\end{align*}
Therefore, the expected loss at $x$ can be written as
\begin{align*}
\Exp_{Y|X}&[L(Y,g) \given X = x] = \tsum_{y \in \Yc} L(y,g) \, p_y(x) \\
&= \tsum_{r=1}^{k-1}
\log\big(1 + \tsum_{j=k+1}^m e^{g_{\pi_j} - g_{\pi_r}} \big) \, p_{\pi_r}(x) \\
&+
\tsum_{r=k}^{m}
\log\big(\tsum_{j=k}^m e^{g_{\pi_j} - g_{\pi_r}} \big) \, p_{\pi_r}(x) .
\end{align*}
Note that the sum inside the logarithm does not depend on $g_{\pi_r}$ for $r<k$.
Therefore, a Bayes optimal classifier will have $g_{\pi_r} = +\infty$ for all $r<k$
as then the first sum vanishes.

Let $p \bydef (p_y(x))_{y \in \Yc}$ and $q \bydef (L(y, g))_{y \in \Yc}$,
then
\begin{align*}
q_{\pi_1} = \ldots = q_{\pi_{k-1}} = 0 \leq q_{\pi_k} \leq \ldots \leq q_{\pi_m}
\end{align*}
and we can re-write the expected loss as
\begin{align*}
\Exp_{Y|X}&[L(Y,g) \given X = x] = \inner{p, q} 
= \inner{p_{\pi}, q_{\pi}} \geq \inner{p_{\tau}, q_{\pi}} ,
\end{align*}
where $p_{\tau_1} \geq p_{\tau_2} \geq \ldots \geq p_{\tau_m}$
and we used the rearrangement inequality.
Therefore, the expected loss is minimized when $\pi$ and $\tau$ coincide
(up to a permutation of the first $k-1$ elements),
which already establishes top-$s$ calibration for all $s \geq k$.

We can also derive a Bayes optimal classifier following the proof of
Proposition~\ref{prop:softmax-topk-calibrated}.
We have
\begin{align*}
\Exp_{Y|X}&[L(Y,g) \given X = x] \\
&=\tsum_{r=k}^{m}
\log\big(\tsum_{j=k}^m e^{g_{\tau_j} - g_{\tau_r}} \big) \, p_{\tau_r}(x) \\
&=\tsum_{r=k}^{m} \Big(
\log\big(\tsum_{j=k}^m e^{g_{\tau_j}} \big) - g_{\tau_r} \Big) \, p_{\tau_r}(x) .
\end{align*}
A critical point is found by setting partial derivatives to zero
for all $y \in \{\tau_k, \ldots, \tau_m\}$, which leads to
\begin{align*}
\frac{e^{g_y}}{\tsum_{j=k}^m e^{g_{\tau_j}}}
\tsum_{r=k}^{m} p_{\tau_r}(x) = p_y(x) .
\end{align*}
We let $g_y = -\infty$ if $p_y(x) = 0$, and obtain finally
\begin{align*}
g^*_{\tau_j} =
\begin{cases}
+\infty &\text{if } j < k, \\
\log\big(\alpha p_{\tau_j}(x)\big)
&\text{if } j \geq k \text{ and } p_{\tau_j}(x) > 0, \\
-\infty &\text{if } j \geq k \text{ and } p_{\tau_j}(x) = 0,
\end{cases}
\end{align*}
as a Bayes optimal classifier for any $\alpha > 0$.

Note that $g^*$ preserves the ranking of
$p_y(x)$ for all $y$ in $\{\tau_k, \ldots, \tau_m\}$,
hence, it is top-$s$ calibrated for all $s \geq k$.
\end{proof}
\fi

Top-$k$ calibration of the remaining top-$k$ losses is an open problem,
which is complicated by the absence of a closed-form expression for most of them.

\section{Optimization Framework}%
\label{sec:optimization}
This section is mainly devoted to efficient optimization
of the multiclass and multilabel methods from \S~\ref{sec:loss-functions}
within the stochastic dual coordinate ascent (SDCA)
framework of Shalev-Shwartz and Zhang \cite{shalev2013stochastic}.
The core reason for efficiency of the optimization scheme
is the ability to formulate variable updates in terms of
projections onto the effective domain of the conjugate loss,
which, in turn, can be solved in time $O(m \log m)$ or faster.
These projections fall into a broad area of nonlinear resource allocation
\cite{patriksson2015algorithms},
where we already have a large selection of specialized algorithms.
For example, we use an algorithm of Kiwiel \cite{kiwiel2008variable}
for \reftext{eq:multi-hinge} and \reftext{eq:topk-hinge-beta},
and contribute analogous algorithms for the remaining losses.
In particular, we propose an entropic projection algorithm
based on the Lambert $W$ function for the \reftext{eq:softmax} loss,
and a variable fixing algorithm
for projecting onto the bipartite simplex \eqref{eq:bipartite-simplex}
for the \reftext{eq:svm-ml}.
We also discuss how the proposed loss functions that do not have
a closed-form expression can be evaluated efficiently,
and perform a runtime comparison against FISTA \cite{beck2009fast}
using the SPAMS optimization toolbox \cite{mairal2010network}.

In \S~\ref{sec:optimization-background},
we state the primal and Fenchel dual optimization problems,
and introduce the Lambert $W$ function.
In \S~\ref{sec:optimization-multiclass},
we consider SDCA update steps and loss computation for multiclass methods,
as well as present our runtime evaluation experiments.
In \S~\ref{sec:optimization-multilabel},
we cover multilabel optimization and present our algorithm for
the Euclidean projection onto the bipartite simplex.

\subsection{Technical Background}\label{sec:optimization-background}

We briefly recall the main facts about the SDCA framework
\cite{shalev2013stochastic},
Fenchel duality \cite{borwein2000convex},
and the Lambert $W$ function \cite{corless1996lambertw}.

\textbf{The primal and dual problems.}
Let $X \in \Rb^{d \times n}$ be the matrix of training examples $x_i \in \Rb^d$,
$K = \tra{X}\!X$ the corresponding Gram matrix,
$W \in \Rb^{d \times m}$ the matrix of primal variables,
$A \in \Rb^{m \times n}$ the matrix of dual variables,
and
$\lambda > 0$ the regularization parameter.
The primal and Fenchel dual \cite{borwein2000convex}
objective functions are given as
\begin{equation}\label{eq:primal-dual}
\begin{aligned}
P(W) &=
+\frac{1}{n} \sum_{i=1}^n
L \left( y_i, \tra{W}x_i \right)
+ \frac{\lambda}{2} \tr\left(\tra{W} W \right) ,
\\
D(A) &= 
-\frac{1}{n} \sum_{i=1}^n
L^* \left( y_i, - \lambda n a_i \right)
- \frac{\lambda}{2} \tr\left( A K \tra{A} \right) ,
\end{aligned}
\end{equation}
where $L^*$ is the convex conjugate of $L$
and $y_i$ is interpreted as a set $Y_i$ if $L$ is a multilabel loss.

SDCA proceeds by sampling a dual variable $a_i \in \Rb^m$,
which corresponds to a training example $x_i \in \Rb^d$,
and modifying it to achieve maximal increase in the dual objective $D(A)$
while keeping other dual variables fixed.
Several sampling strategies can be used, \eg \cite{qu2015quartz},
but we use a simple scheme where the set of indexes is randomly
shuffled before every epoch and then all $a_i$'s
are updated sequentially.
The algorithm terminates when the relative duality gap
$(P(W) - D(A)) / P(W)$ falls below a pre-defined $\varepsilon > 0$,
or the computational budget is exhausted,
in which case we still have an estimate of suboptimality
via the duality gap.

Since the algorithm operates entirely on the dual variables and
the prediction scores $f(x_i)$, it is directly applicable to training
both linear $f(x_i) = \tra{W} x_i$ as well as nonlinear $f(x_i) = A K_i$
classifiers ($K_i$ being the $i$-th column of the Gram matrix $K$).
When $d \ll n$, which is often the case in our experiments,
and we are training a linear classifier, then it is less expensive
to maintain the primal variables $W = X \tra{A}$ \cite{lapin2015topk}
and compute the dot products $\tra{W} x_i$ in $\Rb^d$.
In that case, whenever $a_i$ is updated,
we perform a rank-$1$ update of $W$.

It turns out that every update step $\max_{a_i} D(A)$
is equivalent to the proximal operator\footnote{
The \textbf{proximal operator}, or the \textbf{proximal map},
of a function $f$ is defined as
$\prox_{f}(v) =
\argmin\limits_x \big( f(x) + \tfrac{1}{2} \norms{x - v} \big)$.
}
of a certain function, which can be seen as a projection
onto the effective domain of $L^*$.

\textbf{Lambert $W$ function.}
The Lambert $W$ function is defined as the inverse of the mapping
$w \mapsto w e^w$.
It is widely used in many fields of computer science
\cite{corless1996lambertw, Fukushima201377, Veberic20122622},
and can often be recognized in nonlinear equations involving
the $\exp$ and the $\log$ functions.
Taking logarithms on both sides of the defining equation $z = W e^W$,
we get $\log z = W(z) + \log W(z)$.
Therefore, if we are given an equation of the form
$x + \log x = t$ for some $t \in \Rb$,
we can directly ``solve'' it in closed-form as
$x = W(e^t)$.
The crux of the problem is that the function
$V(t) \bydef W(e^t)$ is transcendental \cite{Fukushima201377}
just like the logarithm and the exponent.
There exist highly optimized implementations for the latter 
and we argue that the same can be done for the Lambert $W$ function.
In fact, there is already some work on this topic
\cite{Fukushima201377, Veberic20122622}, which we also employ in our
implementation.

\begin{figure}[t]\small\centering%
\begin{subfigure}[t]{0.49\columnwidth}\centering
\includegraphics[width=\columnwidth]{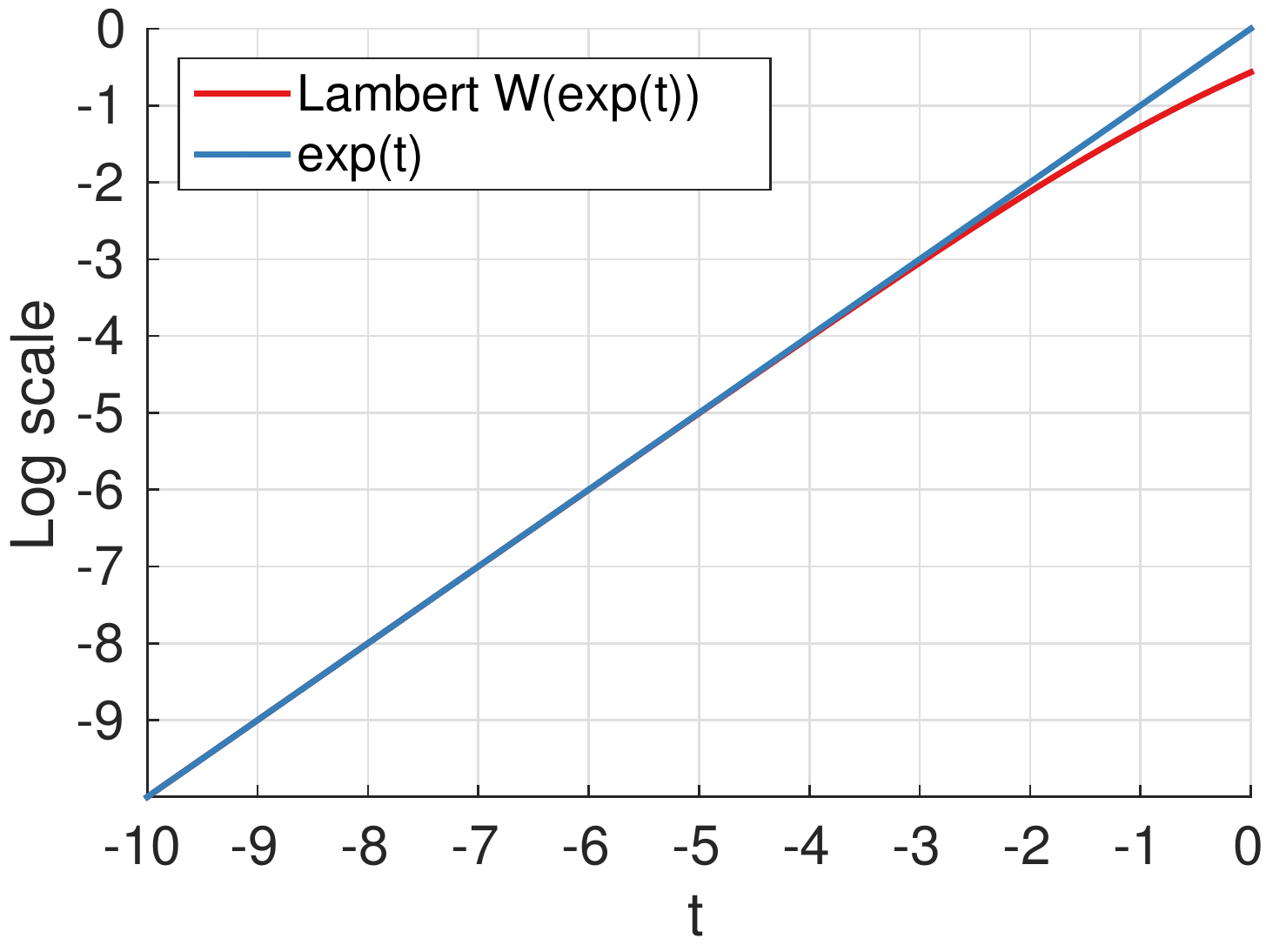}%
\caption{$V(t) \approx e^t$ for $t \ll 0$.}\label{fig:lambert:a}
\end{subfigure}
\begin{subfigure}[t]{0.49\columnwidth}\centering
\includegraphics[width=\columnwidth]{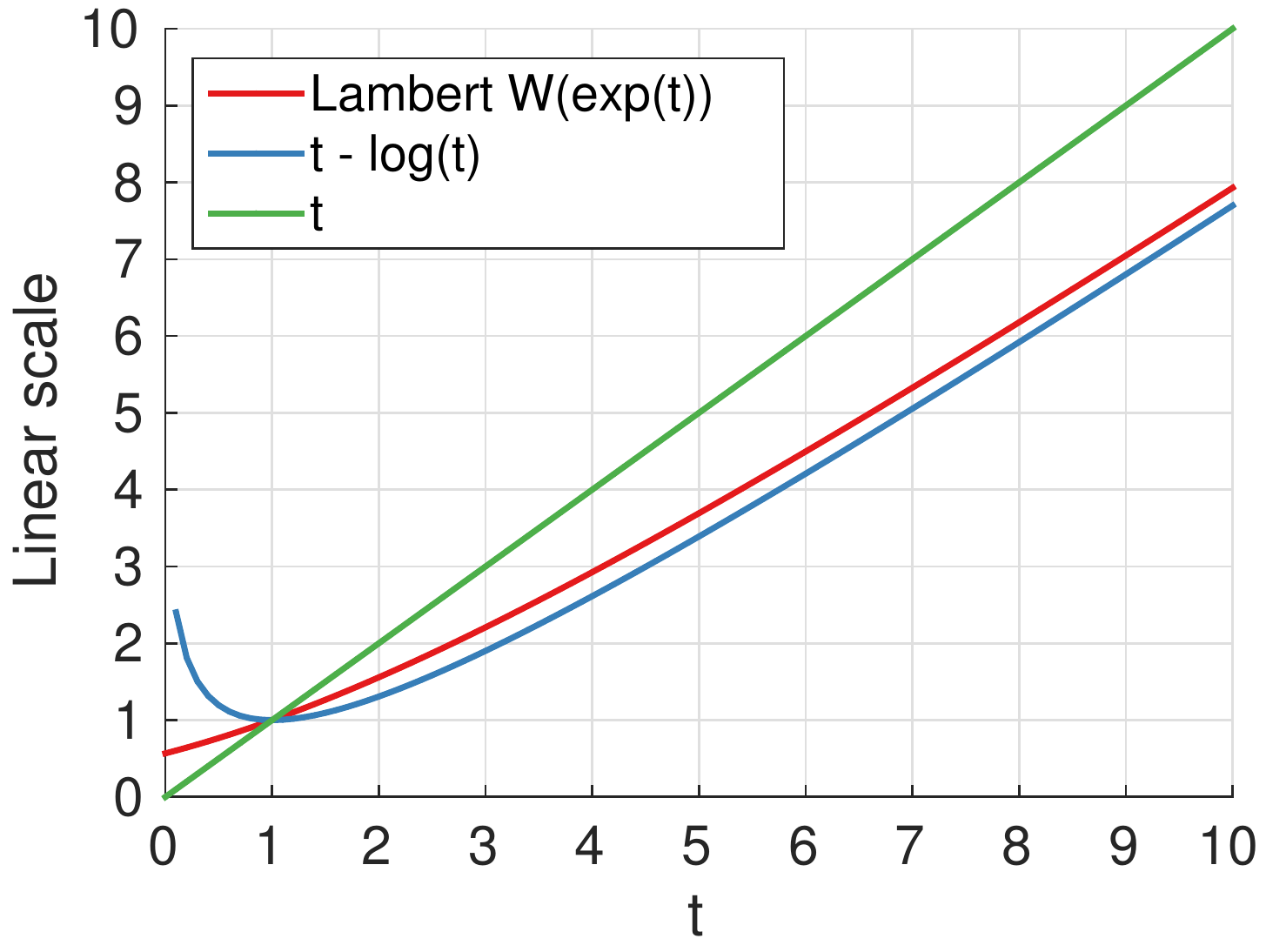}%
\caption{$V(t) \approx t - \log t$ for $t \gg 0$.}\label{fig:lambert:b}
\end{subfigure}
\caption{%
Behavior of the Lambert $W$ function of the exponent ($V(t) = W(e^t)$).
{\bfseries (a)} Log scale plot with $t \in (-10,0)$.
{\bfseries (b)} Linear scale plot with $t \in (0,10)$.
}\label{fig:lambert}%
\end{figure}

To develop intuition about the function $V(t) = W(e^t)$,
which is the Lambert $W$ function of the exponent,
we look at how it behaves for different values of $t$.
An illustration is provided in Figure~\ref{fig:lambert}.
One can see directly from the equation
$x + \log x = t$ that the behavior of $x=V(t)$
changes dramatically depending on whether
$t$ is a large positive or a large negative number.
In the first case, the linear part dominates the logarithm
and the function is approximately linear;
a better approximation is $x(t) \approx t - \log t$,
when $t \gg 1$.
In the second case, the function behaves like an exponent $e^t$.
To see this, we write $x = e^t e^{-x}$ and note that
$e^{-x} \approx 1$ when $t \ll 0$,
therefore, $x(t) \approx e^t$, if $t \ll 0$.

To compute $V(t)$, we use these approximations as initial points
in a $5$-th order Householder method \cite{householder1970numerical}.
A \emph{single} iteration of that method
is already sufficient to get full float precision
and at most two iterations are needed for double,
which makes the function $V(t)$ an attractive
tool for computing entropic projections.

\subsection{Multiclass Methods}\label{sec:optimization-multiclass}

In this section, we cover optimization of the multiclass methods
from \S~\ref{sec:multiclass-methods} within the SDCA framework.
We discuss how to efficiently compute the smoothed losses
that were introduced via conjugation and do not have a closed-form expression.
Finally, we evaluate SDCA convergence in terms of runtime
and show that smoothing with Moreau-Yosida regularization
leads to significant improvements in speed.

As mentioned in \S~\ref{sec:optimization-background} above,
the core of the SDCA algorithm is the update step
$a_i \leftarrow \argmax_{a_i} D(A)$.
Even the primal objective $P(W)$ is only computed
for the duality gap and could conceivably be omitted
if the certificate of optimality is not required.
Next, we focus on how the updates are computed
for the different multiclass methods.

\textbf{SDCA update:
\reftext{eq:ova-hinge}, \reftext{eq:ova-lr}.}
SDCA updates for the binary hinge and logistic losses
are covered in \cite{hsieh2008dual} and \cite{shalev2014accelerated}.
We highlight that the \reftext{eq:ova-hinge} update has a closed-form
expression that leads to scalable training of linear SVMs \cite{hsieh2008dual},
and is implemented in LibLinear \cite{REF08a}.

\textbf{SDCA update:
\reftext{eq:multi-hinge}, \reftext{eq:softmax}, \SvmMultig{\gamma}.}
Although \reftext{eq:multi-hinge} is also covered in
\cite{shalev2014accelerated}, they use a different algorithm
based on sorting, while we do a case distinction \cite{lapin2015topk}.
First, we solve an easier continuous quadratic knapsack problem
using a variable fixing algorithm of Kiwiel \cite{kiwiel2008variable}
which does not require sorting.
This corresponds to enforcing the equality constraint in the simplex
and generally already gives the optimal solution.
The computation is also fast: we observe linear time complexity in practice,
as shown in Figure~\ref{fig:time:a}.
For the remaining hard cases, however, we fall back to sorting
and use a scheme similar to \cite{shalev2014accelerated}.
In our experience, performing the case distinction seemed to
offer significant time savings.

For the \reftext{eq:multi-hinge} and \SvmMultig{\gamma}, we note that they
are special cases of \reftext{eq:smooth-topk-alpha}
and \SvmTopKbg{k}{\gamma} with $k=1$,
as well as \reftext{eq:softmax} is a special case of \reftext{eq:topk-entropy}.

\textbf{SDCA update:
\SvmTopKab{k}, \SvmTopKabg{k}{\gamma}.}
Here, we consider the update step for the smooth
\reftext{eq:smooth-topk-alpha} loss.
The nonsmooth version is directly recovered by setting $\gamma = 0$,
while the update for \SvmTopKbg{k}{\gamma}
is derived similarly using the set $\Skb$ in \eqref{eq:smooth-hinge-update}
instead of $\Ska$.

We show that performing the update step is equivalent to
projecting a certain vector $b$,
computed from the prediction scores $f(x_i) = \tra{W}x_i$,
onto the effective domain of $L^*$,
the top-$k$ simplex,
with an added regularization $\rho \inner{\ones, x}^2$,
which biases the solution to be orthogonal to $\ones$.

\ifprop
\begin{proposition}\label{prop:smooth-topk-hinge-update}
Let $L$ and $L^*$ in \eqref{eq:primal-dual} be respectively
the \SvmTopKag{k}{\gamma} loss and its conjugate as in
Proposition~\ref{prop:topk-smooth}.
The dual variables $a_i$ corresponding to $(x_i,y_i)$ are updated as:
\begin{align}\label{eq:smooth-hinge-update}
\begin{cases}
\wo{a_i}{y_i} \hspace*{-.75em}
&= - \argmin_{x \in \Ska(1/ (\lambda n))}
\big\{ \norms{x - b} + \rho \inner{\ones, x}^2 \big\} , \\
a_{y_i,i} \hspace*{-.75em}
&= - \tsum_{j \neq y_i} a_{j,i} ,
\end{cases}\raisetag{\baselineskip}
\end{align}
where
$b = \frac{1}{\inner{x_i,x_i} + \gamma \lambda n}
\left( \wo{q}{y_i} + (1 - q_{y_i})\ones \right)$,\\
$q = \tra{W} x_i - \inner{x_i,x_i} a_i$,
and
$\rho = \frac{\inner{x_i,x_i}}{\inner{x_i,x_i} + \gamma \lambda n}$.
\end{proposition}
\fi
\ifproof
\begin{proof}
We follow the proof of \cite[Proposition~4]{lapin2015topk}.
Choose an $i \in \{1, \ldots, n\}$ and update $a_i$ to maximize
$$
-\tfrac{1}{n} L^* \left(y_i, - \lambda n a_i \right)
- \tfrac{\lambda}{2} \tr\left( A K \tra{A} \right) .
$$
For the nonsmooth top-$k$ hinge loss,
it was shown \cite{lapin2015topk} that
$$
L^* \left(y_i, - \lambda n a_i \right) =
\inner{c, \lambda n ( a_i - a_{y_i,i} e_{y_i}) }
$$
if $- \lambda n ( a_i - a_{y_i,i} e_{y_i}) \in \Ska$ and
$+ \infty$ otherwise.
Now, for the smoothed loss,
we add regularization and obtain
\begin{align*}
-\tfrac{1}{n} \left(
\tfrac{\gamma}{2} \norms{- \lambda n ( a_i - a_{y_i,i} e_{y_i})}
+ \inner{c, \lambda n ( a_i - a_{y_i,i} e_{y_i}) } \right)
\end{align*}
with $- \lambda n ( a_i - a_{y_i,i} e_{y_i}) \in \Ska$.
Using $c = \ones - e_{y_i}$ and $\inner{\ones, a_i} = 0$,
one can simplify it to
\begin{align*}
- \frac{\gamma n \lambda^2}{2} \normsb{\wo{a_i}{y_i}}
+ \lambda a_{y_i,i} ,
\end{align*}
and the feasibility constraint can be re-written as
\begin{align*}
-\wo{a_i}{y_i} &\in \Ska(\tfrac{1}{\lambda n}) , &
a_{y_i,i} &= \innern{\ones, -\wo{a_i}{y_i}} .
\end{align*}
For the regularization term $\tr\left( A K \tra{A} \right)$, we have
\begin{align*}
\tr\left( A K \tra{A} \right) =
K_{ii} \inner{a_i,a_i} + 2 \sum_{j \neq i} K_{ij} \inner{a_i,a_j}
+ {\rm const} .
\end{align*}
We let
$q = \sum_{j \neq i} K_{ij} a_j = A K_i - K_{ii} a_i$
and $x = -\wo{a_i}{y_i}$:
\begin{align*}
\inner{a_i,a_i} &= \inner{\ones,x}^2 + \inner{x,x}, \\
\inner{q,a_i} &= q_{y_i} \inner{\ones,x} - \innern{\wo{q}{y_i},x}.
\end{align*}
Now, we plug everything together and multiply with $-2/\lambda$.
\begin{align*}
\min_{x \in \Ska(\frac{1}{\lambda n})}
& \gamma \lambda n \norms{x}
- 2 \inner{\ones, x}
+ 2 \big(
q_{y_i} \inner{\ones,x} - \innern{\wo{q}{y_i},x} \big) \\
&+ K_{ii} \big(\inner{\ones,x}^2 + \inner{x,x} \big) .
\end{align*}
Collecting the corresponding terms finishes the proof.
\end{proof}
\fi

We solve \eqref{eq:smooth-hinge-update}
using the algorithm for computing
a (biased) projection onto the top-$k$ simplex,
which we introduced in \cite{lapin2015topk},
with a minor modification of $b$ and $\rho$.
Similarly, the update step for the \SvmTopKbg{k}{\gamma} loss
is solved using a (biased) continuous quadratic knapsack problem,
which we discuss in the supplement of \cite{lapin2015topk}.

Smooth top-$k$ hinge losses converge significantly faster 
than their nonsmooth variants as we show in the scaling experiments below. 
This can be explained by the theoretical results of  \cite{shalev2014accelerated} 
on the convergence rate of SDCA.
They also had similar observations for the smoothed binary hinge loss.

\textbf{SDCA update:
\reftext{eq:topk-entropy}.}
Finally, we derive an optimization problem
for the proposed top-$k$ entropy loss.

\ifprop
\begin{proposition}\label{prop:topk-entropy-update}
Let $L$ in \eqref{eq:primal-dual} be the \reftext{eq:topk-entropy} loss
and $L^*$ be its convex conjugate as in \eqref{eq:softmax-conjugate}
with $\Sm$ replaced by $\Ska$.
The dual variables $a_i$ corresponding to $(x_i,y_i)$ are updated as:
\begin{align}\label{eq:topk-entropy-update}
\begin{cases}
\wo{a_i}{y_i} \hspace*{-.75em}
&= - \tfrac{1}{\lambda n} \argmin\limits_{x \in \Ska} \big\{
\tfrac{\alpha}{2} ( \inner{x,x} + s^2 )
- \inner{b, x} + \inner{x, \log x} \\[-.75em]
& \qquad\qquad\qquad\quad\;
+ (1 - s) \log(1 - s) \given s = \inner{\ones, x} \big\} , \\[-.5em]
a_{y_i,i} \hspace*{-.75em}
&= - \tsum_{j \neq y_i} a_{j,i} ,
\end{cases}\raisetag{.5\baselineskip}
\end{align}
where
$\alpha = \frac{\inner{x_i, x_i}}{\lambda n}$,
$b = \wo{q}{y_i} - q_{y_i}\ones$,
$q = \tra{W} x_i - \inner{x_i,x_i} a_i$.
\end{proposition}
\fi
\ifproof
\begin{proof}
Let $v \bydef - \lambda n a_i$ and $y = y_i$.
Using Proposition~\ref{prop:softmax-conjugate},
\begin{align*}
L^*(v) =
\tsum_{j \neq y} v_j \log v_j + (1 + v_y) \log(1 + v_y) ,
\end{align*}
where $\inner{\ones, v} = 0$ and $\wo{v}{y} \in \Ska$.
Let $x \bydef \wo{v}{y}$ and $s \bydef - v_y$.
We have $s = \inner{\ones, x}$ and from $\tr\left( A K \tra{A} \right)$ we get
\begin{align*}
K_{ii} (\inner{x,x} + s^2) / (\lambda n)^2 
- 2 \innerb{\wo{q}{y} - q_{y} \ones, x} / (\lambda n) ,
\end{align*}
where
$q = \sum_{j \neq i} K_{ij} a_j = A K_i - K_{ii} a_i$.
Finally, we plug everything together as in
Proposition~\ref{prop:smooth-topk-hinge-update}.
\end{proof}
\fi

Problems \eqref{eq:smooth-hinge-update} and \eqref{eq:topk-entropy-update}
have similar structure, but the latter is considerably
more difficult to solve due to the presence of logarithms.
We propose to tackle this problem using the function $V(t)$
introduced in \S~\ref{sec:optimization-background} above.

Our algorithm is an instance of the variable fixing scheme
with the following steps:
(i) partition the variables into disjoint sets
and compute an auxiliary variable $t$ from the optimality conditions;
(ii) compute the values of the variables using $t$ and verify them
against a set of constraints (\eg an upper bound in the top-$k$ simplex);
(iii) if there are no violated constraints, we have computed the solution,
and otherwise examine the next partitioning.

\begin{figure*}[t]\small\centering%
\begin{subfigure}[t]{0.49\columnwidth}\centering
\includegraphics[width=\columnwidth]{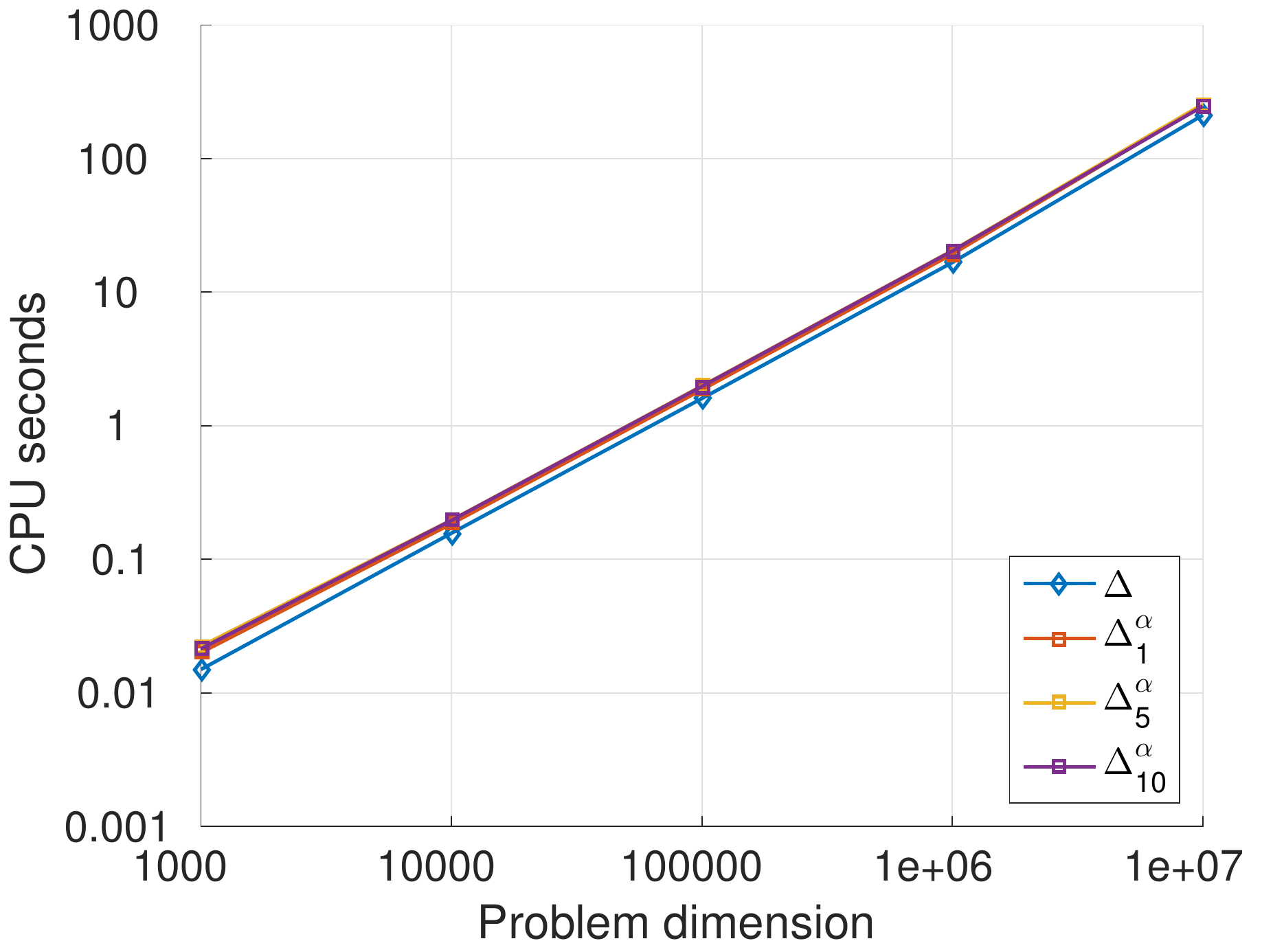}%
\caption{\centering Projection algorithms:\newline
$\Sm$ (Kiwiel \cite{kiwiel2008variable}),
$\Ska$ (ours \cite{lapin2015topk}).}\label{fig:time:a}
\end{subfigure}
\begin{subfigure}[t]{0.49\columnwidth}\centering
\includegraphics[width=\columnwidth]{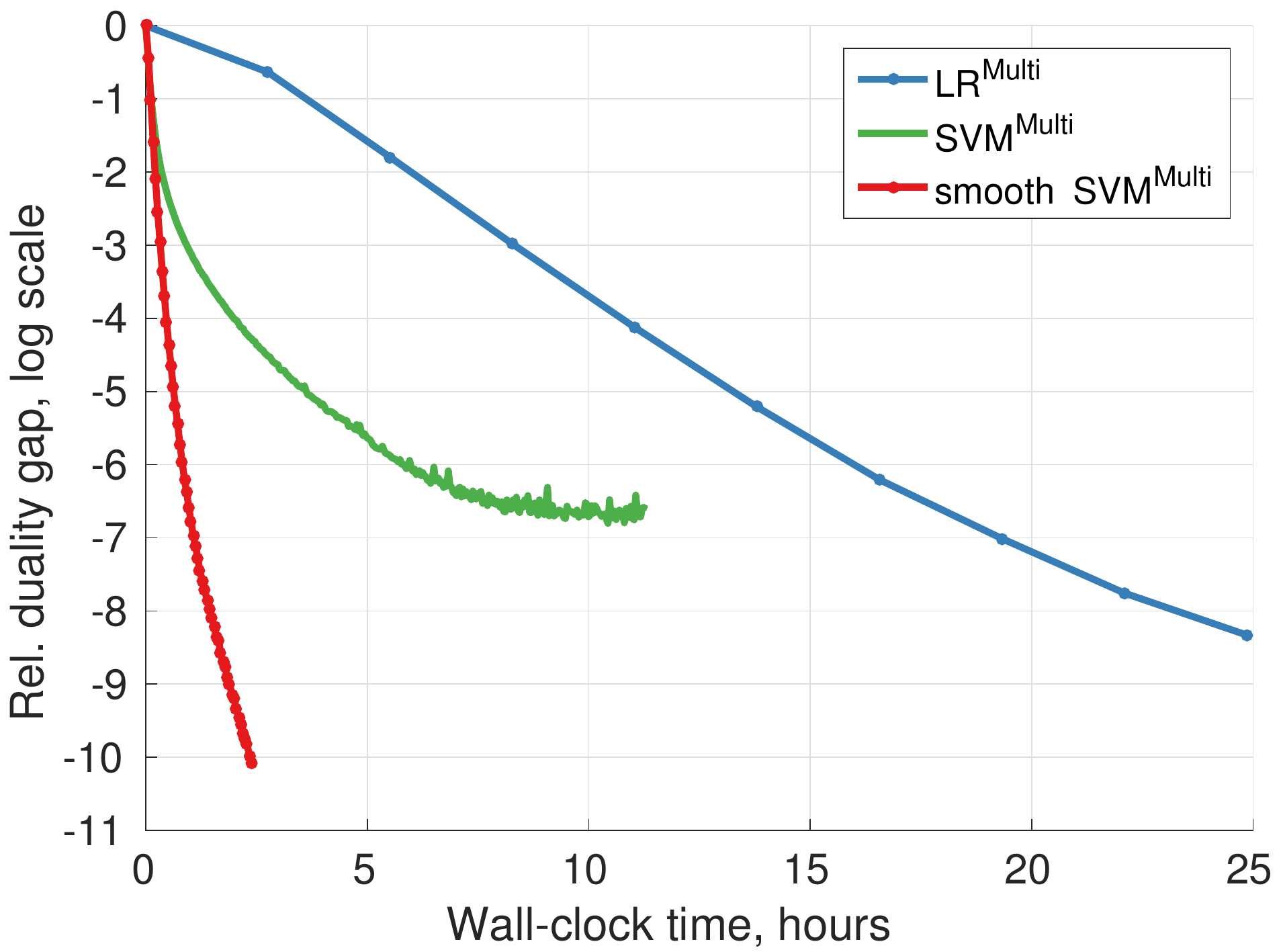}%
\caption{\centering Relative duality gap \V time
(our SDCA solvers).}\label{fig:time:b}
\end{subfigure}
\begin{subfigure}[t]{0.49\columnwidth}\centering
\includegraphics[width=\columnwidth]{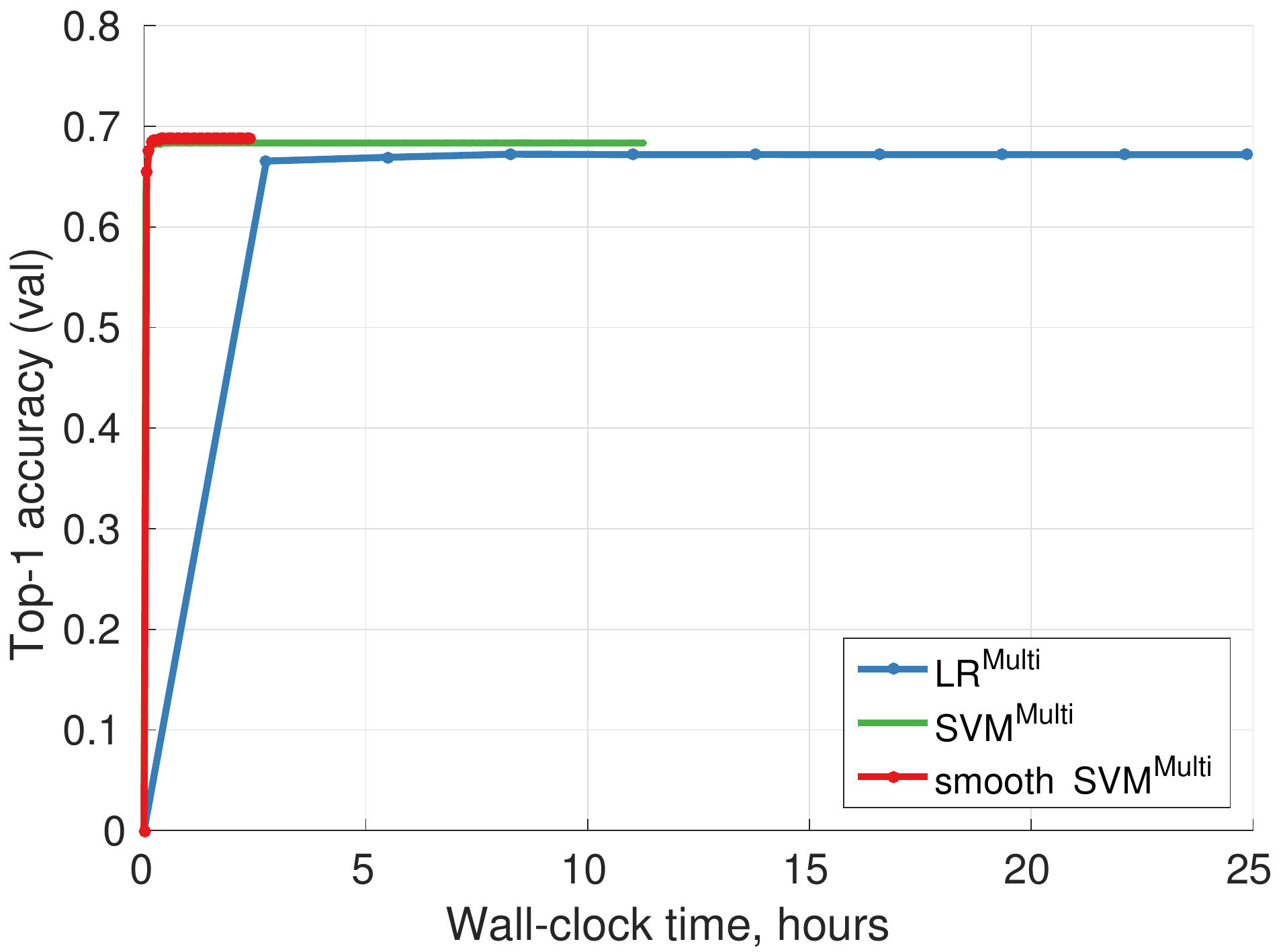}%
\caption{\centering Top-$1$ accuracy \V time\newline
(our SDCA solvers).}\label{fig:time:c}
\end{subfigure}
\begin{subfigure}[t]{0.49\columnwidth}\centering
\includegraphics[width=\columnwidth]{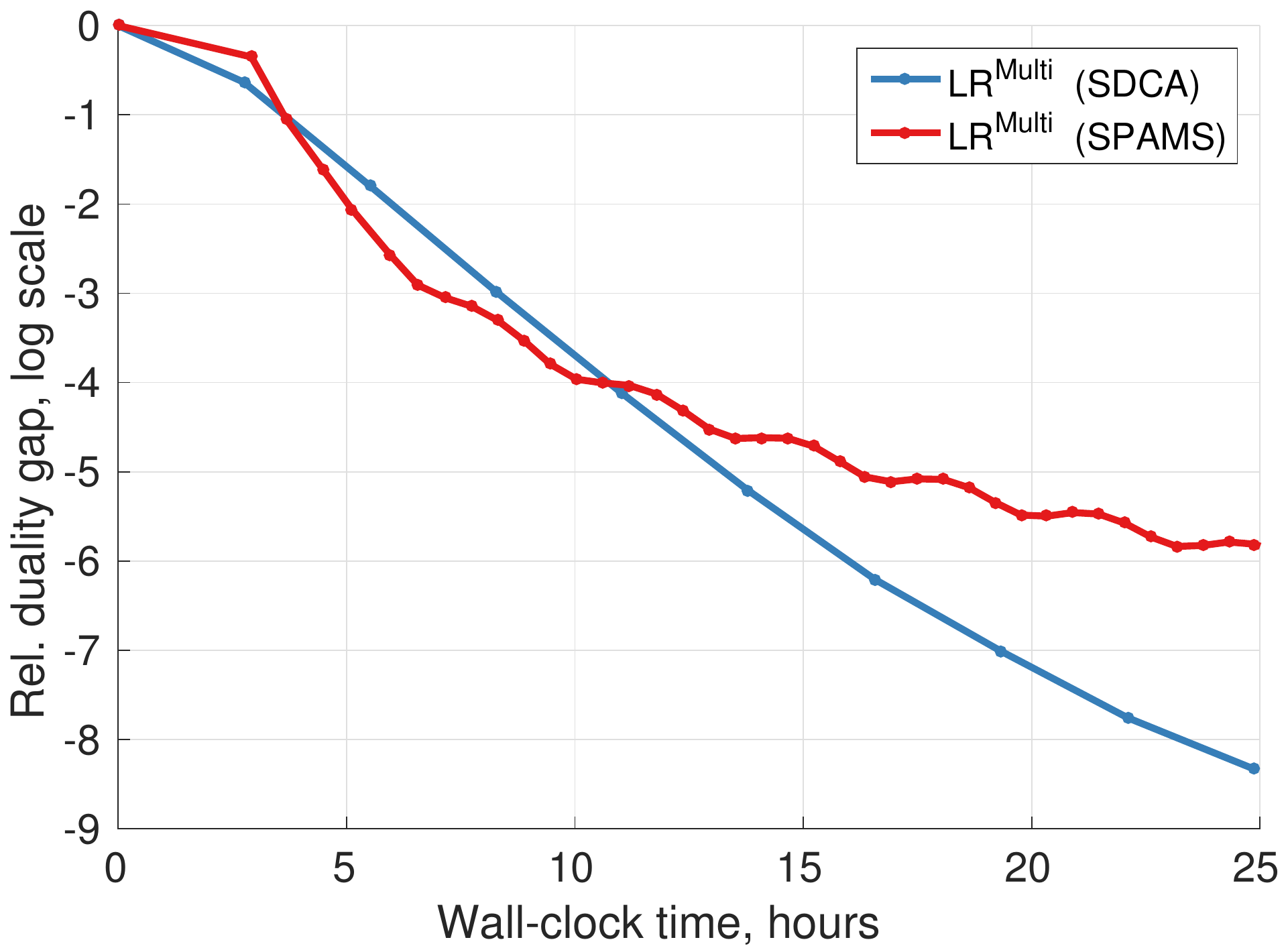}%
\caption{\centering Relative duality gap \V time
(ours and SPAMS \cite{mairal2010network}).}\label{fig:time:d}
\end{subfigure}
\caption{%
\textbf{(a)}
Scaling of the projection algorithms used in SDCA optimization.
\textbf{(b-c)}
SDCA convergence of the
\LrMulti, \SvmMulti, and smooth \SvmMultig{\gamma} methods
on the ImageNet 2012 dataset.
\textbf{(d)}
SDCA \V FISTA as implemented in the SPAMS toolbox.
}\label{fig:time}
\end{figure*}

As we discuss in \cite{lapin2015topk}, there can be at most $k$ partitionings
that we need to consider for $\Ska$ and $\Skb$.
To see this, let $x \in \Ska$ be a feasible point for
\eqref{eq:topk-entropy-update}, and define the subsets
\begin{align}\label{eq:sets-U-M}
U &\bydef \{ j \given x_j = \tfrac{s}{k} \} , &
M &\bydef \{ j \given x_j < \tfrac{s}{k} \} .
\end{align}
Clearly, $\abs{U} \leq k$ must hold,
and $\abs{U} = k$ we consider as a degenerate fall back case.
Therefore, we are primarily interested in the $k$ partitions when
$0 \leq \abs{U} < k$.
Due to monotonicity in the optimality conditions,
one can show that $U$ always corresponds to the largest elements $b_j$
of the vector being projected.
Hence, we start with an empty $U$ and
add indexes of the largest $b_j$'s until the solution is found.

Next, we show how to actually compute $t$ and $x$,
given a candidate partition into $U$ and $M$.

\ifprop
\begin{proposition}\label{prop:solve-topk-entropy-update}
Let $x^*$ be the solution of \eqref{eq:topk-entropy-update}
and let the sets $U$ and $M$ be defined for the given $x^*$
as in \eqref{eq:sets-U-M}, then
\begin{align*}
x_j^* &= \min \big\{ \tfrac{1}{\alpha} V(b_j - t) , \tfrac{s}{k} \big\} ,
\quad \forall\, j,
\end{align*}
and the variables $s$, $t$ satisfy the nonlinear system
\begin{align}\label{eq:solve-topk-entropy-update-system}
\begin{cases}
\alpha (1 - \rho) s - \tsum_{j \in M} V(b_j - t) = 0 , \\
(1-\rho)t + V^{-1}\big(\alpha (1-s) \big)
- \rho V^{-1}(\tfrac{\alpha s}{k}) + A - \alpha = 0 ,
\end{cases}\raisetag{1.5\baselineskip}
\end{align}
where $\rho \bydef \tfrac{\abs{U}}{k}$,
$A \bydef \tfrac{1}{k} \tsum_{j \in U} b_j$,
$V^{-1}$ is the inverse of $V$.

Moreover, if $U$ is empty, then
$x_j^* = \tfrac{1}{\alpha} V(b_j - t)$ for all $j$,
and $t$ can be found from
\begin{align}\label{eq:solve-topk-entropy-update-eq}
V(\alpha - t) + \tsum_j V(b_j - t) = \alpha .
\end{align}
\end{proposition}
\fi
\ifproof
\begin{proof}
The Lagrangian of \eqref{eq:topk-entropy-update} is given by
\begin{gather*}
\Lc(x, s, t, \lambda, \mu, \nu) =
\tfrac{\alpha}{2} ( \inner{x,x} + s^2 )
- \inner{b, x} + \inner{x, \log x} \\
+ (1-s) \log(1-s) + t(\inner{\ones, x} - s) \\
+ \lambda (s-1) - \inner{\mu, x}
+ \inner{\nu, x - \tfrac{s}{k}\ones} ,
\end{gather*}
where $t \in \Rb$, $\lambda, \mu, \nu \geq 0$ are the dual variables.
Computing partial derivatives of $\Lc$ \wrt $x_j$ and $s$,
and setting them to zero, we obtain
\begin{align*}
\alpha x_j + \log x_j &= b_j - 1 - t + \mu_j - \nu_j, \quad \forall j, \\
\alpha (1-s) + \log(1-s) &= \alpha - 1 - t
- \lambda - \tfrac{1}{k}\inner{\ones, \nu},
\quad \forall j .
\end{align*}
Note that only $x_j > 0$ and $s < 1$ satisfy the above constraints,
which implies $\mu_j = 0$ and $\lambda = 0$.
We re-write the above as
\begin{align*}
&\alpha x_j + \log(\alpha x_j) =
b_j - 1 - t + \log \alpha - \nu_j , \\
&\alpha (1-s) + \log\big(\alpha (1-s)\big) =
\alpha - 1 - t + \log \alpha - \tfrac{\inner{\ones, \nu}}{k} .
\end{align*}
These equations correspond to the Lambert $W$ function
of the exponent, $V(t) = W(e^t)$,
discussed in \S~\ref{sec:optimization-background}.
Let $p \bydef \inner{\ones, \nu}$ and
re-define $t \leftarrow 1 + t - \log \alpha$.
\begin{align*}
\alpha x_j &= W\big( \exp(b_j - t - \nu_j) \big) , \\
\alpha (1-s) &= W\big( \exp(\alpha - t - \tfrac{p}{k}) \big) .
\end{align*}
Finally, we obtain the following system:
\begin{align*}
\alpha x_j &= V(b_j - t - \nu_j) , \quad \forall j  \\
\alpha(1 -  s) &= V(\alpha - t - \tfrac{p}{k}) , \\
s &= \inner{\ones, x} , \quad
p = \inner{\ones, \nu} .
\end{align*}
Note that $V(t)$ is a strictly monotonically increasing function,
therefore, it is invertible and we can write
\begin{align*}
b_j -t - \nu_j &= V^{-1}(\alpha x_j) , \\
\alpha - t - \tfrac{p}{k} &= V^{-1}\big(\alpha (1-s) \big) .
\end{align*}
Next, we use the definition of the sets $U$ and $M$,
\begin{align*}
s &= \inner{\ones, x}
= \tsum_U \tfrac{s}{k} + \tsum_M \tfrac{1}{\alpha} V(b_j - t) , \\
p &= \inner{\ones, \nu}
= \tsum_U b_j - \abs{U}\big( t + V^{-1}(\tfrac{\alpha s}{k}) \big) .
\end{align*}
Let $\rho \bydef \tfrac{\abs{U}}{k}$
and $A \bydef \tfrac{1}{k} \tsum_U b_j$, we get
\begin{align*}
(1-\rho)s &= \tfrac{1}{\alpha} \tsum_M V(b_j - t) , \\
\tfrac{p}{k} &= A - \rho\big( t + V^{-1}(\tfrac{\alpha s}{k}) \big) .
\end{align*}
Finally, we eliminate $p$ and obtain the system:
\begin{align*}
&\alpha (1-\rho)s - \tsum_M V(b_j - t) = 0 , \\
&(1-\rho)t + V^{-1}\big(\alpha (1-s) \big)
- \rho V^{-1}(\tfrac{\alpha s}{k}) + A - \alpha = 0 .
\end{align*}
Moreover, when $U$ is empty, it simplifies into a single equation
\begin{align*}
V(\alpha - t) + \tsum_M V(b_j - t) = \alpha .
\end{align*}
\end{proof}
\fi

We solve \eqref{eq:solve-topk-entropy-update-system}
using the Newton's method \cite{nocedal2006numerical},
while for the Eq.~\eqref{eq:solve-topk-entropy-update-eq}
we use a $4$-th order Householder's method \cite{householder1970numerical}
with a faster convergence rate.
The latter is particularly attractive,
since the set $U$ can always be assumed empty for $k=1$,
\ie for the \reftext{eq:softmax} loss,
and is often also empty for the general \reftext{eq:topk-entropy} loss.
As both methods require the derivatives of $V(t)$, we note that
$\partial_t V(t) = V(t) / (1 + V(t))$ \cite{corless1996lambertw},
which means that the derivatives come at no additional cost.
Finally, we note that $V^{-1}(v) = v + \log v$ by definition.

\textbf{Loss computation:
\SvmMultig{\gamma}, \SvmTopKabg{k}{\gamma}.}
Here, we discuss how to evaluate smoothed losses that do not have
a closed-form expression for the primal loss.
Recall that the smooth \reftext{eq:smooth-topk-alpha} loss
is given by
\begin{align*}
L_{\gamma}(a) &= \tfrac{1}{\gamma} \big(
\innern{\wo{(a + c)}{y}, p} - \tfrac{1}{2} \normn{p}^2 \big) ,
\end{align*}
where
$a_j = f_j(x) - f_y(x)$, $c_j = 1 - \iv{y = j}$
for all $j \in \Yc$,
and $p = \proj_{\Ska(\gamma)}\wo{(a + c)}{y}$
is the Euclidean projection of $\wo{(a + c)}{y}$ onto $\Ska(\gamma)$.
We describe an $O(m \log m)$ algorithm to compute the projection $p$
in \cite{lapin2015topk}.
For the special case $k=1$, \ie the \SvmMultig{\gamma} loss,
the algorithm is particularly efficient and
exhibits essentially linear scaling in practice.
Moreover, since we only need the dot products with $p$ in $L_{\gamma}(a)$,
we exploit its special structure,
$p = \min \{ \max \{l, b - t \}, u \}$
with $b = \wo{(a + c)}{y}$,
and avoid explicit computation of $p$.
The same procedure is done for the \SvmTopKbg{k}{\gamma} loss.

\textbf{Loss computation:
\reftext{eq:topk-entropy}.}
Next, we discuss how to evaluate the \reftext{eq:topk-entropy}
loss that was defined via the conjugate of the softmax loss as
\begin{align}\label{eq:topk-entropy-primal}
\max_{x \in \Ska, \, s = \inner{\ones, x}} \big\{
\innern{\wo{a}{y},x} - (1-s) \log(1-s) - \inner{x, \log x} \big\} .
\raisetag{.4\baselineskip}
\end{align}
Note that \eqref{eq:topk-entropy-primal}
is similar to \eqref{eq:topk-entropy-update}
and we use a similar variable fixing scheme, as described above.
However, this problem is much easier: the auxiliary variables $s$ and $t$
are computed directly without having to solve a nonlinear system,
and their computation does not involve the $V(t)$ function.

\ifprop
\begin{proposition}\label{prop:solve-topk-entropy-primal}
Let $x^*$ be the solution of \eqref{eq:topk-entropy-primal}
and let the sets $U$ and $M$ be defined for the given $x^*$
as in \eqref{eq:sets-U-M}, then
\begin{align*}
x_j^* &= \min \big\{ \exp(a_j - t), \tfrac{s}{k} \big\} ,
\quad \forall\, j,
\end{align*}
and the variables $s$, $t$ are computed from
\begin{align}\label{eq:solve-topk-entropy-primal-system}
\begin{cases}
s = 1 / (1 + Q) , \\
t = \log Z + \log(1 + Q) - \log(1 - \rho) ,
\end{cases}%
\end{align}
where $\rho \bydef \tfrac{\abs{U}}{k}$,
$A \bydef \tfrac{1}{k} \tsum_{j \in U} a_j$,
$Z \bydef \tsum_{j \in M} \exp a_j$,
and
\begin{align*}
Q \bydef (1 - \rho)^{(1 - \rho)} / (k^\rho Z^{(1-\rho)} \exp A) .
\end{align*}
The \reftext{eq:topk-entropy} loss is then computed as
\begin{align*}
L(a) =
(A + (1 - \rho) t - \rho \log(\tfrac{s}{k})) s - (1-s)\log(1-s).
\end{align*}

Moreover, if $U$ is empty, then
$x_j^* = \exp(a_j - t)$ for all $j$,
and we recover the softmax loss \reftext{eq:softmax} as
\begin{align*}
L(a) = t = \log(1 + Z) = \log(1 + \tsum_j \exp a_j).
\end{align*}
\end{proposition}
\fi
\ifproof
\begin{proof}
We continue the derivation started in the proof of
Propostion~\ref{prop:topk-entropy-primal}.
First, we write the system that follows directly
from the KKT \cite{boyd2004convex} optimality conditions.
\begin{equation}\label{eq:topk-entropy-kkt}
\begin{aligned}
x_j &= \min\{ \exp(a_j - t), \tfrac{s}{k} \} , \quad \forall j, \\
\nu_j &= \max\{ 0, a_j - t - \log(\tfrac{s}{k}) \} , \quad \forall j, \\
1-s &= \exp(-t - \tfrac{p}{k}) , \\
s &= \inner{\ones, x} , \quad
p = \inner{\ones, \nu} .
\end{aligned}
\end{equation}
Next, we define the two index sets $U$ and $M$ as follows
\begin{align*}
U &\bydef \{ j \given x_j = \tfrac{s}{k} \}, &
M &\bydef \{ j \given x_j < \tfrac{s}{k} \} .
\end{align*}
Note that the set $U$ contains at most $k$ indexes corresponding
to the largest components of $a_j$.
Now, we proceed with finding a $t$ that solves (\ref{eq:topk-entropy-kkt}).
Let $\rho \bydef \frac{\abs{U}}{k}$.
We eliminate $p$ as
\begin{align*}
p &= \sum_j \nu_j
= \sum_U a_j - \abs{U} \big(t + \log(\tfrac{s}{k}) \big)
\quad \Longrightarrow \\
\tfrac{p}{k}
&= \tfrac{1}{k} \sum_U a_j - \rho \big(t + \log(\tfrac{s}{k})\big) .
\end{align*}
Let $Z \bydef \sum_M \exp a_j$, we write for $s$
\begin{align*}
s &= \sum_j x_j
= \sum_U \tfrac{s}{k} + \sum_M \exp(a_j - t) \\
&= \rho s + \exp(-t) \sum_M \exp a_j
= \rho s + \exp(-t) Z .
\end{align*}
We conclude that
\begin{align*}
(1 - \rho) s &= \exp(-t) Z \quad \Longrightarrow \\
t &= \log Z - \log\big( (1-\rho)s \big) .
\end{align*}
Let $A \bydef \tfrac{1}{k} \sum_U a_j$.
We further write
\begin{gather*}
\log(1-s) = -t - \tfrac{p}{k}
= -t - A + \rho \big(t + \log(\tfrac{s}{k})\big) \\
= \rho \log(\tfrac{s}{k}) - A
- (1 - \rho) \big[ \log Z - \log\big( (1-\rho)s \big) \big] ,
\end{gather*}
which yields the following equation for $s$
\begin{align*}
&\log(1-s) - \rho ( \log s - \log k)  + A \\
&+ (1 - \rho) \big[ \log Z - \log(1-\rho) - \log s \big] = 0.
\end{align*}
Therefore,
\begin{align*}
&\log(1-s) - \log s + \rho \log k + A \\
&+ (1-\rho) \log Z
- (1-\rho)\log(1-\rho) = 0 , \\
&\log\left( \frac{1-s}{s} \right) =
\log\left(
\frac{(1-\rho)^{(1-\rho)} \exp(-A)}{k^\rho Z^{(1-\rho)}}
\right) .
\end{align*}
We finally get
\begin{equation*}
\begin{aligned}
s &= 1 / (1 + Q), \\
Q &\bydef (1-\rho)^{(1-\rho)} / (k^\rho Z^{(1-\rho)} e^A) .
\end{aligned}
\end{equation*}
We note that:
\emph{a)} $Q$ is readily computable once the sets $U$ and $M$ are fixed;
and
\emph{b)} $Q = 1/Z$ if $k=1$ since $\rho = A = 0$ in that case.
This yields the formula for $t$ as
\begin{align*}
t = \log Z + \log(1 + Q) - \log(1-\rho).
\end{align*}
As a sanity check, we note that we again recover the softmax loss
for $k=1$, since
$t = \log Z + \log(1 + 1/Z) = \log(1 + Z) = \log(1 + \sum_j \exp a_j)$.

To verify that the computed $s$ and $t$ are compatible with the choice
of the sets $U$ and $M$, we check if this holds:
\begin{align*}
\exp(a_j - t) &\geq \tfrac{s}{k}, \quad \forall j \in U , \\
\exp(a_j - t) &\leq \tfrac{s}{k}, \quad \forall j \in M ,
\end{align*}
which is equivalent to
\begin{align*}
\max_M a_j \leq \log(\tfrac{s}{k}) + t \leq \min_U a_j .
\end{align*}

To compute the actual loss \eqref{eq:topk-entropy-primal}, we have
\begin{align*}
&\inner{a,x} - \inner{x, \log x} - (1-s)\log(1-s) \\
&= \sum_U a_j \tfrac{s}{k} + \sum_M a_j \exp(a_j - t)
- \sum_U \tfrac{s}{k} \log(\tfrac{s}{k}) \\
&- \sum_M (a_j - t) \exp(a_j - t)
- (1-s)\log(1-s) \\
&= A s - \rho s \log(\tfrac{s}{k})
+ t \exp(-t) Z - (1-s)\log(1-s) \\
&= A s - \rho s \log(\tfrac{s}{k})
+ (1-\rho) s t - (1-s)\log(1-s) .
\end{align*}
\end{proof}
\fi

As before, we only need to examine at most $k$ partitions $U$,
adding the next maximal $a_j$ to $U$ until there are no violated constraints.
Therefore, the overall complexity of the procedure
to compute the \reftext{eq:topk-entropy} loss is $O(km)$.

The efficiency of the outlined approach for optimizing
the \reftext{eq:topk-entropy} loss
crucially depends on fast computation of $V(t)$ in the SDCA update.
Our implementation was able to scale to large datasets as we show next.

\textbf{Runtime evaluation.}
First, we highlight the efficiency of our algorithm from \cite{lapin2015topk}
for computing the Euclidean projection onto the top-$k$ simplex,
which is used, in particular, for optimization of the \SvmMultig{\gamma} loss.
The scaling plot is given in Figure~\ref{fig:time:a}
and shows results of an experiment following \cite{liu2009efficient}.
We sample $1000$ points from the normal distribution $\Nc(0,1)$
and solve the projection problems using
the algorithm of Kiwiel \cite{kiwiel2008variable} (denoted as $\Sm$)
and using our method of projecting onto
the set $\Ska$ for different values of $k=1,5,10$.
We report the total CPU time taken on a single
Intel(R) Xeon(R) E5-2680 2.70GHz processor.
As was also observed by \cite{kiwiel2008variable},
we see that the scaling is essentially linear in the problem dimension
and makes the method applicable to problems with a large number of classes.

Next in Figure~\ref{fig:time}, we compare the wall-clock training time
of \reftext{eq:multi-hinge} with a smoothed \SvmMultig{\gamma}
and the \reftext{eq:softmax} objectives.
We plot the relative duality gap (\ref{fig:time:b})
and the validation accuracy (\ref{fig:time:c}) versus time
for the best performing models on the ImageNet 2012 benchmark.
We obtain substantial improvement of the convergence rate
for the smooth \SvmMultig{\gamma} compared to the nonsmooth baseline.
Moreover, we see that the top-$1$ accuracy saturates
after a few passes over the training data,
which justifies the use of a fairly loose stopping criterion
(we use $\varepsilon = 10^{-3}$).
For the \reftext{eq:softmax} loss,
the cost of each epoch is significantly higher
compared to \reftext{eq:multi-hinge},
which is due to the difficulty of solving \eqref{eq:topk-entropy-update}.
This suggests that the smooth \SvmTopKag{1}{1} loss
can offer competitive performance (see \S~\ref{sec:experiments})
at a lower training cost.

Finally, we also compare our implementation of \reftext{eq:softmax}
(marked SDCA in \ref{fig:time:d})
with the SPAMS optimization toolbox \cite{mairal2010network},
which provides an efficient implementation of FISTA \cite{beck2009fast}.
We note that the rate of convergence of SDCA is competitive
with FISTA for $\epsilon \geq 10^{-4}$
and is noticeably better for $\epsilon < 10^{-4}$.
We conclude that our approach for training the \reftext{eq:softmax} model
is competitive with the state-of-the-art,
and faster computation of $V(t)$ can lead to a further speedup.

\subsection{Multilabel Methods}\label{sec:optimization-multilabel}

This section covers optimization of the multilabel objectives
introduced in \S~\ref{sec:multilabel-methods}.
First, we reduce computation of the SDCA update step and evaluation of the smoothed loss
\reftext{eq:smooth-svm-ml} to the problem of computing the Euclidean projection
onto what we called the bipartite simplex $B(r)$, see Eq.~\eqref{eq:bipartite-simplex}.
Next, we contribute a novel variable fixing algorithm for computing that projection.
Finally, we discuss SDCA optimization of the multilabel cross-entropy loss
\reftext{eq:lr-ml}.

\textbf{SDCA update:
\reftext{eq:svm-ml}, \reftext{eq:smooth-svm-ml}.}
Here, we discuss optimization of the smoothed \reftext{eq:smooth-svm-ml} loss.
The update step for the nonsmooth counterpart is recovered
by setting $\gamma = 0$.

\ifprop
\begin{proposition}\label{prop:smooth-svm-ml-update}
Let $L$ and $L^*$ in \eqref{eq:primal-dual} be respectively
the \reftext{eq:smooth-svm-ml} loss and its conjugate as in
Proposition~\ref{prop:smooth-svm-ml}.
The dual variables $a \bydef a_i$ corresponding to
the training pair $(x_i,Y_i)$ are updated as
$(a_y)_{y \in Y_i} = p$ and
$(a_j)_{j \in \bar{Y}_i} = -\bar{p}$,
where
\begin{align*}
(p,\bar{p}) = \proj_{B(1 / \lambda n)}(b,\bar{b}) ,
\end{align*}
$b = \rho \big( \tfrac{1}{2} - q_y \big)_{y \in Y_i}$,
$\bar{b} = \rho \big( \tfrac{1}{2} + q_j \big)_{j \in \bar{Y}_i}$,
$q = \tra{W} x_i - \inner{x_i,x_i} a_i$,
and
$\rho = \tfrac{1}{\inner{x_i,x_i} + \gamma \lambda n}$.
\end{proposition}
\fi
\ifproof
\begin{proof}
We update the dual variables $a \bydef a_i \in \Rb^m$
corresponding to the training example $(x_i, Y_i)$ by
solving the following optimization problem.
\begin{align*}
\max_{a \in \Rb^m} \;
& -\frac{1}{n} L_{\gamma}^*(Y_i, - \lambda n a) - \frac{\lambda}{2}
\tr(A K \tra{A}) ,
\end{align*}
where $\lambda > 0$ is a regularization parameter.
Equivalently, we can divide both the primal and the dual objectives
by $\lambda$ and use $C \bydef \frac{1}{\lambda n} > 0$
as the regularization parameter instead.
The optimization problem becomes
\begin{align}\label{eq:update-problem}
\max_{a \in \Rb^m} \;
& -C L^* \Big( Y_i, - \frac{1}{C} a \Big) - \frac{1}{2} \tr(A K \tra{A}) .
\end{align}
Note that
\begin{align*}
\tr(A K \tra{A}) = K_{ii} \inner{a, a}
+ 2 \sum_{j \neq i} K_{ij} \inner{a_j, a} + {\rm const} ,
\end{align*}
where the ${\rm const}$ does not depend on $a$.
We ignore that constant in the following derivation
and also define an auxiliary vector
$q \bydef \sum_{j \neq i} K_{ij} a_j = A K_i - K_{ii} a_i$ .
Plugging the conjugate from Proposition~\ref{prop:smooth-svm-ml}
into \eqref{eq:update-problem}, we obtain
\begin{align*}
\max_{a \in \Rb^m} \;
& -C \Big(
\frac{1}{2 C} \big( - \tsum_{y \in Y_i} a_y + \tsum_{j \in \bar{Y}_i} a_j \big)
+ \frac{\gamma}{2 C^2} \norms{a} \Big) \\
&- (1/2) \big( K_{ii} \norms{a} + 2 \inner{q, a} \big) \\
\st \; & - \tfrac{1}{C} a \in S_{Y_i}
\end{align*}
We re-write the constraint $- \frac{1}{C} a \in S_{Y_i}$ as
\begin{align*}
&\tsum_{y \in Y_i} a_y = - \tsum_{j \in \bar{Y}_i} a_j \leq C \\
& a_y \geq 0, \;  \forall \, y \in Y_i ; \;\;
a_j \leq 0, \; \forall \, j \in \bar{Y}_i;
\end{align*}
and switch to the equivalent minimization problem below.
\begin{align*}
\min_{a \in \Rb^m} \;\;
& \tfrac{1}{2} \big( K_{ii} + \tfrac{\gamma}{C} \big) \norms{a}
- \tfrac{1}{2} \tsum_{y \in Y_i} a_y
- \tfrac{1}{2} \tsum_{j \in \bar{Y}_i} (- a_j) \\
& + \inner{q, a} \\
& \tsum_{y \in Y_i} a_y = \tsum_{j \in \bar{Y}_i} - a_j \leq C \\
& a_y \geq 0, \;  \forall \, y \in Y_i ; \;\;
- a_j \geq 0, \; \forall \, j \in \bar{Y}_i .
\end{align*}
Note that
\begin{gather*}
- \tfrac{1}{2} \tsum_{y \in Y_i} a_y
- \tfrac{1}{2} \tsum_{j \in \bar{Y}_i} (- a_j)
+ \inner{q, a} \\
= - \tsum_{y \in Y_i} ( \tfrac{1}{2} - q_y ) a_y
- \tsum_{j \in \bar{Y}_i} ( \tfrac{1}{2} + q_j ) (- a_j) ,
\end{gather*}
and let us define
\begin{align*}
x &\bydef (a_y)_{y \in Y_i} \in \Rb^{\abs{Y_i}} , &
b &\bydef \tfrac{1}{K_{ii} + \gamma / C}
(\tfrac{1}{2} - q_y)_{y \in Y_i} \in \Rb^{\abs{Y_i}} , \\
y &\bydef (-a_j)_{j \in \bar{Y}_i} \in \Rb^{\abs{\bar{Y}_i}} , &
\bar{b} &\bydef \tfrac{1}{K_{ii} + \gamma / C}
(\tfrac{1}{2} + q_j)_{j \in \bar{Y}_i} \in \Rb^{\abs{\bar{Y}_i}} .
\end{align*}
The final projection problem for the update step is
\begin{equation}\label{eq:proj_problem}
\begin{aligned}
\min_{x, y} \;\;
& \tfrac{1}{2} \norms{x - b} + \tfrac{1}{2} \norms{y - \bar{b}} \\
& \inner{\ones, x} = \inner{\ones, y} \leq C \\
& x \geq 0, \quad y \geq 0 .
\end{aligned}
\end{equation}
\end{proof}
\fi

Let us make two remarks regarding optimization of the multilabel SVM.
First, we see that the update step involves exactly the same projection
that was used in Proposition~\ref{prop:smooth-svm-ml}
to define the smoothed \reftext{eq:smooth-svm-ml} loss,
with the difference in the vectors being projected and the radius of
the bipartite simplex.
Therefore, we can use the same projection algorithm both during optimization
as well as when computing the loss.
And second, even though \reftext{eq:svm-ml} reduces to \reftext{eq:multi-hinge}
when $Y_i$ is singleton, the derivation of the smoothed loss and the projection
algorithm proposed below for the bipartite simplex are substantially different
from what we proposed in the multiclass setting.
Most notably, the treatment of the dimensions in $Y_i$ and $\bar{Y}_i$ is now symmetric.

\textbf{Loss computation:
\reftext{eq:smooth-svm-ml}.}
The smooth multilabel SVM loss \reftext{eq:smooth-svm-ml} is given by
\begin{align*}
L_{\gamma}(u) &= \tfrac{1}{\gamma}\big(
\inner{b,p} - \tfrac{1}{2}\norms{p} +
\inner{\bar{b},\bar{p}} - \tfrac{1}{2}\norms{\bar{p}} \big),
\end{align*}
where
$b = \big( \tfrac{1}{2} - u_y \big)_{y \in Y}$,
$\bar{b} = \big( \tfrac{1}{2} + u_j \big)_{j \in \bar{Y}}$,
$u = f(x)$,
and
$(p,\bar{p}) = \proj_{B(\gamma)}(b,\bar{b})$.
Below, we propose an efficient variable fixing algorithm
to compute the Euclidean projection onto $B(\gamma)$.
We also note that we can use the same trick that we used for
\reftext{eq:smooth-topk-alpha} and
exploit the special form of the projection to avoid explicit computation
of $p$ and $\bar{p}$.

\textbf{Euclidean projection onto the bipartite simplex $B(\rho)$.}
The optimization problem that we seek to solve is:
\begin{equation}\label{eq:proj-bipartite-simplex}
\begin{aligned}
(p, \bar{p}) = \argmin_{x \in \Rb^m_+, \, y \in \Rb^n_+} \;\;
& \tfrac{1}{2} \norms{x - b} + \tfrac{1}{2} \norms{y - \bar{b}} \\[-.5em]
& \inner{\ones, x} = \inner{\ones, y} \leq \rho .
\end{aligned}
\end{equation}
This problem has been considered by
Shalev-Shwartz and Singer \cite{shalev2006efficient},
who proposed a breakpoint searching algorithm based on sorting,
as well as by Liu and Ye \cite{liu2009efficient},
who formulated it as a root finding problem that is solved via bisection.
Next, we contribute a novel variable fixing algorithm
that is inspired by the algorithm of Kiwiel \cite{kiwiel2008variable}
for the continuous quadratic knapsack problem
(\aka projection onto simplex).

\begin{enumerate}[leftmargin=1.25em]
\item \emph{Initialization.}
Define the sets $I_x = \{1, \ldots, m\}$, $L_x = \{\}$,
$I_y = \{1, \ldots, n\}$, $L_y = \{\}$,
and solve the independent subproblems below using the algorithm of
\cite{kiwiel2008variable}.
\begin{align*}
p &= \argmin_{x \in \Rb^m_+} \big\{
\tfrac{1}{2} \norms{x - b} \given \inner{\ones, x} = \rho \big\} , \\
\bar{p} &= \argmin_{y \in \Rb^n_+} \big\{
\tfrac{1}{2} \norms{x - \bar{b}} \given \inner{\ones, y} = \rho \big\} .
\end{align*}
Let $t'$ and $s'$ be the resulting optimal thresholds, such that
$p = \max\{ 0, b - t' \}$ and $\bar{p} = \max\{ 0, \bar{b} - s' \}$.
If $t' + s' \geq 0$, then $(p,\bar{p})$ is the solution to
\eqref{eq:proj-bipartite-simplex}; stop.

\item\label{alg:step1}
\emph{Restricted subproblem.}
Compute $t$ as
\begin{align*}
t &= \big( \tsum_{I_x} b_j - \tsum_{I_y} \bar{b}_j \big)
/ (\abs{I_x} + \abs{I_y}) ,
\end{align*}
and let $x_j(t) = b_j - t$, $y_j(t) = \bar{b}_j + t$.

\item \emph{Feasibility check.}
Compute
\begin{align*}
\Delta_x &= \tsum_{I_x^L} (b_j - t),
\text{ where } I_x^L = \{ j \in I_x \given x_j(t) \leq 0 \} , \\
\Delta_y &= \tsum_{I_y^L} (\bar{b}_j + t),
\text{ where } I_y^L = \{ j \in I_y \given y_j(t) \leq 0 \} .
\end{align*}

\item \emph{Stopping criterion.}
If $\Delta_x = \Delta_y$, then the solution to
\eqref{eq:proj-bipartite-simplex} is given by
$p = \max\{ 0, b - t \}$ and $\bar{p} = \max\{ 0, \bar{b} + t \}$;
stop.

\item \emph{Variable fixing.}
If $\Delta_x > \Delta_y$, update
$I_x \leftarrow I_x \setminus I_x^L$,
$L_x \leftarrow L_x \cup I_x^L$.
If $\Delta_x < \Delta_y$, update
$I_y \leftarrow I_y \setminus I_y^L$,
$L_y \leftarrow L_y \cup I_y^L$.
Go to step~\ref{alg:step1}.
\end{enumerate}

\ifprop
\begin{proposition}\label{prop:proj-bipartite-simplex}
The algorithm above solves \eqref{eq:proj-bipartite-simplex}.
\end{proposition}
\fi
\ifproof
\begin{proof}
We sketch the main parts of the proof that show correctness of the algorithm.
A complete and formal derivation would follow the proof given in
\cite{kiwiel2008variable}.

The Lagrangian for the optimization problem
\eqref{eq:proj-bipartite-simplex} is
\begin{gather*}
\Lc(x,y,t,s,\lambda,\mu,\nu) 
= \tfrac{1}{2} \norms{x - b} + \tfrac{1}{2} \norms{y - \bar{b}} \\
+ t (\inner{\ones, x} - r)
+ s (\inner{\ones, y} - r)
+ \lambda (r - \rho)
- \inner{\mu, x} - \inner{\nu, y} ,
\end{gather*}
and it leads to the following KKT conditions
\begin{equation}\label{eq:proj_kkt}
\begin{aligned}
x_j &= b_j - t + \mu_j, & \mu_j x_j &= 0 , & \mu_j &\geq 0 , \\
y_k &= \bar{b}_k - s + \nu_k, & \nu_k y_k &= 0 , & \nu_k &\geq 0 , \\
\lambda &= t + s, & \lambda (r - \rho) &= 0 , & \lambda &\geq 0 .
\end{aligned}
\end{equation}

If $\rho = 0$, the solution is trivial. Assume $\rho > 0$ and let
\begin{align*}
x(t) &= \max\{ 0, b - t \}, &
y(s) &= \max\{ 0, \bar{b} - s \},
\end{align*}
where $t$, $s$
are the dual variables from \eqref{eq:proj_kkt} and we have
\begin{align*}
(t + s) (r - \rho) &= 0, &
t + s &\geq 0, &
0 \leq r &\leq \rho .
\end{align*}
We define index sets for $x$ as
\begin{align*}
I_x &= \{ j \given b_j - t > 0 \}, &
L_x &= \{ j \given b_j - t \leq 0 \}, &
m_x &= \abs{I_x},
\end{align*}
and similar sets $I_y$, $L_y$ for $y$.
Solving a reduced subproblem
\begin{align*}
\min \{ \tfrac{1}{2} \norms{x - b} \given \inner{\ones, x} = r \} ,
\end{align*}
for $t$ and a similar problem for $s$, yields
\begin{align}\label{eq:solve_ts}
t &= \tfrac{1}{m_x} \big( \tsum_{j \in I_x} b_j - r \big), &
s &= \tfrac{1}{m_y} \big( \tsum_{j \in I_y} \bar{b}_j - r \big) .
\end{align}

We consider two cases: $r = \rho$ and $r < \rho$.
If $r = \rho$, then we have two variables $t$ and $s$ to optimize over,
but the optimization problem \eqref{eq:proj-bipartite-simplex}
decouples into two simplex projection problems which can
be solved independently.
\begin{equation}\label{eq:solve_xy}
\begin{aligned}
&\min \{
\tfrac{1}{2} \norms{x - b} \given
\inner{\ones, x} = \rho, \; x_j \geq 0 \} , \\
&\min \{
\tfrac{1}{2} \norms{y - \bar{b}} \given
\inner{\ones, y} = \rho, \; y_j \geq 0 \} .
\end{aligned}
\end{equation}
Let $t'$ and $s'$ be solutions to the independent problems
\eqref{eq:solve_xy}.
If $t' + s' \geq 0$, we have that the KKT conditions \eqref{eq:proj_kkt}
are fulfilled and we have, therefore, the solution to the original problem
\eqref{eq:proj-bipartite-simplex}.
Otherwise, we have that the optimal $t^* + s^* > t' + s'$ and so
at least one of the two variables must increase.
Let $t^* > t'$, then $\inner{\ones, x(t^*)} < \inner{\ones, x(t')} = \rho$,
therefore $r^* < \rho$.

If $r < \rho$, then $t + s = 0$. We eliminate $s$, which leads to
\begin{align*}
\tfrac{1}{m_x} \big( \tsum_{j \in I_x} b_j - r \big)
= -
\tfrac{1}{m_y} \big( \tsum_{j \in I_y} \bar{b}_j - r \big) .
\end{align*}
This can now be solved for $r$ as
\begin{align}\label{eq:solve_r}
r = \big(
m_y \tsum_{I_x} b_j + m_x \tsum_{I_y} \bar{b}_j \big) / (m_x + m_y) .
\end{align}
One can verify that $r < \rho$ if $r$ is computed by \eqref{eq:solve_r}
and $t' + s' < 0$.
Plugging \eqref{eq:solve_r} into \eqref{eq:solve_ts}, we get
\begin{align}\label{eq:solve_t}
t &= \big(
\tsum_{I_x} b_j - \tsum_{I_y} \bar{b}_j
\big) / (m_x + m_y) .
\end{align}
One can further verify that $t > t'$ and $-t > s'$,
where $t$ is computed by \eqref{eq:solve_t},
$t'$, $s'$ are computed by \eqref{eq:solve_ts} with $r = \rho$,
and $t' + s' < 0$.
Therefore, if $x_j(t') = 0$ for some $j \in L_x(t')$,
then $x_j(t) = 0$, and so $L_x(t') \subset L_x(t)$.
The variables that were fixed to the lower bound while solving
\eqref{eq:solve_xy} with $r = \rho$ remain fixed
when considering $r < \rho$.
\end{proof}
\fi

The proposed algorithm is easy to implement, does not require sorting,
and scales well in practice, as demonstrated by our experiments
on VOC 2007 and MS COCO.

\textbf{Runtime evaluation.}
We also compare the runtime of the proposed variable fixing algorithm and
the sorting based algorithm of \cite{shalev2006efficient}.
We perform no comparison to \cite{liu2009efficient}
as their code is not available.
Furthermore, the algorithms that we consider are exact,
while the method of \cite{liu2009efficient} is approximate
and its runtime is dependent on the required precision.
The experimental setup is the same as in \S~\ref{sec:optimization-multiclass}
above, and our results are reported in Table~\ref{tbl:time-ml}.

\begin{table}[H]\footnotesize\centering\setlength{\tabcolsep}{.85em}
\begin{tabular}{l|ccccc}
\toprule
Dimension $d$ & $10^3$ & $10^4$ & $10^5$ & $10^6$ & $10^7$ \\
\midrule
\midrule
Sorting based \cite{shalev2006efficient} &
$0.07$ & $0.56$ & $6.92$ & $85.56$ & $1364.94$ \\
Variable fixing (ours) &
$0.02$ & $0.15$ & $1.48$ & $16.46$ & $169.81$ \\
\midrule
Improvement factor &
$3.07$ & $3.79$ & $4.69$ & $5.20$ & $8.04$ \\
\bottomrule
\end{tabular}
\caption{Runtime (in seconds) for solving
$1000$ projection problems onto $B(\rho)$ with $\rho = 10$ and $m = n = d/2$,
see Eq.~\eqref{eq:proj-bipartite-simplex}.
The data is generated \iid from $\Norm(0,1)$.}
\label{tbl:time-ml}
\end{table}

We observe consistent improvement in runtime over the sorting based implementation,
and we use our algorithm to train \reftext{eq:smooth-svm-ml}
in further experiments.

\textbf{SDCA update:
\reftext{eq:lr-ml}.}
Finally, we discuss optimization of the multilabel cross-entropy loss
\reftext{eq:lr-ml}.
We show that the corresponding SDCA update step
is equivalent to a certain entropic projection problem,
which we propose to tackle using the $V(t)$ function introduced above.

\ifprop
\begin{proposition}\label{prop:lr-ml-update}
Let $L$ and $L^*$ in \eqref{eq:primal-dual} be respectively
the \reftext{eq:lr-ml} loss and its conjugate from
Proposition~\ref{prop:lr-ml-conj}.
The dual variables $a \bydef a_i$ corresponding to
the training pair $(x_i,Y_i)$ are updated as
$(a_y)_{y \in Y_i} = -\tfrac{1}{\lambda n} \big( p - \tfrac{1}{k} \big)$ and
$(a_j)_{j \in \bar{Y}_i} = -\tfrac{1}{\lambda n} \bar{p}$,
where
\begin{equation}\label{eq:ce_problem}
\begin{aligned}
(p,\bar{p}) = \argmin_{x \geq 0 , \, y \geq 0 } \;
& \tfrac{\alpha}{2} \normn{x - b}^2 + \inner{x, \log x} + \\[-.75em]
& \tfrac{\alpha}{2} \normn{y - \bar{b}}^2 + \inner{y, \log y} , \\[0.25em]
\st \; & \inner{\ones, x} + \inner{\ones, y} = 1 ,
\end{aligned}
\end{equation}
$k = \abs{Y_i}$,
$\alpha = \tfrac{\inner{x_i, x_i}}{\lambda n}$,
$b = \big( \tfrac{1}{\alpha} q_j + \tfrac{1}{k} \big)_{j \in Y_i}$,
$\bar{b} = \big( \tfrac{1}{\alpha} q_j \big)_{j \in \bar{Y}_i}$,
and $q = \tra{W} x_i - \inner{x_i,x_i} a_i$.

Moreover, the solution of \eqref{eq:ce_problem} is given by
\begin{align*}
p_j &= \tfrac{1}{\alpha} V(\alpha b_j - t) , \; \forall\, j, &
\bar{p}_j &= \tfrac{1}{\alpha} V(\alpha \bar{b}_j - t) , \; \forall\, j,
\end{align*}
where $t$ is computed from
\begin{align}\label{eq:lr-ml-update-eq}
\tsum_{j \in Y_i} V(q_j + \tfrac{\alpha}{k} - t)
+ \tsum_{j \in \bar{Y}_i} V(q_j - t) = \alpha .
\end{align}
\end{proposition}
\fi
\ifproof
\begin{proof}
Let $q \bydef \sum_{j \neq i} K_{ij} a_j = A K_i - K_{ii} a_i$
and $C \bydef \frac{1}{\lambda n}$, as before.
We need to solve
\begin{align*}
\max_{a \in \Rb^m} \;\;
& - C L^* \Big( Y_i, -\frac{1}{C} a \Big)
- \tfrac{1}{2} \big( K_{ii} \norms{a} + 2 \inner{q, a} \big) .
\end{align*}
Let $x$ and $y$ be defined as
\begin{align*}
&\begin{cases}
x = \big( -\tfrac{1}{C} a_j + \tfrac{1}{k} \big)_{j \in Y_i} \\
y = \big( -\tfrac{1}{C} a_j \big)_{j \in \bar{Y}_i} ,
\end{cases}
&\Longrightarrow &
&\begin{cases}
a_j = -C \big( x_j - \tfrac{1}{k} \big) , \\
a_j = -C y_j .
\end{cases}
\end{align*}
We have that
\begin{align*}
K_{ii} \norms{a} + 2 \inner{q, a} &=
K_{ii} C^2 \big( \normb{x - \tfrac{1}{k}\ones}^2 + \norms{y} \big) \\
&- 2C \big( \innerb{q_{Y_i}, x - \tfrac{1}{k}\ones}
+ \innern{q_{\bar{Y}_i}, y} \big) .
\end{align*}
Ignoring the constant terms and switching the sign, we obtain
\begin{align*}
\min_{x \geq 0 , \, y \geq 0} \;
& \inner{x, \log x} + \tfrac{1}{2} K_{ii} C \norms{x}
- K_{ii} C \tfrac{1}{k} \inner{\ones, x} - \innerb{q_{Y_i}, x} \\[-.5em]
& \inner{y, \log y} + \tfrac{1}{2} K_{ii} C \norms{y}
- \innerb{q_{\bar{Y}_i}, y} \\
\st \; & \inner{\ones, x} + \inner{\ones, y} = 1 
\end{align*}
Let $\alpha \bydef K_{ii} C$ and define
\begin{align*}
b_j &= \tfrac{1}{\alpha} q_j + \tfrac{1}{k}, \; j \in Y_i, &
\bar{b}_j &= \tfrac{1}{\alpha} q_j , \; j \in \bar{Y}_i.
\end{align*}
The final proximal problem for the update step is given as
\begin{equation*}
\begin{aligned}
\min_{x \geq 0 , \, y \geq 0} \;
& \inner{x, \log x} + \tfrac{\alpha}{2} \norms{x - b}
+ \inner{y, \log y} + \tfrac{\alpha}{2} \norms{y - \bar{b}} \\[-.5em]
& \inner{\ones, x} + \inner{\ones, y} = 1 .
\end{aligned}
\end{equation*}

Next, we discuss how to solve \eqref{eq:ce_problem}.
The Lagrangian for this problem is given by
\begin{gather*}
\Lc(x,y,\lambda,\mu,\nu) = \inner{x, \log x} + \tfrac{\alpha}{2} \norms{x - b}
+ \inner{y, \log y} \\ + \tfrac{\alpha}{2} \norms{y - \bar{b}}
+ \lambda (\inner{\ones, x} + \inner{\ones, y} - 1)
- \inner{\mu, x} - \inner{\nu, y} .
\end{gather*}
Setting the partial derivatives to zero, we obtain
\begin{align*}
\log x_j + \alpha x_j &= \alpha b_j - \lambda - 1 + \mu_j , \\
\log y_j + \alpha y_j &= \alpha \bar{b}_j - \lambda - 1 + \nu_j .
\end{align*}
We $x_j > 0$ and $y_j > 0$,
which implies $\mu_j = 0$ and $\nu_j = 0$.
\begin{align*}
\log (\alpha x_j) + \alpha x_j &= \alpha b_j - \lambda - 1 + \log \alpha , \\
\log (\alpha y_j) + \alpha y_j &= \alpha \bar{b}_j - \lambda - 1 + \log \alpha .
\end{align*}
Let $t \bydef \lambda + 1 - \log \alpha$, we have
\begin{align*}
\alpha x_j &= W(\exp(\alpha b_j - t)) = V(\alpha b_j - t), \\
\alpha y_j &= W(\exp(\alpha \bar{b}_j - t)) = V(\alpha \bar{b}_j - t),
\end{align*}
where $W$ is the Lambert $W$ function.
Let
\begin{align*}
g(t) &= \sum_{j \in Y_i} V(b_j + \tfrac{\alpha}{k} - t)
+ \sum_{j \in \bar{Y}_i} V(b_j - t) - \alpha ,
\end{align*}
then the optimal $t^*$ is the root of $g(t) = 0$,
which corresponds to the constraint
$\inner{\ones, x} + \inner{\ones, y} = 1$.
\end{proof}
\fi

We use a $4$-th order Householder's method \cite{householder1970numerical}
to solve \eqref{eq:lr-ml-update-eq},
similar to the \reftext{eq:topk-entropy} loss above.
Solving the nonlinear equation in $t$ is the main
computational challenge when updating the dual variables.
However, as this procedure does not require iteration over
the index partitions, it is generally faster than optimization
of the \reftext{eq:topk-entropy} loss.

\section{Experiments}%
\label{sec:experiments}
This section provides a broad array of experiments
on $24$ different datasets comparing multiclass and multilabel
performance of the $13$ loss functions from \S~\ref{sec:loss-functions}.
We look at different aspects of empirical evaluation:
performance on synthetic and real data,
use of handcrafted features and the features extracted from a ConvNet,
targeting a specific performance measure
and being generally competitive over a range of metrics.

In \S~\ref{sec:synthetic},
we show on synthetic data that the \reftext{eq:truncated-topk-entropy} loss
targeting specifically the top-$2$ error
outperforms all competing methods by a large margin.
In \S~\ref{sec:multiclass-experiments},
we focus on evaluating top-$k$ performance of multiclass methods
on $11$ real-world benchmark datasets including ImageNet and Places.
In \S~\ref{sec:multilabel-experiments},
we cover multilabel classification in two groups of experiments:
(i) a comparative study following \cite{madjarov2012extensive}
on $10$ popular multilabel datasets;
(ii) image classification on Pascal VOC and MS COCO
in a novel setting contrasting multiclass, top-$k$, and multilabel methods.

\begin{figure}[ht]\small\centering%
\begin{subfigure}[t]{0.49\columnwidth}\centering
\includegraphics[width=\columnwidth]{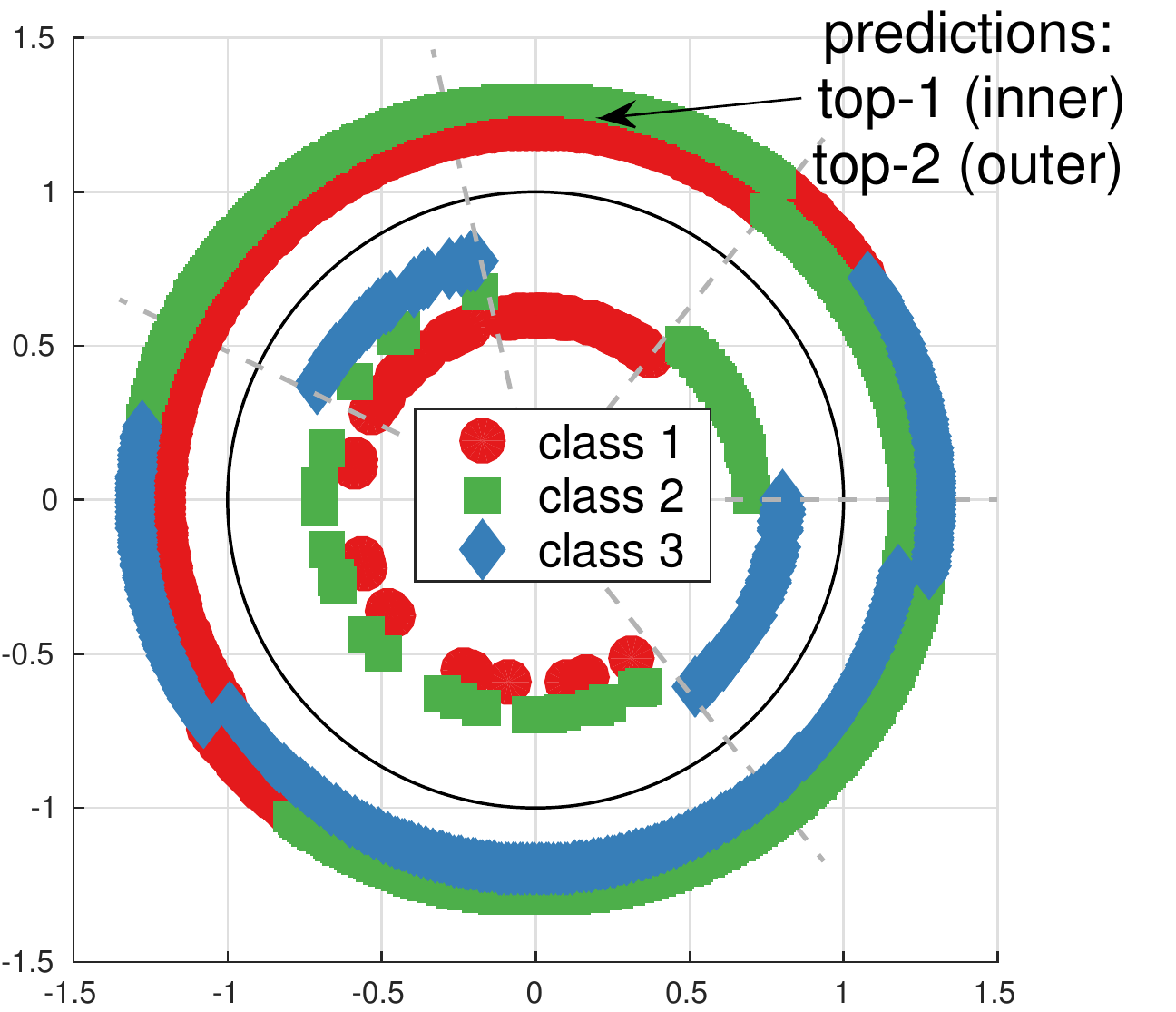}%
\caption{\centering \SvmTopKag{1}{1} test accuracy\newline
(top-$1$ / top-$2$): $65.7\%$ / $81.3\%$}\label{fig:toy:a}
\end{subfigure}
\begin{subfigure}[t]{0.49\columnwidth}\centering
\includegraphics[width=\columnwidth]{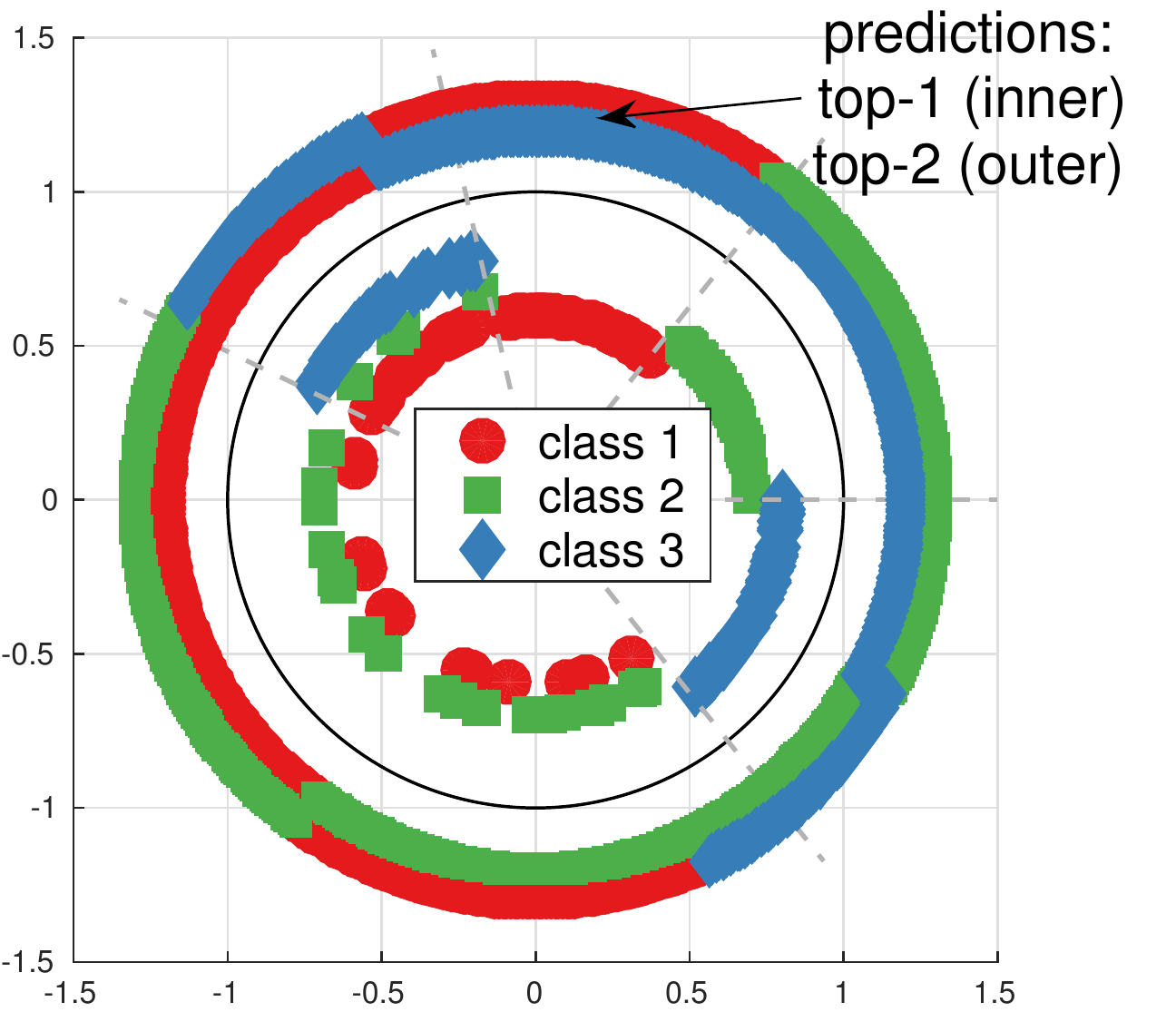}%
\caption{\centering  \LrTopKn{2}\ test accuracy\newline
(top-$1$ / top-$2$): $29.4\%$, $96.1\%$}\label{fig:toy:b}
\end{subfigure}
\caption{%
Synthetic data in $\Rb^2$ (color markers inside of the black circle)
and visualization of top-$1$ and top-$2$ predictions
(resp.\ outside of the circle).
{\bfseries (a)} \SvmTopKag{1}{1} optimizes
the top-$1$ error which increases its top-$2$ error.
{\bfseries (b)} \LrTopKn{2}\
ignores top-$1$ mistakes and optimizes directly the top-$2$ error.
}\label{fig:toy}
\end{figure}

\subsection{Synthetic Example}\label{sec:synthetic}
In this section, we demonstrate in a synthetic experiment 
that our proposed top-$2$ losses outperform
the top-$1$ losses when the aim is optimal top-$2$ performance.
The dataset with three classes is shown in the inner circle of
Figure~\ref{fig:toy}.

\textbf{Sampling scheme.}
First, we generate samples in $[0,7] \subset \Rb$ which is subdivided
into $5$ segments. All segments have unit length,
except for the $4$-th segment which has length $3$.
We sample in each of the $5$ segments
according to the following distribution:
$(0, 1, .4, .3, 0)$ for class $1$;
$(1, 0, .1, .7, 0)$ for class $2$;
$(0, 0, .5, 0, 1)$ for class $3$.
Finally, the data is rescaled to $[0,1]$ and mapped onto the unit circle.

Samples of different classes are plotted next to each other for better visibility
as there is significant class overlap. We visualize
top-$1/2$ predictions with two colored circles outside of the black circle.
We sample $200$/$200$/$200\K$ points for training/validation/test
and tune $C = 1/(\lambda n)$ in the range $2^{-18}$ to $2^{18}$.
Results are shown in Table~\ref{tbl:toy}.

\begin{table}[ht]\scriptsize\centering\setlength{\tabcolsep}{.4em}\intabletrue
\begin{tabular}{l|cc||l|cc}
\multicolumn{6}{c}{\small\textbf{Circle} (synthetic)} \\\toprule
Method & Top-1 & Top-2 & Method & Top-1 & Top-2\\
\midrule
\midrule
\SvmOva & $54.3$ & $85.8$ &
\SvmTopKg{1}{1} & $\mathbf{65.7}$ & $83.9$ \\
\LrOva & $54.7$ & $81.7$ &
\SvmTopKg{2}{0/1} & $54.4$ / $54.5$ & $87.1$ / $87.0$ \\
\SvmMulti & $58.9$ & $89.3$ &
\LrTopK{2} & $54.6$ & $87.6$ \\
\LrMulti & $54.7$ & $81.7$ &
\LrTopKn{2} & $58.4$ & $\mathbf{96.1}$ \\
\bottomrule
\end{tabular}
\caption{Top-$k$ accuracy (\%) on synthetic data.
{\bfseries Left:} Baseline methods.
{\bfseries Right:} Top-$k$ SVM (nonsmooth / smooth)
and top-$k$ softmax losses (convex and nonconvex).
}
\label{tbl:toy}\intablefalse
\end{table}

In each column, we provide the results for the model
(as determined by the hyperparameter $C$) that optimizes
the corresponding top-$k$ accuracy.
First, we note that all top-$1$ baselines perform similar in top-$1$
performance, except for \SvmMulti\ and \SvmTopKg{1}{1}
which show better results.
Next, we see that our top-$2$ losses improve
the top-$2$ accuracy and the improvement is most significant
for the nonconvex \LrTopKn{2} loss,
which is close to the optimal solution for this dataset.
This is because \LrTopKn{2} provides a tight bound on the top-$2$ error
and ignores the top-$1$ errors in the loss.
Unfortunately, similar significant improvements are not observed
on the real-world datasets that we tried.
This might be due to the high dimension of the feature spaces,
which yields well separable problems.

\subsection{Multiclass Experiments}\label{sec:multiclass-experiments}

The goal of this section is to provide an extensive empirical evaluation
of the loss functions from \S~\ref{sec:multiclass-methods}
in terms of top-$k$ performance.
To that end, we compare multiclass and top-$k$ methods
on $11$ datasets ranging in size
($500$ to $2.4\M$ training examples, $10$ to $1000$ classes),
problem domain (vision, non-vision),
and granularity (scene, object, and fine-grained classification).
The detailed statistics is given in Table~\ref{tbl:stats}.

\begin{table}[H]\scriptsize\centering\setlength{\tabcolsep}{.45em}
\begin{tabular}{l|ccc||l|ccc}
\toprule
Dataset & $m$ & $n$ & $d$ & Dataset & $m$ & $n$ & $d$ \\
\midrule
\midrule
ALOI \cite{rocha2014multiclass} &
$1\K$ & $54\K$ & $128$ &
Indoor 67 \cite{quattoni2009recognizing} &
$67$ & $5354$ & $4\K$ \\
Caltech 101 Sil \cite{swersky2012probabilistic} &
$101$ & $4100$ & $784$ &
Letter \cite{hsu2002comparison} &
$26$ & $10.5\K$ & $16$ \\
CUB \cite{wah2011caltech} &
$202$ & $5994$ & $4\K$ &
News 20 \cite{lang1995newsweeder} &
$20$ & $15.9\K$ & $16\K$ \\
Flowers \cite{nilsback2008automated} &
$102$ & $2040$ & $4\K$ &
Places 205 \cite{zhou2014learning} &
$205$ & $2.4\M$ & $4\K$ \\
FMD \cite{sharan2009material} &
$10$ & $500$ & $4\K$ &
SUN 397 \cite{xiao2010sun} &
$397$ & $19.9\K$ & $4\K$ \\
ImageNet 2012 \cite{ILSVRCarxiv14} &
$1\K$ & $1.3\M$ & $4\K$ &
& & & \\
\bottomrule
\end{tabular}
\caption{Statistics of multiclass classification benchmarks
($m$: \# classes, $n$: \# training examples, $d$: \# features).
}
\label{tbl:stats}
\end{table}

Please refer to Table~\ref{tbl:summary} for an overview of the methods
and our naming convention.
Further comparison with other established ranking based losses
can be found in \cite{lapin2015topk}.

\textbf{Solvers.}
We use LibLinear \cite{REF08a} for the one-vs-all baselines
\SvmOva\ and \LrOva; and our code from \cite{lapin2015topk}
for \SvmTopK{k}.
We extended the latter to support the smooth \SvmTopKg{k}{\gamma}
and the \LrTopK{k} losses.
The multiclass baselines \SvmMulti\ and \LrMulti\
correspond respectively to \SvmTopK{1} and \LrTopK{1}.
For the nonconvex \LrTopKn{k},
we use the \LrMulti\ solution as an initial point
and perform gradient descent with line search \cite{nocedal2006numerical}.
We cross-validate hyper-parameters in the range
$10^{-5}$ to $10^3$, extending it when the optimal value is
at the boundary.

\textbf{Features.}
For ALOI, Letter, and News20 datasets,
we use the features provided by the LibSVM \cite{chang2011libsvm} datasets.
For ALOI, we randomly split the data into equally sized 
training and test sets preserving class distributions.
The Letter dataset comes with a separate validation set,
which we use for model selection only.
For News20, we use PCA to reduce dimensionality
of sparse features
from $62060$ to $15478$
preserving all non-singular PCA components\footnote{
Our SDCA-based solvers are designed for dense inputs.}.

For Caltech101 Silhouettes, we use the features and the
train/val/test splits provided by \cite{swersky2012probabilistic}.

For CUB, Flowers, FMD, and ImageNet 2012,
we use MatConvNet \cite{vedaldi15matconvnet} to extract
the outputs of the last fully connected layer of
the VGGNet-16 model \cite{simonyan2014very}.

For Indoor 67, SUN 397, and Places 205,
we perform the same feature extraction,
but use the VGGNet-16 model of \cite{wang2015places}
which was pre-trained on Places 205.

\begin{table*}[tp]\scriptsize\centering\setlength{\tabcolsep}{.5em}\intabletrue
\begin{tabular}{l|cccc|cccc|cccc|cccc}
\multicolumn{17}{c}{\small\textbf{
\qquad\qquad\qquad\qquad\quad
ALOI \qquad\qquad\qquad\qquad\;
Letter \qquad\qquad\qquad\qquad
News 20 \qquad\qquad
Caltech 101 Silhouettes
}} \\
\toprule
Reference: &
\multicolumn{4}{c|}{Top-$1$: $93 \pm 1.2$ \cite{rocha2014multiclass}} &
\multicolumn{4}{c|}{Top-$1$: $97.98$ \cite{hsu2002comparison} (RBF)} &
\multicolumn{4}{c|}{Top-$1$: $86.9$ \cite{rennie2001improving}} &
$62.1$ & $79.6$ & $83.4$ & \cite{swersky2012probabilistic} \\
\midrule
\midrule
Method &
Top-$1$ & Top-$3$ & Top-$5$ & Top-$10$ &
Top-$1$ & Top-$3$ & Top-$5$ & Top-$10$ &
Top-$1$ & Top-$3$ & Top-$5$ & Top-$10$ &
Top-$1$ & Top-$3$ & Top-$5$ & Top-$10$ \\
\midrule
\midrule
\SvmOva &
$82.4$ & $89.5$ & $91.5$ & $93.7$ &
$63.0$ & $82.0$ & $88.1$ & $94.6$ &
$84.3$ & $95.4$ & $97.9$ & $\mathbf{99.5}$ &
$61.8$ & $76.5$ & $80.8$ & $86.6$ \\
\LrOva &
$86.1$ & $93.0$ & $94.8$ & $96.6$ &
$68.1$ & $86.1$ & $90.6$ & $96.2$ &
$84.9$ & $96.3$ & $97.8$ & $99.3$ &
$63.2$ & $80.4$ & $84.4$ & $89.4$ \\
\midrule
\SvmMulti &
$90.0$ & $95.1$ & $96.7$ & $98.1$ &
$76.5$ & $89.2$ & $93.1$ & $97.7$ &
$85.4$ & $94.9$ & $97.2$ & $99.1$ &
$62.8$ & $77.8$ & $82.0$ & $86.9$ \\
\LrMulti &
$89.8$ & $95.7$ & $97.1$ & $98.4$ &
$75.3$ & $90.3$ & $94.3$ & $98.0$ &
$84.5$ & $96.4$ & $98.1$ & $\mathbf{99.5}$ &
$63.2$ & $\mathbf{81.2}$ & $85.1$ & $89.7$ \\
\midrule
\SvmTopK{3} &
$89.2$ & $95.5$ & $97.2$ & $98.4$ &
$74.0$ & $91.0$ & $94.4$ & $97.8$ &
$85.1$ & $96.6$ & $98.2$ & $99.3$ &
$63.4$ & $79.7$ & $83.6$ & $88.3$ \\
\SvmTopK{5} & $87.3$ & $95.6$ & $97.4$ & $98.6$ &
$70.8$ & $\mathbf{91.5}$ & $95.1$ & $98.4$ &
$84.3$ & $96.7$ & $98.4$ & $99.3$ &
$63.3$ & $80.0$ & $84.3$ & $88.7$ \\
\SvmTopK{10} &
$85.0$ & $95.5$ & $97.3$ & $\mathbf{98.7}$ &
$61.6$ & $88.9$ & $96.0$ & $99.6$ &
$82.7$ & $96.5$ & $98.4$ & $99.3$ &
$63.0$ & $80.5$ & $84.6$ & $89.1$ \\
\midrule
\SvmTopKg{1}{1} &
$\mathbf{90.6}$ & $95.5$ & $96.7$ & $98.2$ &
$\mathbf{76.8}$ & $89.9$ & $93.6$ & $97.6$ &
$\mathbf{85.6}$ & $96.3$ & $98.0$ & $99.3$ &
$\mathbf{63.9}$ & $80.3$ & $84.0$ & $89.0$ \\
\SvmTopKg{3}{1} &
$89.6$ & $95.7$ & $97.3$ & $98.4$ &
$74.1$ & $90.9$ & $94.5$ & $97.9$ &
$85.1$ & $96.6$ & $98.4$ & $99.4$ &
$63.3$ & $80.1$ & $84.0$ & $89.2$ \\
\SvmTopKg{5}{1} &
$87.6$ & $95.7$ & $\mathbf{97.5}$ & $98.6$ &
$70.8$ & $\mathbf{91.5}$ & $95.2$ & $98.6$ &
$84.5$ & $96.7$ & $98.4$ & $99.4$ &
$63.3$ & $80.5$ & $84.5$ & $89.1$ \\
\SvmTopKg{10}{1} &
$85.2$ & $95.6$ & $97.4$ & $\mathbf{98.7}$ &
$61.7$ & $89.1$ & $95.9$ & $\mathbf{99.7}$ &
$82.9$ & $96.5$ & $98.4$ & $\mathbf{99.5}$ &
$63.1$ & $80.5$ & $84.8$ & $89.1$ \\
\midrule
\LrTopK{3} &
$89.0$ & $95.8$ & $97.2$ & $98.4$ &
$73.0$ & $90.8$ & $94.9$ & $98.5$ &
$84.7$ & $96.6$ & $98.3$ & $99.4$ &
$63.3$ & $81.1$ & $85.0$ & $89.9$ \\
\LrTopK{5} &
$87.9$ & $95.8$ & $97.2$ & $98.4$ &
$69.7$ & $90.9$ & $95.1$ & $98.8$ &
$84.3$ & $\mathbf{96.8}$ & $\mathbf{98.6}$ & $99.4$ &
$63.2$ & $80.9$ & $85.2$ & $89.9$ \\
\LrTopK{10} &
$86.0$ & $95.6$ & $97.3$ & $98.5$ &
$65.0$ & $89.7$ & $\mathbf{96.2}$ & $99.6$ &
$82.7$ & $96.4$ & $98.5$ & $99.4$ &
$62.5$ & $80.8$ & $\mathbf{85.4}$ & $90.1$ \\
\midrule
\LrTopKn{3} &
$89.3$ & $\mathbf{95.9}$ & $97.3$ & $98.5$ &
$63.6$ & $91.1$ & $95.6$ & $98.8$ &
$83.4$ & $96.4$ & $98.3$ & $99.4$ &
$60.7$ & $81.1$ & $85.2$ & $\mathbf{90.2}$ \\
\LrTopKn{5} &
$87.9$ & $95.7$ & $97.3$ & $98.6$ &
$50.3$ & $87.7$ & $96.1$ & $99.4$ &
$83.2$ & $96.0$ & $98.2$ & $99.4$ &
$58.3$ & $79.8$ & $85.2$ & $\mathbf{90.2}$ \\
\LrTopKn{10} &
$85.2$ & $94.8$ & $97.1$ & $98.5$ &
$46.5$ & $80.9$ & $93.7$ & $99.6$ &
$82.9$ & $95.7$ & $97.9$ & $99.4$ &
$51.9$ & $78.4$ & $84.6$ & $\mathbf{90.2}$ \\
\bottomrule
\end{tabular}
\\[.4em]\setlength{\tabcolsep}{.55em}
\begin{tabular}{l|cccc|cccc|cccc|ccc}
\multicolumn{16}{c}{\small\textbf{
\qquad\qquad\qquad\quad
Indoor 67 \qquad\qquad\qquad\quad
CUB \qquad\qquad\qquad\qquad\;\;
Flowers \qquad\qquad\qquad\quad
FMD
}} \\
\toprule
Reference (Top-$1$): &
\multicolumn{4}{c|}{$82.0$ \cite{wang2015places}} &
\multicolumn{4}{c|}{$62.8$ \cite{cimpoi15deep} / $76.37$ \cite{zhang2014part}} &
\multicolumn{4}{c|}{$86.8$ \cite{razavian2014cnn}} &
\multicolumn{3}{c}{$77.4$ \cite{cimpoi15deep} / $82.4$ \cite{cimpoi15deep}} \\
\midrule
\midrule
Method &
Top-$1$ & Top-$3$ & Top-$5$ & Top-$10$ &
Top-$1$ & Top-$3$ & Top-$5$ & Top-$10$ &
Top-$1$ & Top-$3$ & Top-$5$ & Top-$10$ &
Top-$1$ & Top-$3$ & Top-$5$ \\
\midrule
\midrule
\SvmOva &
$81.9$ & $94.3$ & $96.5$ & $98.0$ &
$60.6$ & $77.1$ & $83.4$ & $89.9$ &
$82.0$ & $91.7$ & $94.3$ & $96.8$ &
$77.4$ & $92.4$ & $96.4$ \\
\LrOva &
$82.0$ & $94.9$ & $97.2$ & $98.7$ &
$62.3$ & $80.5$ & $87.4$ & $93.5$ &
$82.6$ & $92.2$ & $94.8$ & $97.6$ &
$79.6$ & $94.2$ & $\mathbf{98.2}$ \\
\midrule
\SvmMulti &
$82.5$ & $\mathbf{95.4}$ & $97.3$ & $99.1$ &
$61.0$ & $79.2$ & $85.7$ & $92.3$ &
$82.5$ & $92.2$ & $94.8$ & $96.4$ &
$77.6$ & $93.8$ & $97.2$ \\
\LrMulti &
$82.4$ & $95.2$ & $\mathbf{98.0}$ & $99.1$ &
$62.3$ & $81.7$ & $87.9$ & $\mathbf{93.9}$ &
$82.9$ & $92.4$ & $95.1$ & $97.8$ &
$79.0$ & $94.6$ & $97.8$ \\
\midrule
\SvmTopK{3} &
$81.6$ & $95.1$ & $97.7$ & $99.0$ &
$61.3$ & $80.4$ & $86.3$ & $92.5$ &
$81.9$ & $92.2$ & $95.0$ & $96.1$ &
$78.8$ & $94.6$ & $97.8$ \\
\SvmTopK{5} &
$79.9$ & $95.0$ & $97.7$ & $99.0$ &
$60.9$ & $81.2$ & $87.2$ & $92.9$ &
$81.7$ & $92.4$ & $95.1$ & $97.8$ &
$78.4$ & $94.4$ & $97.6$ \\
\SvmTopK{10} &
$78.4$ & $95.1$ & $97.4$ & $99.0$ &
$59.6$ & $81.3$ & $87.7$ & $93.4$ &
$80.5$ & $91.9$ & $95.1$ & $97.7$ \\
\midrule
\SvmTopKg{1}{1} &
$\mathbf{82.6}$ & $95.2$ & $97.6$ & $99.0$ &
$61.9$ & $80.2$ & $86.9$ & $93.1$ &
$\mathbf{83.0}$ & $92.4$ & $95.1$ & $97.6$ &
$78.6$ & $93.8$ & $98.0$ \\
\SvmTopKg{3}{1} &
$81.6$ & $95.1$ & $97.8$ & $99.0$ &
$61.9$ & $81.1$ & $86.6$ & $93.2$ &
$82.5$ & $92.3$ & $95.2$ & $97.7$ &
$79.0$ & $94.4$ & $98.0$ \\
\SvmTopKg{5}{1} &
$80.4$ & $95.1$ & $97.8$ & $99.1$ &
$61.3$ & $81.3$ & $87.4$ & $92.9$ &
$82.0$ & $\mathbf{92.5}$ & $95.1$ & $97.8$ &
$79.4$ & $94.4$ & $97.6$ \\
\SvmTopKg{10}{1} &
$78.3$ & $95.1$ & $97.5$ & $99.0$ &
$59.8$ & $81.4$ & $87.8$ & $93.4$ &
$80.6$ & $91.9$ & $95.1$ & $97.7$ \\
\midrule
\LrTopK{3} &
$81.4$ & $\mathbf{95.4}$ & $97.6$ & $\mathbf{99.2}$ &
$\mathbf{62.5}$ & $81.8$ & $87.9$ & $\mathbf{93.9}$ &
$82.5$ & $92.0$ & $\mathbf{95.3}$ & $97.8$ &
$\mathbf{79.8}$ & $94.8$ & $98.0$ \\
\LrTopK{5} &
$80.3$ & $95.0$ & $97.7$ & $99.0$ &
$62.0$ & $\mathbf{81.9}$ & $88.1$ & $93.8$ &
$82.1$ & $92.2$ & $95.1$ & $\mathbf{97.9}$ &
$79.4$ & $94.4$ & $98.0$ \\
\LrTopK{10} &
$79.2$ & $95.1$ & $97.6$ & $99.0$ &
$61.2$ & $81.6$ & $\mathbf{88.2}$ & $93.8$ &
$80.9$ & $92.1$ & $95.0$ & $97.7$ \\
\midrule
\LrTopKn{3} &
$79.8$ & $95.0$ & $97.5$ & $99.1$ &
$62.0$ & $81.4$ & $87.6$ & $93.4$ &
$82.1$ & $92.2$ & $95.2$ & $97.6$ &
$78.4$ & $\mathbf{95.4}$ & $\mathbf{98.2}$ \\
\LrTopKn{5} &
$76.4$ & $94.3$ & $97.3$ & $99.0$ &
$61.4$ & $81.2$ & $87.7$ & $93.7$ &
$81.4$ & $92.0$ & $95.0$ & $97.7$ &
$77.2$ & $94.0$ & $97.8$ \\
\LrTopKn{10} &
$72.6$ & $92.8$ & $97.1$ & $98.9$ &
$59.7$ & $80.7$ & $87.2$ & $93.4$ &
$77.9$ & $91.1$ & $94.3$ & $97.3$ \\
\bottomrule
\end{tabular}
\\[.5em]\setlength{\tabcolsep}{.7em}
\begin{tabular}{l|cccc|cccc|cccc}
\multicolumn{13}{c}{\small{
\qquad\qquad\qquad\qquad\qquad
\textbf{SUN 397} (10 splits) \qquad\qquad\qquad\qquad\quad
\textbf{Places 205} (val) \qquad\qquad
\textbf{ImageNet 2012} (val)
}} \\
\toprule
Reference: &
\multicolumn{4}{c|}{Top-$1$: $66.9$ \cite{wang2015places}} &
$60.6$ & & $88.5$ & \cite{wang2015places} &
$76.3$ & & $93.2$ & \cite{simonyan2014very} \\
\midrule
\midrule
Method &
Top-$1$ & Top-$3$ & Top-$5$ & Top-$10$ &
Top-$1$ & Top-$3$ & Top-$5$ & Top-$10$ &
Top-$1$ & Top-$3$ & Top-$5$ & Top-$10$ \\
\midrule
\midrule
\SvmMulti &
$65.8 \pm 0.1$ & $85.1 \pm 0.2$ & $90.8 \pm 0.1$ & $95.3 \pm 0.1$ &
$58.4$ & $78.7$ & $84.7$ & $89.9$ &
$68.3$ & $82.9$ & $87.0$ & $91.1$ \\
\LrMulti &
$\mathbf{67.5 \pm 0.1}$ & $\mathbf{87.7 \pm 0.2}$ & $\mathbf{92.9 \pm 0.1}$ & $\mathbf{96.8 \pm 0.1}$ &
$59.0$ & $\mathbf{80.6}$ & $\mathbf{87.6}$ & $\mathbf{94.3}$ &
$67.2$ & $83.2$ & $87.7$ & $92.2$ \\
\midrule
\SvmTopK{3} &
$66.5 \pm 0.2$ & $86.5 \pm 0.1$ & $91.8 \pm 0.1$ & $95.9 \pm 0.1$ &
$58.6$ & $80.3$ & $87.3$ & $93.3$ &
$68.2$ & $84.0$ & $88.1$ & $92.1$ \\
\SvmTopK{5} &
$66.3 \pm 0.2$ & $87.0 \pm 0.2$ & $92.2 \pm 0.2$ & $96.3 \pm 0.1$ &
$58.4$ & $80.5$ & $87.4$ & $94.0$ &
$67.8$ & $\mathbf{84.1}$ & $88.2$ & $92.4$ \\
\SvmTopK{10} &
$64.8 \pm 0.3$ & $87.2 \pm 0.2$ & $92.6 \pm 0.1$ & $96.6 \pm 0.1$ &
$58.0$ & $80.4$ & $87.4$ & $\mathbf{94.3}$ &
$67.0$ & $83.8$ & $88.3$ & $\mathbf{92.6}$ \\
\midrule
\SvmTopKg{1}{1} &
$67.4 \pm 0.2$ & $86.8 \pm 0.1$ & $92.0 \pm 0.1$ & $96.1 \pm 0.1$ &
$\mathbf{59.2}$ & $80.5$ & $87.3$ & $93.8$ &
$\mathbf{68.7}$ & $83.9$ & $88.0$ & $92.1$ \\
\SvmTopKg{3}{1} &
$67.0 \pm 0.2$ & $87.0 \pm 0.1$ & $92.2 \pm 0.1$ & $96.2 \pm 0.0$ &
$58.9$ & $80.5$ & $\mathbf{87.6}$ & $93.9$ &
$68.2$ & $\mathbf{84.1}$ & $88.2$ & $92.3$ \\
\SvmTopKg{5}{1} &
$66.5 \pm 0.2$ & $87.2 \pm 0.1$ & $92.4 \pm 0.2$ & $96.3 \pm 0.0$ &
$58.5$ & $80.5$ & $87.5$ & $94.1$ &
$67.9$ & $\mathbf{84.1}$ & $\mathbf{88.4}$ & $92.5$ \\
\SvmTopKg{10}{1} &
$64.9 \pm 0.3$ & $87.3 \pm 0.2$ & $92.6 \pm 0.2$ & $96.6 \pm 0.1$ &
$58.0$ & $80.4$ & $87.5$ & $\mathbf{94.3}$ &
$67.1$ & $83.8$ & $88.3$ & $\mathbf{92.6}$ \\
\midrule
\LrTopK{3} &
$67.2 \pm 0.2$ & $\mathbf{87.7 \pm 0.2}$ & $\mathbf{92.9 \pm 0.1}$ & $\mathbf{96.8 \pm 0.1}$ &
$58.7$ & $\mathbf{80.6}$ & $\mathbf{87.6}$ & $94.2$ &
$66.8$ & $83.1$ & $87.8$ & $92.2$ \\
\LrTopK{5} &
$66.6 \pm 0.3$ & $\mathbf{87.7 \pm 0.2}$ & $\mathbf{92.9 \pm 0.1}$ & $\mathbf{96.8 \pm 0.1}$ &
$58.1$ & $80.4$ & $87.4$ & $94.2$ &
$66.5$ & $83.0$ & $87.7$ & $92.2$ \\
\LrTopK{10} &
$65.2 \pm 0.3$ & $87.4 \pm 0.1$ & $92.8 \pm 0.1$ & $\mathbf{96.8 \pm 0.1}$ &
$57.0$ & $80.0$ & $87.2$ & $94.1$ &
$65.8$ & $82.8$ & $87.6$ & $92.1$ \\
\bottomrule
\end{tabular}
\caption{Top-$k$ accuracy (\%) on various datasets.
The first line is a reference performance on each dataset
and reports top-$1$ accuracy except when the numbers are aligned with Top-$k$.
We compare the one-vs-all and multiclass baselines
with the \SvmTopKa{k} \cite{lapin2015topk} as well as the proposed
smooth \SvmTopKag{k}{\gamma}, \LrTopK{k}, and the nonconvex \LrTopKn{k}.
}
\label{tbl:all}\intablefalse
\end{table*}

\textbf{Discussion.}
The results are given in Table~\ref{tbl:all},
and we can make several interesting observations.
First, while the OVA schemes perform quite similar to 
the multiclass approaches
(OVA logistic regression \V softmax, OVA SVM \V multiclass SVM),
which confirms earlier observations in \cite{akata2014good,rifkin2004defense},
the OVA schemes performed worse on ALOI and Letter.
Thus, we generally recommend the multiclass losses
instead of the OVA schemes.

Comparing the softmax loss and multiclass SVM,
we see that there is no clear winner in top-$1$ performance,
but softmax consistently outperforms multiclass SVM
in top-$k$ performance for $k>1$.
This might be due
to the strong property of softmax being top-$k$ calibrated for all $k$.
Note that this trend is uniform across all datasets,
in particular, also for the ones where the features
are not coming from a ConvNet.
Both the smooth top-$k$ SVM and the top-$k$ entropy losses 
perform slightly better than softmax if one
compares specific top-$k$ errors.
However, the good performance of the truncated top-$k$ entropy loss
on synthetic data did not transfer to the real world datasets.

\begin{table}[ht]\scriptsize\centering\setlength{\tabcolsep}{1em}\intabletrue
\begin{tabular}{l|cccc}
\multicolumn{5}{c}{\small\textbf{Places 205} (val)} \\
\toprule
Method &
Top-$1$ & Top-$3$ & Top-$5$ & Top-$10$ \\
\midrule
\midrule
\LrMulti &
$59.97$ & $81.39$ & $88.17$ & $94.59$ \\
\midrule
\SvmTopKg{3}{1}\ (FT) &
$60.73$ & $82.09$ & $88.58$ & $94.56$ \\
\SvmTopKg{5}{1}\ (FT) &
$\mathbf{60.88}$ & $\mathbf{82.18}$ & $\mathbf{88.78}$ & $94.75$ \\
\midrule
\LrTopKn{3}\ (FT) &
$60.51$ & $81.86$ & $88.69$ & $94.78$ \\
\LrTopKn{5}\ (FT) &
$60.48$ & $81.66$ & $88.66$ & $94.80$ \\
\midrule
\LrMulti\ (FT) &
$60.73$ & $82.07$ & $88.71$ & $\mathbf{94.82}$ \\
\bottomrule
\end{tabular}

\vspace{1em}
\begin{tabular}{l|cccc}
\multicolumn{5}{c}{\small\textbf{ImageNet 2012} (val)} \\
\toprule
Method &
Top-$1$ & Top-$3$ & Top-$5$ & Top-$10$ \\
\midrule
\midrule
\LrMulti &
$68.60$ & $84.29$ & $88.66$ & $92.83$ \\
\midrule
\SvmTopKg{3}{1}\ (FT) &
$71.66$ & $86.63$ & $90.55$ & $94.17$ \\
\SvmTopKg{5}{1}\ (FT) &
$71.60$ & $86.67$ & $90.56$ & $94.23$ \\
\midrule
\LrTopKn{3}\ (FT) &
$71.41$ & $86.80$ & $90.77$ & $94.35$ \\
\LrTopKn{5}\ (FT) &
$71.20$ & $86.57$ & $90.75$ & $\mathbf{94.38}$ \\
\midrule
\LrMulti\ (FT) &
$\mathbf{72.11}$ & $\mathbf{87.08}$ & $\mathbf{90.88}$ & $\mathbf{94.38}$ \\
\bottomrule
\end{tabular}
\caption{Top-$k$ accuracy (\%), as reported by Caffe \cite{jia2014caffe},
on large scale datasets after fine-tuning (FT)
for approximately one epoch on Places and 3 epochs on ImageNet.
The first line (\LrMulti) is the reference performance w/o fine-tuning.
}
\label{tbl:finetune}\intablefalse
\end{table}

\textbf{Fine-tuning experiments.}
We also performed a number of fine-tuning experiments
where the original network was trained further for $1$-$3$ epochs
with the smooth top-$k$ hinge and
the truncated top-$k$ entropy losses\footnote{
Code: \url{https://github.com/mlapin/caffe/tree/topk}}.
The motivation was to see if the full end-to-end training would be more
beneficial compared to training just the classifier.
Results are reported in Table~\ref{tbl:finetune}.
We should note that the setting is now slightly different:
there is no feature extraction step with the MatConvNet
and there is a non-regularized bias term in Caffe \cite{jia2014caffe}.
We see that the top-$k$ specific losses are able to improve
the performance compared to the reference model,
and that, on Places 205,
the smooth \SvmTopKg{5}{1}\ loss achieves the best top-$1..5$ performance.
However, in this set of experiments,
we also observed similar improvements
when fine-tuning with the standard softmax loss,
which achieves the best performance on ImageNet 2012.
Further training beyond $3$ epochs did not change the results
significantly.

\begin{table*}[tp]\scriptsize\centering\setlength{\tabcolsep}{.275em}\intabletrue
\begin{tabular}{l|cc|cc|ccc||cc|cc|ccc||cc|cc|ccc}
\multicolumn{1}{c}{} &
\multicolumn{7}{c}{\small\textbf{bibtex}} &
\multicolumn{7}{c}{\small\textbf{bookmarks}} &
\multicolumn{7}{c}{\small\textbf{corel5k}}
\\\toprule
Method &
\RLoss & \HLoss & \Acc & \SAcc & \Fmicro & \Fmacro & \Finst &
\RLoss & \HLoss & \Acc & \SAcc & \Fmicro & \Fmacro & \Finst &
\RLoss & \HLoss & \Acc & \SAcc & \Fmicro & \Fmacro & \Finst \\
\midrule
\midrule
\RFPCT\ \cite{madjarov2012extensive} &
$0.093$ & $0.013$ & $16.6$ & $9.8$ & $23.0$ & $5.5$ & $21.2$ &
$0.104$ & $0.009$ & $20.4$ & $18.9$ & $23.6$ & $10.1$ & $21.3$ &
$0.117$ & $\mathbf{0.009}$ & $0.9$ & $0.0$ & $1.8$ & $0.4$ & $1.4$ \\
\HOMER\ \cite{madjarov2012extensive} &
$0.255$ & $0.014$ & $33.0$ & $16.5$ & $42.9$ & $26.6$ & $42.6$ &
- & - & - & - & - & - & - &
$0.352$ & $0.012$ & $17.9$ & $0.2$ & $27.5$ & $3.6$ & $28.0$ \\
\BR\ \cite{madjarov2012extensive} {\tiny(RBF)} &
$0.068$ & $\mathbf{0.012}$ & $34.8$ & $\mathbf{19.4}$ & $45.7$ & $30.7$ & $43.3$ &
- & - & - & - & - & - & - &
$0.117$ & $0.017$ & $3.0$ & $0.0$ & $5.9$ & $2.1$ & $4.7$ \\
\midrule
\LrML\ &
$\mathbf{0.053}$ & $0.013$ & $30.9$ & $14.2$ & $42.5$ & $35.0$ & $38.8$ &
$0.079$ & $0.009$ & $22.5$ & $16.5$ & $29.5$ & $21.7$ & $27.0$ &
$\mathbf{0.101}$ & $\mathbf{0.009}$ & $17.5$ & $0.0$ & $27.1$ & $6.4$ & $27.3$ \\
\SvmML\ &
$0.094$ & $0.013$ & $28.6$ & $13.2$ & $40.6$ & $31.5$ & $36.1$ &
$0.140$ & $0.009$ & $24.0$ & $19.8$ & $27.5$ & $18.4$ & $26.7$ &
$0.205$ & $\mathbf{0.009}$ & $9.9$ & $0.8$ & $18.5$ & $5.0$ & $17.5$ \\
\SvmMLg{\gamma} &
$0.073$ & $0.013$ & $31.4$ & $16.2$ & $43.5$ & $33.5$ & $39.3$ &
$0.091$ & $0.009$ & $28.0$ & $20.7$ & $34.0$ & $22.6$ & $32.4$ &
$0.174$ & $\mathbf{0.009}$ & $18.8$ & $1.0$ & $29.4$ & $5.9$ & $26.3$ \\
\midrule
\LrML\ {\tiny(RBF)} &
$0.054$ & $0.013$ & $33.8$ & $14.6$ & $45.4$ & $31.8$ & $42.0$ &
$\mathbf{0.072}$ & $0.009$ & $25.1$ & $19.7$ & $33.1$ & $24.6$ & $29.0$ &
$\mathbf{0.101}$ & $\mathbf{0.009}$ & $18.0$ & $1.0$ & $28.5$ & $6.0$ & $27.8$ \\
\SvmML\ {\tiny(RBF)} &
$0.067$ & $0.013$ & $36.2$ & $19.0$ & $\mathbf{46.5}$ & $37.1$ & $\mathbf{44.6}$ &
$0.103$ & $0.009$ & $30.3$ & $22.9$ & $35.8$ & $26.0$ & $34.9$ &
$0.107$ & $\mathbf{0.009}$ & $18.1$ & $\mathbf{1.8}$ & $28.8$ & $6.7$ & $27.2$ \\
\SvmMLg{\gamma} {\tiny(RBF)} &
$0.067$ & $\mathbf{0.012}$ & $\mathbf{36.6}$ & $18.4$ & $\mathbf{46.5}$ & $\mathbf{37.2}$ & $\mathbf{44.6}$ &
$0.079$ & $\mathbf{0.008}$ & $\mathbf{31.8}$ & $\mathbf{23.0}$ & $\mathbf{38.0}$ & $\mathbf{28.0}$ & $\mathbf{36.7}$ &
$0.105$ & $\mathbf{0.009}$ & $\mathbf{19.3}$ & $\mathbf{1.8}$ & $\mathbf{30.2}$ & $\mathbf{6.8}$ & $\mathbf{28.8}$ \\
\bottomrule

\multicolumn{22}{c}{}\\
\multicolumn{1}{c}{} &
\multicolumn{7}{c}{\small\textbf{delicious}} &
\multicolumn{7}{c}{\small\textbf{enron}} &
\multicolumn{7}{c}{\small\textbf{mediamill}}
\\\toprule
Method &
\RLoss & \HLoss & \Acc & \SAcc & \Fmicro & \Fmacro & \Finst &
\RLoss & \HLoss & \Acc & \SAcc & \Fmicro & \Fmacro & \Finst &
\RLoss & \HLoss & \Acc & \SAcc & \Fmicro & \Fmacro & \Finst \\
\midrule
\midrule
\RFPCT\ \cite{madjarov2012extensive} &
$0.106$ & $\mathbf{0.018}$ & $14.6$ & $0.7$ & $24.8$ & $8.3$ & $24.4$ &
$0.079$ & $0.046$ & $41.6$ & $13.1$ & $53.7$ & $12.2$ & $55.2$ &
$0.047$ & $\mathbf{0.029}$ & $44.1$ & $12.2$ & $56.3$ & $11.2$ & $\mathbf{58.9}$ \\
\HOMER\ \cite{madjarov2012extensive} &
$0.379$ & $0.022$ & $20.7$ & $0.1$ & $33.9$ & $10.3$ & $34.3$ &
$0.183$ & $0.051$ & $47.8$ & $14.5$ & $59.1$ & $16.7$ & $\mathbf{61.3}$ &
$0.177$ & $0.038$ & $41.3$ & $5.3$ & $55.3$ & $7.3$ & $57.9$ \\
\BR\ \cite{madjarov2012extensive} {\tiny(RBF)} &
$0.114$ & $\mathbf{0.018}$ & $13.6$ & $0.4$ & $23.4$ & $9.6$ & $23.0$ &
$0.084$ & $\mathbf{0.045}$ & $44.6$ & $14.9$ & $56.4$ & $14.3$ & $58.2$ &
$0.061$ & $0.032$ & $40.3$ & $8.0$ & $53.3$ & $5.6$ & $55.7$ \\
\midrule
\LrML\ &
$0.123$ & $0.019$ & $11.6$ & $0.3$ & $21.4$ & $10.9$ & $19.5$ &
$0.074$ & $0.055$ & $38.5$ & $7.8$ & $53.0$ & $21.9$ & $50.4$ &
$\mathbf{0.042}$ & $0.033$ & $41.2$ & $7.8$ & $54.8$ & $17.1$ & $54.4$ \\
\SvmML\ &
$0.184$ & $0.019$ & $6.9$ & $0.2$ & $11.1$ & $6.6$ & $12.2$ &
$0.136$ & $0.055$ & $38.9$ & $10.5$ & $50.3$ & $21.6$ & $50.9$ &
$0.102$ & $0.034$ & $35.6$ & $7.9$ & $47.2$ & $16.5$ & $49.2$ \\
\SvmMLg{\gamma} &
$0.163$ & $0.019$ & $14.9$ & $0.3$ & $27.1$ & $12.1$ & $23.6$ &
$0.095$ & $0.050$ & $42.8$ & $10.5$ & $56.2$ & $23.2$ & $54.9$ &
$0.058$ & $0.032$ & $41.8$ & $8.4$ & $56.1$ & $17.7$ & $54.7$ \\
\midrule
\LrML\ {\tiny(RBF)} &
$\mathbf{0.096}$ & $0.019$ & $22.1$ & $1.5$ & $37.2$ & $12.4$ & $34.5$ &
$\mathbf{0.070}$ & $0.047$ & $46.3$ & $13.0$ & $58.4$ & $20.3$ & $57.9$ &
$\mathbf{0.042}$ & $0.033$ & $42.0$ & $10.0$ & $56.2$ & $21.8$ & $53.3$ \\
\SvmML\ {\tiny(RBF)} &
$0.137$ & $\mathbf{0.018}$ & $17.8$ & $\mathbf{1.7}$ & $32.4$ & $16.7$ & $26.4$ &
$0.090$ & $0.047$ & $46.6$ & $15.0$ & $58.1$ & $26.8$ & $58.4$ &
$0.072$ & $0.031$ & $43.3$ & $11.8$ & $57.6$ & $25.9$ & $55.3$ \\
\SvmMLg{\gamma} {\tiny(RBF)} &
$0.099$ & $\mathbf{0.018}$ & $\mathbf{23.1}$ & $1.6$ & $\mathbf{39.0}$ & $\mathbf{18.2}$ & $\mathbf{35.7}$ &
$0.076$ & $0.047$ & $\mathbf{48.6}$ & $\mathbf{16.1}$ & $\mathbf{59.5}$ & $\mathbf{26.9}$ & $59.9$ &
$0.046$ & $\mathbf{0.029}$ & $\mathbf{46.6}$ & $\mathbf{13.3}$ & $\mathbf{61.0}$ & $\mathbf{27.1}$ & $58.7$ \\
\bottomrule

\multicolumn{22}{c}{}\\
\multicolumn{1}{c}{} &
\multicolumn{7}{c}{\small\textbf{medical}} &
\multicolumn{7}{c}{\small\textbf{scene}} &
\multicolumn{7}{c}{\small\textbf{yeast}}
\\\toprule
Method &
\RLoss & \HLoss & \Acc & \SAcc & \Fmicro & \Fmacro & \Finst &
\RLoss & \HLoss & \Acc & \SAcc & \Fmicro & \Fmacro & \Finst &
\RLoss & \HLoss & \Acc & \SAcc & \Fmicro & \Fmacro & \Finst \\
\midrule
\midrule
\RFPCT\ \cite{madjarov2012extensive} &
$0.024$ & $0.014$ & $59.1$ & $53.8$ & $69.3$ & $20.7$ & $61.6$ &
$0.072$ & $0.094$ & $54.1$ & $51.8$ & $66.9$ & $65.8$ & $55.3$ &
$0.167$ & $0.197$ & $47.8$ & $15.2$ & $61.7$ & $32.2$ & $61.4$ \\
\HOMER\ \cite{madjarov2012extensive} &
$0.090$ & $\mathbf{0.012}$ & $71.3$ & $61.0$ & $77.3$ & $28.2$ & $76.1$ &
$0.119$ & $0.082$ & $\mathbf{71.7}$ & $\mathbf{66.1}$ & $\mathbf{76.4}$ & $\mathbf{76.8}$ & $74.5$ &
$0.205$ & $0.207$ & $55.9$ & $21.3$ & $67.3$ & $44.7$ & $\mathbf{68.7}$ \\
\BR\ \cite{madjarov2012extensive} {\tiny(RBF)} &
$\mathbf{0.021}$ & $0.077$ & $20.6$ & $0.0$ & $34.3$ & $36.1$ & $32.8$ &
$\mathbf{0.060}$ & $\mathbf{0.079}$ & $68.9$ & $63.9$ & $76.1$ & $76.5$ & $71.4$ &
$0.164$ & $0.190$ & $52.0$ & $19.0$ & $65.2$ & $39.2$ & $65.0$ \\
\midrule
\LrML\ &
$0.024$ & $0.013$ & $68.7$ & $56.9$ & $76.2$ & $35.1$ & $75.2$ &
$0.081$ & $0.120$ & $58.3$ & $39.4$ & $67.2$ & $68.4$ & $66.3$ &
$0.352$ & $0.264$ & $36.0$ & $8.4$ & $48.3$ & $44.4$ & $48.0$ \\
\SvmML\ &
$0.026$ & $0.013$ & $72.9$ & $\mathbf{62.8}$ & $78.4$ & $34.8$ & $\mathbf{77.5}$ &
$0.082$ & $0.114$ & $60.7$ & $46.0$ & $68.7$ & $69.5$ & $67.7$ &
$0.424$ & $0.280$ & $31.3$ & $5.8$ & $46.7$ & $42.9$ & $45.6$ \\
\SvmMLg{\gamma} &
$0.023$ & $\mathbf{0.012}$ & $\mathbf{73.1}$ & $60.2$ & $78.7$ & $\mathbf{36.7}$ & $77.4$ &
$0.081$ & $0.114$ & $60.4$ & $44.1$ & $68.7$ & $69.5$ & $67.6$ &
$0.366$ & $0.261$ & $35.6$ & $9.5$ & $46.6$ & $44.7$ & $47.3$ \\
\midrule
\LrML\ {\tiny(RBF)} &
$0.031$ & $0.016$ & $64.9$ & $46.5$ & $72.6$ & $28.0$ & $71.4$ &
$0.068$ & $0.096$ & $63.6$ & $54.3$ & $72.0$ & $73.0$ & $70.5$ &
$0.160$ & $0.193$ & $55.1$ & $19.0$ & $67.5$ & $47.1$ & $66.9$ \\
\SvmML\ {\tiny(RBF)} &
$0.027$ & $\mathbf{0.012}$ & $72.5$ & $61.7$ & $\mathbf{78.9}$ & $36.6$ & $76.6$ &
$0.069$ & $0.088$ & $69.1$ & $58.0$ & $75.1$ & $75.9$ & $\mathbf{75.4}$ &
$0.159$ & $0.188$ & $\mathbf{56.2}$ & $\mathbf{21.6}$ & $68.2$ & $48.1$ & $66.7$ \\
\SvmMLg{\gamma} {\tiny(RBF)} &
$0.027$ & $\mathbf{0.012}$ & $72.5$ & $61.6$ & $\mathbf{78.9}$ & $36.6$ & $77.2$ &
$0.064$ & $0.088$ & $68.0$ & $58.0$ & $74.7$ & $75.2$ & $75.0$ &
$\mathbf{0.157}$ & $\mathbf{0.187}$ & $\mathbf{56.2}$ & $19.8$ & $\mathbf{68.4}$ & $\mathbf{48.2}$ & $67.0$ \\
\bottomrule
\end{tabular}
\caption{
Multilabel classification.
The $3$ best performing methods from the study by Madjarov \etal \cite{madjarov2012extensive}
are compared to our multilabel methods from \S~\ref{sec:multilabel-methods}.
\textbf{Baselines:}
\RFPCT\ -- random forest of predicting clustering trees \cite{kocev2007ensembles};
\HOMER\ -- hierarchy of multilabel classifiers \cite{tsoumakas2008effective};
\BR\ -- binary relevance method using \SvmOva.
\HOMER\ and all the methods marked with (RBF) use an RBF kernel.
Following \cite{koyejo2015consistent},
the cut-off threshold $\delta$ for our methods is chosen by cross validation.}
\label{tbl:ml-small-cv}\intablefalse
\end{table*}

\textbf{Conclusion.}
We see that a safe choice for multiclass problems
seems to be the \LrMulti\ loss
as it yields reasonably good results in all top-$k$ errors.
A competitive alternative is the smooth \SvmMultig{\gamma} loss
which can be faster to train
(see runtime experiments in \S~\ref{sec:optimization-multiclass}).
If one wants to optimize directly for a top-$k$ error
(at the cost of a higher top-$1$ error),
then further improvements are possible using either
the smooth top-$k$ SVM or the top-$k$ entropy losses.

\subsection{Multilabel Experiments}\label{sec:multilabel-experiments}

The aim of this section is threefold.
First, we establish competitive performance of our multilabel
classification methods from \S~\ref{sec:multilabel-methods}
comparing them to the top $3$ methods from an extensive
experimental study by Madjarov \etal \cite{madjarov2012extensive}
on $10$ multilabel benchmark datasets of varying scale and complexity.
Next, we discuss an interesting learning setting when top-$k$ classification
methods emerge as a transition step between multiclass and multilabel
approaches.
Finally, we evaluate multiclass, top-$k$, and multilabel classification
methods on Pascal VOC 2007 \cite{everingham2010pascal}
and the more challenging Microsoft COCO \cite{lin2014microsoft}
image classification benchmarks.

\begin{table}[H]\scriptsize\centering\setlength{\tabcolsep}{.22em}
\begin{tabular}{l|cccc||l|cccc}
\toprule
Dataset & $m$ & $n$ & $d$ & $l_c$ & Dataset & $m$ & $n$ & $d$ & $l_c$ \\
\midrule
\midrule
bibtex \cite{katakis2008multilabel} &
$159$ & $5\K$ & $2\K$ & $2.40$ &
enron \cite{klimt2004enron} &
$53$ & $1\K$ & $1\K$ & $3.38$ \\
bookmarks \cite{katakis2008multilabel} &
$208$ & $60\K$ & $2\K$ & $2.03$ &
mediamill \cite{snoek2006challenge} &
$101$ & $31\K$ & $120$ & $4.38$ \\
corel5k \cite{duygulu2002object} &
$374$ & $4.5\K$ & $499$ & $3.52$ &
medical \cite{read2009classifier} &
$45$ & $645$ & $1.5\K$ & $1.25$ \\
delicious \cite{tsoumakas2008effective} &
$983$ & $13\K$ & $500$ & $19.02$ &
scene \cite{boutell2004learning} &
$6$ & $1.2\K$ & $294$ & $1.07$ \\
emotions \cite{trohidis2008multi} &
$6$ & $391$ & $72$ & $1.87$ &
yeast \cite{elisseeff2001kernel} &
$14$ & $1.5\K$ & $103$ & $4.24$ \\
\midrule
VOC 2007 \cite{everingham2010pascal} &
$20$ & $5\K$ & $2\K$ & $1.46$ &
MS COCO \cite{lin2014microsoft} &
$80$ & $83\K$ & $2\K$ & $2.91$ \\
\bottomrule
\end{tabular}
\caption{Statistics of multilabel benchmarks
($m$: \# classes, $n$: \# training examples, $d$: \# features,
$l_c$: label cardinality).
}
\label{tbl:stats-ml}
\end{table}

\begin{table}[t]\scriptsize\centering\setlength{\tabcolsep}{.5em}\intabletrue
\begin{tabular}{l|cc|cc|ccc}
\multicolumn{8}{c}{\small\textbf{emotions}}
\\\toprule
Method & \RLoss & \HLoss & \Acc & \SAcc & \Fmicro & \Fmacro & \Finst \\
\midrule
\midrule
\RFPCT\ \cite{madjarov2012extensive} &
$\mathbf{0.151}$ & $\mathbf{0.189}$ & $51.9$ & $\mathbf{30.7}$ & $67.2$ & $65.0$ & $61.1$ \\
\HOMER\ \cite{madjarov2012extensive} &
$0.297$ & $0.361$ & $47.1$ & $16.3$ & $58.8$ & $57.0$ & $61.4$ \\
\BR\ \cite{madjarov2012extensive} {\tiny(RBF)} &
$0.246$ & $0.257$ & $36.1$ & $12.9$ & $50.9$ & $44.0$ & $46.9$ \\
\midrule
\LrML\ & $0.186$ & $0.239$ & $\mathbf{53.6}$ & $22.8$ & $66.9$ & $66.6$ & $64.0$ \\
\SvmML\ & $0.217$ & $0.238$ & $50.4$ & $23.3$ & $63.4$ & $65.2$ & $63.9$ \\
\SvmMLg{\gamma} & $0.178$ & $0.230$ & $54.0$ & $23.3$ & $\mathbf{67.3}$ & $\mathbf{66.7}$ & $\mathbf{65.5}$ \\
\midrule
\LrML\ {\tiny(RBF)} & $0.225$ & $0.266$ & $47.2$ & $19.3$ & $61.1$ & $62.0$ & $58.4$ \\
\SvmML\ {\tiny(RBF)} & $0.186$ & $0.224$ & $53.0$ & $21.3$ & $65.5$ & $64.3$ & $64.1$ \\
\SvmMLg{\gamma} {\tiny(RBF)} & $0.187$ & $0.224$ & $49.3$ & $21.3$ & $65.5$ & $64.2$ & $61.1$ \\
\bottomrule
\end{tabular}
\caption{Continuation of Table~\ref{tbl:ml-small-cv}.}
\label{tbl:ml-small-cv-emotions}\intablefalse
\end{table}

\textbf{Multilabel classification.}
Here, we seek to establish a solid baseline
to evaluate our implementation of the multilabel \reftext{eq:svm-ml},
smooth \reftext{eq:smooth-svm-ml},
and the \reftext{eq:lr-ml} methods.
To that end, we follow the work of
Madjarov \etal \cite{madjarov2012extensive}
who provide a clear description of the evaluation protocol
and an extensive experimental comparison
of $12$ multilabel classification methods
on $11$ datasets reporting $16$ performance metrics.
We limit our comparison to the $3$ best performing methods
from their study, namely
the random forest of predicting clustering trees \cite{kocev2007ensembles},
the hierarchy of multilabel classifiers \cite{tsoumakas2008effective},
and the binary relevance method using \SvmOva.
We report results on $10$ datasets as there was an issue
with the published train/test splits on the remaining benchmark\footnote{
See \url{https://github.com/tsoumakas/mulan/issues/4} for details.}.
The datasets vary greatly in size and \textbf{label cardinality}
(the average number of labels per example),
as can be seen from the basic statistics in Table~\ref{tbl:stats-ml}.
Further details about each of the datasets can be found in
\cite{madjarov2012extensive}.

We follow closely the evaluation protocol of \cite{madjarov2012extensive}
except for the selection of the cut-off threshold $\delta$
(see \S~\ref{sec:perf-metrics} for definition).
Following \cite{read2009classifier},
Madjarov \etal choose $\delta$ by matching label cardinality
between the training and test data.
While it is fast and easy to compute,
that approach has two drawbacks:
(i) being an instance of transductive learning,
the method requires re-computation of $\delta$
every time test data changes;
(ii) the choice of $\delta$ is not tuned to any performance
measure and is likely to be suboptimal.
In our experiments (not reported here),
we observed generally comparable, but slightly lower results
compared to when $\delta$ is selected on a validation set as discussed next.

Instead, Koyejo \etal \cite{koyejo2015consistent}
recently showed that a consistent classifier is obtained
when one computes $\delta$ by optimizing a given performance measure
on a hold-out validation set.
While there are at most $m n$ distinct values of $\delta$
that would need to be considered, we limit the search to the grid
$\{-10^{(-5.9:.2:1)}, 0, 10^{(-5.9:.2:1)} \}$ of $71$ values.

Following \cite{madjarov2012extensive}, we use $10$-fold
cross-validation to select $C = 1/(\lambda n)$,
the RBF kernel parameter $\theta = 1/(2 \sigma^2)$,
and the threshold $\delta$, as described above.
We use rather large and fine-grained grids both for $C$
(from $2^{-20}$ to $2^5$) and $\theta$ (from $2^{-15}$ to $2^3$).
The smoothing parameter is always set $\gamma=1$.

Tables~\ref{tbl:ml-small-cv} and \ref{tbl:ml-small-cv-emotions}
present our experimental results.
We report $7$ performance metrics previously introduced in
\S~\ref{sec:perf-metrics} and tune the hyper-parameters
for each metric individually.
All metrics, except the rank loss and the hamming loss,
are given in percents.
Since the \RFPCT{} method did not use the RBF kernel in
\cite{koyejo2015consistent},
we also report results with the linear kernel for our methods
in the middle section of each table.

Overall, experimental results indicate competitive performance
of our methods across all datasets and evaluation measures.
Specifically, we highlight that the smooth \reftext{eq:smooth-svm-ml}
with the RBF kernel yields the best performance
in $38$ out of $70$ cases.
On the two largest datasets, bookmarks and delicious,
where the previous methods even struggled to complete training,
we are able to achieve significant performance improvements
both in rank loss as well as in partition-based measures.
Finally, we note that while the previous methods show rather large
variability in performance, all three of our multilabel methods
tend to be more stable and show results that are concentrated
around the best performing method in each of the cases.

\begin{table*}[tp]\scriptsize\centering\setlength{\tabcolsep}{.33em}\intabletrue
\begin{tabular}{c|l|cccccccccccccccccccc|c}
\toprule
Labels & Method &
aero & bike & bird & boat & bottle & bus & car & cat & chair & cow & table & dog & horse & mbike & prsn & plant & sheep & sofa & train & tv & mAP \\
\midrule
\midrule
\multirow{8}{*}{\multirowcell{multi-\\class}} &
\LrMulti\ & $99.2$ & $95.0$ & $92.5$ & $92.3$ & $61.9$ & $86.6$ & $93.4$ & $95.8$ & $55.3$ & $85.8$ & $82.0$ & $92.1$ & $97.2$ & $91.5$ & $93.2$ & $70.8$ & $82.1$ & $82.6$ & $97.8$ & $81.1$ & $86.4$ \\
&
\LrTopK{k} & $99.1$ & $\mathbf{96.0}$ & $92.3$ & $95.4$ & $62.1$ & $89.2$ & $93.9$ & $95.3$ & $58.5$ & $88.1$ & $72.8$ & $94.2$ & $97.3$ & $93.8$ & $93.0$ & $67.9$ & $87.7$ & $83.4$ & $97.6$ & $\mathbf{85.3}$ & $87.1$ \\
\noalign{\smallskip}\cline{2-23}\noalign{\smallskip}
&
\SvmMulti\ & $\mathbf{99.5}$ & $94.0$ & $\mathbf{97.0}$ & $\mathbf{96.8}$ & $62.1$ & $\mathbf{93.4}$ & $\mathbf{94.6}$ & $\mathbf{97.5}$ & $65.0$ & $89.9$ & $\mathbf{85.1}$ & $\mathbf{97.4}$ & $\mathbf{97.8}$ & $\mathbf{95.5}$ & $93.7$ & $\mathbf{71.0}$ & $90.2$ & $84.4$ & $\mathbf{98.7}$ & $82.3$ & $89.3$ \\
&
\SvmTopKa{k} & $99.3$ & $95.5$ & $94.7$ & $95.5$ & $61.5$ & $91.9$ & $\mathbf{94.6}$ & $97.4$ & $\mathbf{66.7}$ & $89.0$ & $80.8$ & $97.1$ & $97.7$ & $95.4$ & $\mathbf{95.3}$ & $70.7$ & $90.2$ & $84.3$ & $98.5$ & $84.8$ & $89.0$ \\
&
\SvmTopKb{k} & $99.4$ & $95.5$ & $96.0$ & $95.9$ & $63.5$ & $92.6$ & $\mathbf{94.6}$ & $97.4$ & $66.1$ & $\mathbf{90.2}$ & $84.1$ & $97.1$ & $\mathbf{97.8}$ & $\mathbf{95.5}$ & $95.0$ & $70.9$ & $\mathbf{91.7}$ & $\mathbf{84.6}$ & $98.5$ & $83.9$ & $\mathbf{89.5}$ \\
\noalign{\smallskip}\cline{2-23}\noalign{\smallskip}
&
\SvmMultig{\gamma} & $99.4$ & $95.4$ & $95.0$ & $95.5$ & $64.3$ & $91.9$ & $94.4$ & $97.0$ & $64.0$ & $90.0$ & $84.7$ & $96.1$ & $97.7$ & $94.8$ & $94.2$ & $70.6$ & $89.7$ & $\mathbf{84.6}$ & $98.3$ & $83.3$ & $89.0$ \\
&
\SvmTopKag{k}{\gamma} & $99.3$ & $\mathbf{96.0}$ & $93.2$ & $95.0$ & $63.6$ & $90.7$ & $94.3$ & $97.0$ & $62.4$ & $89.3$ & $79.7$ & $96.0$ & $97.6$ & $95.0$ & $95.0$ & $70.2$ & $89.5$ & $84.4$ & $98.3$ & $83.9$ & $88.5$ \\
&
\SvmTopKbg{k}{\gamma} & $99.3$ & $95.6$ & $94.7$ & $95.2$ & $\mathbf{64.4}$ & $91.8$ & $94.5$ & $97.1$ & $65.1$ & $89.8$ & $84.2$ & $96.3$ & $97.7$ & $94.9$ & $94.8$ & $70.6$ & $89.7$ & $\mathbf{84.6}$ & $98.4$ & $84.0$ & $89.1$ \\
\midrule
\midrule
\multirow{3}{*}{\multirowcell{multi-\\label}} &
\LrML\ & $98.8$ & $94.2$ & $92.3$ & $90.6$ & $56.6$ & $83.3$ & $92.1$ & $95.8$ & $65.0$ & $85.3$ & $84.0$ & $93.9$ & $96.5$ & $93.6$ & $92.5$ & $69.4$ & $83.8$ & $81.2$ & $97.7$ & $78.2$ & $86.2$ \\
&
\SvmML\ & $99.5$ & $\mathbf{96.5}$ & $\mathbf{97.5}$ & $\mathbf{96.7}$ & $\mathbf{71.8}$ & $\mathbf{93.6}$ & $\mathbf{95.3}$ & $\mathbf{97.8}$ & $\mathbf{79.3}$ & $\mathbf{92.0}$ & $87.6$ & $\mathbf{98.4}$ & $98.2$ & $\mathbf{96.6}$ & $97.9$ & $\mathbf{73.1}$ & $\mathbf{93.3}$ & $\mathbf{83.8}$ & $\mathbf{98.7}$ & $\mathbf{88.5}$ & $\mathbf{91.8}$ \\
&
\SvmMLg{\gamma} & $\mathbf{99.6}$ & $96.3$ & $97.1$ & $96.4$ & $69.5$ & $93.3$ & $94.9$ & $97.5$ & $76.7$ & $91.3$ & $\mathbf{88.0}$ & $98.0$ & $\mathbf{98.3}$ & $\mathbf{96.6}$ & $\mathbf{98.1}$ & $72.6$ & $93.2$ & $83.7$ & $98.6$ & $88.0$ & $91.4$ \\
\bottomrule
\end{tabular}
\caption{
Pascal VOC 2007 \textbf{classification} results.
Evaluation of multiclass, top-$k$, and multilabel classification methods.
Methods in the ``multiclass'' section above
use only a \emph{single} label per image,
while methods in the ``multilabel'' section use all annotated labels.
Please see the section \textbf{Multiclass to multilabel}
for further details on the learning setting.}
\label{tbl:ml-voc07-ap}\intablefalse
\end{table*}

\begin{table}[th]\scriptsize\centering\setlength{\tabcolsep}{.2em}\intabletrue
\begin{tabular}{l|ccc|cc|cc|ccc}
\toprule
Method & \RecAtK{1} & \RecAtK{3} & \RecAtK{5} & \RLoss & \HLoss & \Acc & \SAcc & \Fmicro & \Fmacro & \Finst \\
\midrule
\midrule
\LrMulti\ & $76.2$ & $94.3$ & $98.0$ & $0.016$ & $0.029$ & $73.2$ & $54.4$ & $78.1$ & $75.0$ & $80.0$ \\
\LrTopK{k} & $76.1$ & $94.3$ & $97.8$ & $0.016$ & $0.038$ & $48.1$ & $44.0$ & $59.8$ & $72.1$ & $61.7$ \\
\midrule
\SvmMulti\ & $76.5$ & $94.1$ & $97.6$ & $0.017$ & $0.025$ & $76.0$ & $60.6$ & $80.3$ & $77.5$ & $81.9$ \\
\SvmTopKa{k} & $76.4$ & $94.6$ & $98.0$ & $0.015$ & $0.025$ & $76.6$ & $61.3$ & $81.4$ & $77.9$ & $81.7$ \\
\SvmTopKb{k} & $\mathbf{76.8}$ & $\mathbf{95.2}$ & $98.1$ & $\mathbf{0.014}$ & $\mathbf{0.024}$ & $\mathbf{77.3}$ & $\mathbf{62.0}$ & $82.0$ & $\mathbf{78.6}$ & $\mathbf{83.1}$ \\
\midrule
\SvmMultig{\gamma} & $76.6$ & $94.7$ & $98.0$ & $0.015$ & $0.025$ & $75.9$ & $60.0$ & $81.2$ & $78.3$ & $82.3$ \\
\SvmTopKag{k}{\gamma} & $76.4$ & $95.0$ & $\mathbf{98.2}$ & $0.015$ & $0.025$ & $76.2$ & $59.3$ & $81.2$ & $76.7$ & $81.5$ \\
\SvmTopKbg{k}{\gamma} & $76.7$ & $\mathbf{95.2}$ & $98.1$ & $\mathbf{0.014}$ & $\mathbf{0.024}$ & $76.7$ & $60.5$ & $\mathbf{82.2}$ & $78.1$ & $82.8$ \\
\midrule
\midrule
\LrML\ & $76.5$ & $96.3$ & $98.9$ & $0.010$ & $0.027$ & $75.2$ & $59.8$ & $81.3$ & $77.4$ & $80.7$ \\
\SvmML\ & $\mathbf{78.0}$ & $96.8$ & $98.9$ & $\mathbf{0.008}$ & $0.019$ & $81.6$ & $69.8$ & $85.8$ & $81.9$ & $86.1$ \\
\SvmMLg{\gamma} & $77.9$ & $\mathbf{97.3}$ & $\mathbf{99.1}$ & $\mathbf{0.008}$ & $\mathbf{0.018}$ & $\mathbf{82.4}$ & $\mathbf{70.8}$ & $\mathbf{86.8}$ & $\mathbf{83.0}$ & $\mathbf{86.4}$ \\
\bottomrule
\end{tabular}
\caption{
Pascal VOC 2007 multilabel classification results.}
\label{tbl:ml-voc07-cv}\intablefalse
\end{table}

\textbf{Multiclass to multilabel.}
Collecting ground truth annotation is hard.
Even when the annotation is simply an image level tag,
providing a \emph{consistent} and \emph{exhaustive} list
of labels for every image in the training set would require
significant effort.
It is much easier to provide a weaker form of annotation
where only a single prominent object is tagged.
An interesting question is then whether it is still possible
to train multilabel classifiers from multiclass annotation.
And if so, how large is the performance gap compared to methods
trained with full multilabel annotation?
In the following, we set to explore that setting
and answer the questions above.

We also note that top-$k$ classification emerges naturally
as an intermediate step between multiclass and multilabel learning.
Recall that top-$k$ loss functions operate in the multiclass setting
where there is a single label per example, but that label is hard
to guess correctly on the first attempt.
One could imagine that the example is actually associated with
$k$ labels, but only a single label is revealed in the annotation.
Therefore, it is also interesting to see if our top-$k$ loss
functions can offer an advantage over the classic multiclass losses
in this setting.

To evaluate the multiclass, top-$k$, and multilabel loss functions
on a common task, we choose two multilabel image classification benchmarks:
Pascal VOC 2007 and Microsoft COCO.
Multilabel methods are trained using full image level annotation
(\ie all class labels, but no bounding boxes or segmentation),
while multiclass and top-$k$ methods are trained
using a \emph{single} label per image.
Both datasets offer object level bounding box annotations
which can be used to estimate relative sizes of objects in the scene.
For multiclass training, we only keep the label of the largest object,
which is our proxy to estimating the prominent object in the image.
All methods are evaluated using full annotation at test time.
Note that except for pruning the training labels,
we do \emph{not} use bounding boxes anywhere during training or testing.

\textbf{Experimental setup.}
We use $5\K$ images for training and $5\K$ for testing
on Pascal VOC 2007, and $83\K$ for training and $40\K$ for testing on
the MS COCO validation set.
We split the training data in half for parameter tuning,
and re-train on the full set for testing.
We tune the regularization parameter $C = 1/(\lambda n)$
in the range from $2^{-20}$ to $2^{15}$,
and the top-$k$ parameter $k$ in the range $\{2, 3, 4, 5\}$.
For the partition-based measures, we also tune the threshold
$\delta$ in the range $[0.1, 10]$ with $100$ equally spaced points.
That range was chosen by observing the distribution of $\delta$
when it is computed by matching the label cardinality
between training and test data.
All hyper-parameters are tuned for each method and performance metric
individually.

To isolate the effect of loss functions on classifier training
from feature learning, we follow the classic approach of extracting
features as a pre-processing step and then train our classifiers
on the fixed image representation.
We use our own implementation of SDCA based solvers for all of the methods
considered in this section.
That offers strong convergence guarantees
due to (i) convexity of the objective
and (ii) having the duality gap as the stopping criterion.

\begin{table*}[tp]\scriptsize\centering\setlength{\tabcolsep}{.675em}\intabletrue
\begin{tabular}{c|l|c|cccc|cccc|cc|cc|ccc}
\toprule
Labels & Method & mAP & \PrecAtK{1} & \PrecAtK{2} & \PrecAtK{3} & \PrecAtK{5} & \RecAtK{1} & \RecAtK{3} & \RecAtK{5} & \RecAtK{10} & \RLoss & \HLoss & \Acc & \SAcc & \Fmicro & \Fmacro & \Finst \\
\midrule
\midrule
\multirow{7}{*}{\multirowcell{multi-\\class}} &
\LrMulti\ & $54.6$ & $92.6$ & $66.9$ & $52.2$ & $37.0$ & $44.8$ & $65.8$ & $73.9$ & $82.9$ & $0.066$ & $0.028$ & $43.4$ & $15.9$ & $52.4$ & $41.1$ & $55.8$ \\
\noalign{\smallskip}\cline{2-18}\noalign{\smallskip}
&
\SvmMulti\ & $54.2$ & $92.8$ & $66.9$ & $51.9$ & $36.6$ & $44.9$ & $65.6$ & $73.4$ & $83.5$ & $0.057$ & $\mathbf{0.025}$ & $48.3$ & $20.7$ & $55.6$ & $43.3$ & $60.1$ \\
&
\SvmTopKa{k} & $58.3$ & $92.8$ & $68.1$ & $53.2$ & $\mathbf{37.5}$ & $44.9$ & $66.8$ & $\mathbf{74.8}$ & $83.7$ & $0.054$ & $\mathbf{0.025}$ & $48.8$ & $20.8$ & $56.8$ & $44.4$ & $59.9$ \\
&
\SvmTopKb{k} & $59.0$ & $\mathbf{93.2}$ & $\mathbf{68.4}$ & $53.2$ & $\mathbf{37.5}$ & $\mathbf{45.0}$ & $\mathbf{66.9}$ & $74.7$ & $\mathbf{84.0}$ & $\mathbf{0.053}$ & $\mathbf{0.025}$ & $\mathbf{49.7}$ & $\mathbf{21.2}$ & $\mathbf{57.5}$ & $44.5$ & $\mathbf{61.3}$ \\
\noalign{\smallskip}\cline{2-18}\noalign{\smallskip}
&
\SvmMultig{\gamma} & $58.1$ & $93.0$ & $67.7$ & $52.8$ & $37.2$ & $\mathbf{45.0}$ & $66.5$ & $74.3$ & $83.6$ & $0.056$ & $\mathbf{0.025}$ & $48.9$ & $20.2$ & $56.5$ & $\mathbf{44.6}$ & $60.6$ \\
&
\SvmTopKag{k}{\gamma} & $58.4$ & $92.8$ & $68.1$ & $53.2$ & $\mathbf{37.5}$ & $44.9$ & $66.8$ & $\mathbf{74.8}$ & $83.7$ & $0.055$ & $\mathbf{0.025}$ & $48.4$ & $20.4$ & $57.0$ & $44.5$ & $60.3$ \\
&
\SvmTopKbg{k}{\gamma} & $\mathbf{59.1}$ & $\mathbf{93.2}$ & $\mathbf{68.4}$ & $\mathbf{53.3}$ & $37.4$ & $\mathbf{45.0}$ & $\mathbf{66.9}$ & $74.7$ & $83.7$ & $0.054$ & $\mathbf{0.025}$ & $49.3$ & $20.9$ & $57.2$ & $44.4$ & $60.9$ \\
\midrule
\midrule
\multirow{3}{*}{\multirowcell{multi-\\label}} &
\LrML\ & $58.2$ & $92.8$ & $76.3$ & $61.8$ & $44.2$ & $44.9$ & $75.5$ & $84.6$ & $\mathbf{93.2}$ & $0.021$ & $0.030$ & $43.7$ & $16.9$ & $52.9$ & $49.6$ & $55.6$ \\
&
\SvmML\ & $63.0$ & $92.1$ & $72.9$ & $57.4$ & $40.6$ & $44.1$ & $70.8$ & $78.9$ & $89.1$ & $0.040$ & $0.024$ & $49.5$ & $25.6$ & $58.3$ & $50.6$ & $60.5$ \\
&
\SvmMLg{\gamma} & $\mathbf{71.0}$ & $\mathbf{95.7}$ & $\mathbf{79.5}$ & $\mathbf{63.4}$ & $\mathbf{44.6}$ & $\mathbf{46.2}$ & $\mathbf{77.1}$ & $\mathbf{85.3}$ & $\mathbf{93.2}$ & $\mathbf{0.020}$ & $\mathbf{0.021}$ & $\mathbf{57.4}$ & $\mathbf{29.8}$ & $\mathbf{65.5}$ & $\mathbf{58.9}$ & $\mathbf{67.9}$ \\
\bottomrule
\end{tabular}
\caption{
MS COCO multilabel \textbf{classification} results.
Methods in the ``multiclass'' section
use only a \emph{single} label per image,
while methods in the ``multilabel'' section use all annotated labels.
Please see the section \textbf{Multiclass to multilabel}
for further details on the learning setting,
and \S~\ref{sec:perf-metrics} for details on the evaluation measures.}
\label{tbl:ml-coco-cv}\intablefalse
\end{table*}

Our feature extraction pipeline is fairly common and follows the steps
outlined in \cite{simonyan2014very,wei2016region}.
We compute multiple feature vectors per image.
Every original image is resized isotropically so that the smallest side is
equal to $Q \in \{ 256, 384, 512 \}$ pixels, and then horizontal flips
are added for a total of $6$ images at $3$ scales.
We use MatConvNet \cite{vedaldi15matconvnet}
and apply the ResNet-$152$ model \cite{he2016deep}
which has been pre-trained on ImageNet.
We extract features from the pool5 layer
and obtain about $500$ feature vectors of dimension $2048$
per image on Pascal VOC
(the exact number depends on the size of the original image).
To reduce computational costs on COCO,
we increase the stride of that layer to $2$
for $Q \in \{ 384, 512 \}$,
which yields about $140$ feature vectors per image
and a total of $n=12\M$ training examples.
Unlike \cite{wei2016region}, we do not compute an additional
global descriptor and also perform no normalization.
Our preliminary experiments showed no advantage in doing so,
and we decided to keep the pipeline close to the original
ResNet network.

Every feature vector can be mapped to a region in the original image.
For training, we simply replicate the same image labels effectively
increasing the size of the training set.
At test time, we obtain a single ranking of class labels
per image by max pooling the scores for each class.
We follow this basic setup, but note that a $1-2\%$ improvement
is possible with a more sophisticated aggregation of information
from the different image regions \cite{wei2016region,zhao2016regional}.

\textbf{Pascal VOC 2007.}
Here, we discuss the results presented in
Tables~\ref{tbl:ml-voc07-ap} and \ref{tbl:ml-voc07-cv}.
We start with the first table which reports the standard VOC evaluation
measure, the mean AP.
First, we compare top-$1$ (multiclass) and top-$k$ classification methods.
As before, although the differences are small,
we see consistent improvements in each of the three groups:
\LrMulti\ to \LrTopK{k},
\SvmMulti\ to \SvmTopKb{k},
and \SvmMultig{\gamma} to \SvmTopKbg{k}{\gamma}.
The best top-$1$ method is \SvmMulti{}
with $89.3\%$ mAP,
which is outperformed by \SvmTopKb{k} reporting the best multiclass
result of $89.5\%$ mAP.

Next, we look at the performance gap between
multiclass and multilabel settings.
The best mAP of $91.8\%$ is achieved by the multilabel SVM, \SvmML{},
which exploits full annotation to boost its performance.
However, the gap of just above $2\%$ suggests a non-trivial
trade-off between the additional annotation effort
and the resulting classification performance.
One limitation of the results on VOC 2007 is the relatively low
label cardinality of only $1.5$ labels per image.
We will see how the picture changes on COCO where
the label cardinality is about $3$ labels per image.

Comparing the smooth and nonsmooth losses,
we see that nonsmooth loss functions tend to perform better on this dataset.
Moreover, SVM seems to perform significantly better than softmax.
While this is a somewhat surprising result, it has been observed
previously, \eg with the R-CNN detector \cite{lenc15rcnn,girshick2015fast},
and with deeply-supervised CNNs \cite{lee2015deeply},
even though their comparison was to OVA SVM.

Finally, we note that the current state of the art
classification results on VOC 2007 are reported in
\cite{wei2016region,zhao2016regional,wang2016beyond}.
Our $91.8\%$ mAP of \SvmML{} matches exactly the result of
LSSVM-Max in \cite{wei2016region},
which operates in the setting closest to ours
in terms of image representation and the learning architecture.
Their proposed PRSVM method performs additional inference
(as opposed to simple max pooling) and achieves $92.9\%$ mAP.
Multiscale orderless pooling from \cite{zhao2016regional}
is directly comparable to our setting and yields $90.8\%$ mAP.
Performing inference on the extracted image regions,
they too report around $93\%$ mAP,
while additionally exploiting bounding box annotations
boosts the performance to $93.7\%$.

While mAP is the established performance measure on Pascal VOC datasets,
it does not evaluate how well a method captures inter-class
correlations since the AP is computed for each class independently.
To address this limitation, we also report a number of multilabel
performance metrics from \S~\ref{sec:perf-metrics}
in Table~\ref{tbl:ml-voc07-cv}.
The best performing method in the multiclass category is again \SvmTopKb{k},
but the improvement over the baseline \SvmMulti\ is more pronounced.
Furthermore, the smooth \SvmMLg{\gamma} now clearly outperforms
its nonsmooth counterpart
also significantly increasing the gap between
multiclass and multilabel methods.

\textbf{MS COCO.}
Table~\ref{tbl:ml-coco-cv} presents
our results on the MS COCO benchmark.
The general trend is similar to that observed on VOC 2007:
top-$k$ methods tend to outperform top-$1$ multiclass baselines,
but are outperformed by multilabel methods that exploit full annotation.
However, the differences between the methods are more meaningful on
this dataset.
In particular, smooth \SvmTopKbg{k}{\gamma} achieves $59.1\%$ mAP,
which is a $1\%$ improvement over \SvmMultig{\gamma},
while multilabel \SvmMLg{\gamma} boosts the performance to $71\%$.
The improvement of over $10\%$ highlights the value of multilabel
annotation, even though this result is subject to the bias of our
label selection procedure for multiclass methods:
small objects may have not been repesented well.
That class imbalance could be also the reason for relatively poor mAP
performance of \SvmMulti\ and \LrMulti\ methods in these experiments.

The current state of the art classification results on COCO are reported
in \cite{zhao2016regional}.
A comparable architecture achieved $69.7\%$ mAP,
while performing inference on the multiple regions per image
and exploiting the bounding box annotations boosted the performance
to $73\%$ mAP.

Looking at multilabel evaluation measures,
we can also make a few interesting observations.
First, the rank loss seems to correlate well with
the other performance measures,
which is good since that is the metric
that our loss functions are designed to optimize.
Second, strong performance at \PrecAtK{1}
suggests that a single guess is generally sufficient to guess
\emph{a} correct label. However, due to high class imbalance
this result is not too impressive and is humbled by the performance
of \RecAtK{k}: even $10$ attempts may not suffice to guess all relevant labels.
The difficulty of properly ranking the less represented classes
is also highlighted by the relatively low
accuracy and subset accuracy results,
although the latter metric may be too stringent for a large scale benchmark.

\section{Conclusion}%
\label{sec:conclusion}
We have done an extensive experimental study of
multiclass, top-$k$, and multilabel performance optimization.
We observed that the softmax loss and the smooth hinge loss
are competitive across all top-$k$ errors
and should be considered the primary candidates in practice.
Our new top-$k$ loss functions
can further improve these results, especially
if one is targeting a particular top-$k$ error as the performance measure,
or if the training examples are multilabel in nature.
The latter transition from multiclass to multilabel classification
indicates that effective multilabel classifiers can be trained
from single label annotations.
Our results also show that the classical multilabel SVM is
competitive in mAP on Pascal VOC 2007, however,
the proposed \emph{smooth} multilabel SVM outperforms the competing methods
in other metrics on Pascal VOC, and in all metrics on MS COCO.
Finally, we would like to highlight our optimization schemes
for \reftext{eq:topk-entropy}, \SvmTopKg{k}{\gamma}, and \reftext{eq:smooth-svm-ml},
which include the softmax loss and multiclass, multilabel SVM
as special cases.

\bibliographystyle{IEEEtran}
\bibliography{main}

\appendices

\propfalse
\prooftrue

\section{Proofs from \S~\ref{sec:loss-functions}}
\subsection{Proof of Proposition~\ref{prop:topk-smooth}}

\subsection{Proof of Proposition~\ref{prop:softmax-conjugate}}

\subsection{Proof of Proposition~\ref{prop:topk-entropy-primal}}

\subsection{Proof of Proposition~\ref{prop:svm-ml-conj}}

\subsection{Proof of Proposition~\ref{prop:smooth-svm-ml}}

\subsection{Proof of Proposition~\ref{prop:lr-ml-conj}}

\section{Proofs from \S~\ref{sec:topk-calibration}}
\subsection{Proof of Lemma~\ref{lem:bayes-topk-error}}
\ifprop
\begin{lemma}\label{lem:bayes-topk-error}
The Bayes optimal top-$k$ error at $x$ is
\begin{align*}
\min_{g \in \Rb^m} \Exp_{Y \given X}[\kerr{Y,g} \given X=x]
= 1 - \tsum_{j=1}^k p_{\tau_j}(x),
\end{align*}
where
$ p_{\tau_1}(x) \geq p_{\tau_2}(x) \geq \ldots \geq p_{\tau_m}(x)$.
A classifier $f$ is \textbf{top-$k$ Bayes optimal at $x$} if and only if
\begin{align*}
\big\{ y \given f_y(x) \geq f_{\pi_k}(x) \big\} \subset
\big\{ y \given p_y(x) \geq p_{\tau_k}(x) \big\} ,
\end{align*}
where $ f_{\pi_1}(x) \geq f_{\pi_2}(x) \geq \ldots \geq f_{\pi_m}(x)$.
\end{lemma}
\fi
\ifproof
\begin{proof}
For any $g = f(x) \in \R^m$, let $\pi$ be a permutation such that
$g_{\pi_1} \geq g_{\pi_2} \geq \ldots \geq g_{\pi_m}$.
The expected top-$k$ error at $x$ is
\begin{align*}
&\Exp_{Y \given X}[\kerr{Y,g} \given X=x] 
= \tsum_{y \in \Yc} \iv{g_{\pi_k} > g_{y}} p_y(x) \\
&= \tsum_{y \in \Yc} \iv{g_{\pi_k} > g_{\pi_y}} p_{\pi_y}(x)
= \tsum_{j=k+1}^m p_{\pi_j}(x) \\
&= 1 - \tsum_{j=1}^k p_{\pi_j}(x) .
\end{align*}
The error is minimal when $\tsum_{j=1}^k p_{\pi_j}(x)$ is maximal,
which corresponds to taking the $k$ largest conditional probabilities
$\tsum_{j=1}^k p_{\tau_j}(x)$ and yields the Bayes optimal top-$k$ error at $x$.

Since the relative order within $\{ p_{\tau_j}(x) \}_{j=1}^k$ is irrelevant
for the top-$k$ error, any classifier $f(x)$, for which the sets
$\{\pi_1, \ldots, \pi_k\}$ and $\{\tau_1, \ldots, \tau_k\}$
coincide, is Bayes optimal.

Note that we assumed \Wlog that there is a clear cut
$p_{\tau_k}(x) > p_{\tau_{k+1}}(x)$
between the $k$ most likely classes and the rest.
In general, ties can be resolved arbitrarily as long as we can guarantee
that the $k$ largest components of $f(x)$ correspond to the classes (indexes)
that yield the maximal sum $\tsum_{j=1}^k p_{\pi_j}(x)$ and lead to
top-$k$ Bayes optimality.
\end{proof}
\fi

\subsection{Proof of Proposition~\ref{prop:calibrated-hinge}}
\ifprop
\begin{proposition}\label{prop:calibrated-hinge}
OVA SVM is not top-$k$ calibrated.
\end{proposition}
\fi
\ifproof
\begin{proof}
First, we show that the Bayes optimal function for the binary hinge loss is
\begin{align*}
f^*(x) &= 2 \iv{\Pr(Y=1 \given X=x) > \tfrac{1}{2}} -1 .
\end{align*}
We decompose the expected loss as
\[ \Exp_{X,Y}[L(Y,f(X))]=\Exp_{X}[\Exp_{Y|X}[L(Y,f(x)) \given X = x]].\]
Thus, one can compute the Bayes optimal classifier $f^*$
pointwise by solving
\[ \argmin_{\alpha \in \R} \Exp_{Y|X}[L(Y,\alpha) \given X=x],\]
for every $x \in \R^d$, which leads to the following problem
\begin{align*}
\argmin_{\alpha \in \R} \;
\max\{0, 1-\alpha\} p_1(x) + \max\{0, 1+\alpha\} p_{-1}(x) ,
\end{align*}
where $p_y(x) \bydef \Pr(Y=y \given X=x)$.
It is obvious that the optimal $\alpha^*$ is contained in $[-1,1]$.
We get
\[
\argmin_{-1\leq \alpha\leq 1} \;
(1 - \alpha) p_1(x) + (1 + \alpha)p_{-1}(x).
\]
The minimum is attained at the boundary and we get
\[
f^*(x) =
\begin{cases} +1 & \text{if } p_1(x) > \frac{1}{2} , \\
-1 & \text{if } p_1(x) \leq \frac{1}{2} .
\end{cases}
\]
Therefore, the Bayes optimal classifier for the hinge loss is not
a strictly monotonically increasing function of $p_1(x)$.

To show that OVA hinge is not top-$k$ calibrated, we construct an example problem
with $3$ classes and $p_1(x) = 0.4$, $p_2(x) = p_3(x) = 0.3$.
Note that for every class $y = 1, 2, 3$, the Bayes optimal binary classifier is $-1$,
hence the predicted ranking of labels is arbitrary and
may not produce the Bayes optimal top-$k$ error.
\end{proof}
\fi

\subsection{Proof of Proposition~\ref{prop:calibrated-lr}}
\ifprop
\begin{proposition}\label{prop:calibrated-lr}
OVA logistic regression is top-$k$ calibrated.
\end{proposition}
\fi
\ifproof
\begin{proof}
First, we show that the Bayes optimal function for the binary logistic loss is
\begin{align*}
f^*(x) &= \log\Big(\frac{p_1(x)}{1-p_1(x)}\Big) .
\end{align*}
As above, the pointwise optimization problem is
\[ \argmin_{\alpha \in \R} \;
\log(1+\exp(-\alpha))p_1(x) + \log(1+\exp(\alpha))p_{-1}(x).\]
The logistic loss is known to be convex and differentiable and thus the optimum can be computed via
\[
\frac{-\exp(-\alpha)}{1+\exp(-\alpha)}p_1(x)
+ \frac{\exp(\alpha)}{1+\exp(\alpha)}p_{-1}(x)=0 .
\]
Re-writing the first fraction we get
\[
\frac{-1}{1+\exp(\alpha)}p_1(x)
+ \frac{\exp(\alpha)}{1+\exp(\alpha)}p_{-1}(x)=0 ,
\]
which can be solved as
$ \alpha^* = \log\Big(\frac{p_1(x)}{p_{-1}(x)}\Big)$
and leads to the formula for
the Bayes optimal classifier stated above.

We check now that the function $\phi: (0,1) \rightarrow \Rb$
defined as $\phi(x)=\log(\frac{x}{1-x})$ is strictly monotonically increasing.
\begin{align*}
\phi'(x) &= \frac{1-x}{x}\big( \frac{1}{1-x} + \frac{x}{(1-x)^2}\big) \\
&= \frac{1-x}{x}\frac{1}{(1-x)^2} = \frac{1}{x(1-x)} > 0, \quad \forall x \in (0,1).
\end{align*}
The derivative is strictly positive on $(0,1)$, which implies that
$\phi$ is strictly monotonically increasing.
The logistic loss, therefore, fulfills the conditions of
Lemma~\ref{lem:ova-topk-calibrated} and is top-$k$ calibrated for
any $1 \leq k \leq m$.
\end{proof}
\fi

\subsection{Proof of Proposition~\ref{prop:smooth-ova-hinge-calibrated}}

\subsection{Proof of Proposition~\ref{prop:multi-hinge-topk-calibrated}}

\subsection{Proof of Proposition~\ref{prop:softmax-topk-calibrated}}
\ifprop
\begin{proposition}\label{prop:softmax-topk-calibrated}
Multiclass softmax loss is top-$k$ calibrated.
\end{proposition}
\fi
\ifproof
\begin{proof}
The multiclass softmax loss is (top-$1$) calibrated for the zero-one error
in the following sense.
If
\[
f^*(x) \in \argmin_{ g \in \R^m} \Exp_{Y|X}[L(Y,g) \given X=x] ,
\]
then for some $\alpha>0$ and all $y \in \Yc$
\[
f^*_y(x) = \begin{cases}
\log(\alpha \, p_y(x) ) & \text{if } p_y(x) > 0, \\
-\infty & \text{otherwise},
\end{cases}
\]
which implies
\[
\argmax_{y \in \Yc} f^*_y(x) = \argmax_{y \in \Yc} \Pr(Y=y \given X=x) .
\]
We now prove this result and show that it also generalizes
to top-$k$ calibration for $k>1$.
Using the identity
$$
L(y,g) = \log\big(\tsum_{j \in \Yc} e^{g_j - g_y}\big)
= \log\big(\tsum_{j \in \Yc} e^{g_j}\big) - g_y
$$
and the fact that $\sum_{y \in \Yc} p_y(x) = 1$,
we write for a $g \in \R^m$
\begin{align*}
\Exp&_{Y|X}[L(Y,g) \given X=x] \\
&= \sum_{y \in \Yc} L(y,g) p_y(x)
= \log\big( \sum_{y \in \Yc} e^{g_y} \big) - \sum_{y \in \Yc} g_y p_x(y) .
\end{align*}
As the loss is convex and differentiable,
we get the global optimum by computing a critical point.
We have
\begin{align*}
\frac{\partial}{\partial g_j}\Exp_{Y|X}[L(Y,g) \given X=x] =
\frac{e^{g_j}}{\sum_{y \in \Yc} e^{g_y}} - p_j(x) = 0
\end{align*}
for $j \in \Yc$.
We note that the critical point is not unique as multiplication
$g\rightarrow \kappa g$ leaves the equation invariant for any $\kappa > 0$.
One can verify that $e^{g_j} = \alpha p_j(x)$
satisfies the equations for any $\alpha > 0$.
This yields a solution
\begin{align*}
f^*_y(x) &= \begin{cases}
\log(\alpha p_y(x) ) & \text{if } p_y(x) > 0 , \\
-\infty & \text{otherwise},
\end{cases}
\end{align*}
for any fixed $\alpha>0$.
We note that $f^*_y$ is a strictly monotonically increasing function
of the conditional class probabilities.
Therefore, it preserves the ranking of $p_y(x)$
and implies that $f^*$ is top-$k$ calibrated for any $1 \leq k \leq m$.
\end{proof}
\fi

\subsection{Proof of Proposition~\ref{prop:truncated-topk-entropy-topk-calibrated}}

\section{Proofs from \S~\ref{sec:optimization}}
\subsection{Proof of Proposition~\ref{prop:smooth-topk-hinge-update}}

\subsection{Proof of Proposition~\ref{prop:topk-entropy-update}}
\ifprop
\begin{proposition}\label{prop:topk-entropy-update}
Let $L$ in \eqref{eq:primal-dual} be the \reftext{eq:topk-entropy} loss
and $L^*$ be its convex conjugate as in \eqref{eq:softmax-conjugate}
with $\Sm$ replaced by $\Ska$.
The dual variables $a_i$ corresponding to $(x_i,y_i)$ are updated as:
\begin{align}\label{eq:topk-entropy-update}
\begin{cases}
\wo{a_i}{y_i} \hspace*{-.75em}
&= - \tfrac{1}{\lambda n} \argmin\limits_{x \in \Ska} \big\{
\tfrac{\alpha}{2} ( \inner{x,x} + s^2 )
- \inner{b, x} + \inner{x, \log x} \\[-.75em]
& \qquad\qquad\qquad\quad\;
+ (1 - s) \log(1 - s) \given s = \inner{\ones, x} \big\} , \\[-.5em]
a_{y_i,i} \hspace*{-.75em}
&= - \tsum_{j \neq y_i} a_{j,i} ,
\end{cases}\raisetag{.5\baselineskip}
\end{align}
where
$\alpha = \frac{\inner{x_i, x_i}}{\lambda n}$,
$b = \wo{q}{y_i} - q_{y_i}\ones$,
$q = \tra{W} x_i - \inner{x_i,x_i} a_i$.
\end{proposition}
\fi
\ifproof
\begin{proof}
Let $v \bydef - \lambda n a_i$ and $y = y_i$.
Using Proposition~\ref{prop:softmax-conjugate},
\begin{align*}
L^*(v) =
\tsum_{j \neq y} v_j \log v_j + (1 + v_y) \log(1 + v_y) ,
\end{align*}
where $\inner{\ones, v} = 0$ and $\wo{v}{y} \in \Ska$.
Let $x \bydef \wo{v}{y}$ and $s \bydef - v_y$.
We have $s = \inner{\ones, x}$ and from $\tr\left( A K \tra{A} \right)$ we get
\begin{align*}
K_{ii} (\inner{x,x} + s^2) / (\lambda n)^2 
- 2 \innerb{\wo{q}{y} - q_{y} \ones, x} / (\lambda n) ,
\end{align*}
where
$q = \sum_{j \neq i} K_{ij} a_j = A K_i - K_{ii} a_i$.
Finally, we plug everything together as in
Proposition~\ref{prop:smooth-topk-hinge-update}.
\end{proof}
\fi

\subsection{Proof of Proposition~\ref{prop:solve-topk-entropy-update}}

\subsection{Proof of Proposition~\ref{prop:solve-topk-entropy-primal}}

\subsection{Proof of Proposition~\ref{prop:smooth-svm-ml-update}}

\subsection{Proof of Proposition~\ref{prop:proj-bipartite-simplex}}

\subsection{Proof of Proposition~\ref{prop:lr-ml-update}}

\end{document}